%% file: DD-RFF.tex
\documentclass{article}

\usepackage[final,nonatbib]{neurips_2022}

\usepackage{xcolor}
\usepackage[utf8]{inputenc} % allow utf-8 input
\usepackage[T1]{fontenc}    % use 8-bit T1 fonts
\usepackage[backref=page]{hyperref}       % hyperlinks
\usepackage{url}            % simple URL typesetting
\usepackage{booktabs}       % professional-quality tables
\usepackage{amsfonts}       % blackboard math symbols
\usepackage{nicefrac}       % compact symbols for 1/2, etc.
\usepackage{microtype}      % microtypography
\usepackage{xcolor}         % colors

\usepackage{amsmath,amssymb,mathrsfs,amsthm,mathtools}
\usepackage[capitalise]{cleveref}
\usepackage{algorithm}
\usepackage{algorithmic}
\usepackage{amsmath,graphicx}
\usepackage[american]{babel}
\usepackage{color}
\usepackage{hyperref}
\usepackage{amssymb}
\usepackage{multirow}
\usepackage{amsfonts,mathrsfs}
\usepackage{bm,url}
\usepackage{booktabs,subfigure}
\usepackage{amsthm}
\usepackage{microtype}
\usepackage{multirow}
\usepackage{threeparttable}
\usepackage{booktabs} % for professional tables
\usepackage{float}
\usepackage{wrapfig}
\usepackage{float}
\usepackage{tikz} 
\usepackage{schemata}

\usepackage{enumitem}

\usepackage{selectp}
\definecolor{fhcolor}{rgb}{0.523, 0.235, 0.625}

\usepackage{tikz-qtree}
\usetikzlibrary{shadows,trees}

\usepackage{tikz}
\usetikzlibrary{positioning,shadows,arrows,trees,shapes,fit}

\DeclareMathOperator*{\argmin}{argmin}

\newtheorem{theorem}{Theorem}
\newtheorem{lemma}{Lemma}
\newtheorem{proposition}{Proposition}

\newtheorem{case}{Case}
\newtheorem{assumption}{Assumption}

\usepackage[framemethod=tikz]{mdframed}

\definecolor{ocre}{RGB}{243,102,25}
\definecolor{mygray}{RGB}{243,243,244}

\newmdenv[
innertopmargin=0pt,
backgroundcolor=mygray,
linecolor=ocre,
innerleftmargin=0pt,
innerrightmargin=0pt,
leftmargin=10pt
]{mymath}

\title{On the Double Descent of Random Features Models Trained with SGD}

\author{%
	Fanghui Liu\thanks{Most of this work was done when Fanghui was at KU Leuven. Correspondence to: Fanghui Liu \texttt{<fanghui.liu@epfl.ch>}.} \\
	LIONS, EPFL\\
	\texttt{fanghui.liu@epfl.ch}
	\And
	Johan A.K. Suykens \\
		ESAT-STADIUS, KU Leuven \\
	 \texttt{johan.suykens@esat.kuleuven.be} \\
	 \AND
	 Volkan Cevher \\
	LIONS, EPFL\\
	\texttt{volkan.cevher@epfl.ch} \\
}

\begin{document}
	\maketitle
	\begin{abstract}
		We study generalization properties of random features (RF) regression in high dimensions optimized by stochastic gradient descent (SGD) in under-/over-parameterized regime. In this work, we derive precise non-asymptotic error bounds of RF regression under both constant and polynomial-decay step-size SGD setting, and observe the double descent phenomenon both theoretically and empirically. 
		Our analysis shows how to cope with multiple randomness sources of initialization, label noise, and data sampling (as well as stochastic gradients) with no closed-form solution, and also goes beyond the commonly-used Gaussian/spherical data assumption. 
		Our theoretical results demonstrate that, with SGD training, RF regression still generalizes well for interpolation learning, and is able to characterize the double descent behavior by the unimodality of variance and monotonic decrease of bias. Besides, we also prove that the constant step-size SGD setting incurs no loss in convergence rate when compared to the exact minimum-norm interpolator, as a theoretical justification of using SGD in practice.

	\end{abstract}
	
\input{maintext}

\bibliographystyle{unsrt}
\bibliography{refs}

%%%%%%%%%%%%%%%%%%%%%%%%%%%%%%%%%%%%%%%%%%%%%%%%%%%%%%%%%%%%
\newpage
\input{checklist}
\input{appendix}

\end{document}

%% file: maintext.tex
\section{Introduction}

Over-parameterized models, e.g., linear/kernel regression \cite{hastie2019surprises,bartlett2020benign,wu2020optimal,mei2019generalization} and neural networks \cite{nakkiran2019deep,yang2020rethinking,ju2021generalization}, still generalize well even if the labels are pure noise \cite{zhang2016understanding}.
Such high-capacity models have received significant attention recently as they
go against with classical generalization theory. A paradigm for understanding this important phenomenon is \emph{double descent} \cite{belkin2019reconciling}, in which the test error first decreases with increasing number of model parameters in the under-parameterized regime. They large error is yielded until interpolating the data, which is called the interpolation threshold. Finally, the test error decreases again in the over-parameterized regime. 

Our work partakes in this research vein and studies the random features (RF) model  \cite{rahimi2007random}, as a simplified version of neural networks, in the context of double descent phenomenon. Briefly, RF model samples random features $\{ \bm \omega_i \}_{i=1}^m$ from a specific distribution, corresponding to a kernel function. We then construct an explicit map: $\bm x \in \mathbb{R}^d \mapsto \sigma(\bm W \bm x) \in \mathbb{R}^m$, where $\bm W = [\bm \omega_1, \cdots, \bm \omega_m]^{\!\top} \in \mathbb{R}^{m \times d}$ is the random features matrix and $\sigma(\cdot)$ is the nonlinear (activation) function determined by the kernel. As a result, the RF model can be viewed as training a two-layer neural network where the weights in the first layer are chosen randomly and then fixed (a.k.a.\ the random features) and only the output layer is optimized, striking a trade-off between practical performance and accessibility to analysis \cite{mei2019generalization,d2020double}. 
An RF model becomes an over-parameterized model if we take the number of random features $m$ larger than that of training data $n$. The literature on RF under the over-parameterized regime can be split into various camps according to different assumptions on the formulation of target function, data distribution, and activation functions \cite{mei2019generalization,ba2020generalization,d2020double,liao2020random,gerace2020generalisation,lin2020causes} (see comparisons in Table~\ref{Tabsetting} in Appendix~\ref{app:setting}). The existing theoretical results demonstrate that the excess risk curve exhibits double descent. % with respect to the ratio between $m$, $n$, and the dimension $d$.

Nevertheless, the analysis framework of previous work on RF regression mainly relies on the least-squares closed-form solution, including \emph{minimum-norm} interpolator and ridge regressor.
Besides, they often assume the data with specific distribution, e.g., to be Gaussian or uniformly spread on a sphere.
Such dependency on the analytic solution and relatively strong data assumption in fact mismatches practical neural networks optimized by stochastic gradient descent (SGD) based algorithms.
Our work precisely bridges this gap: We provide a new analysis framework for the generalization properties of RF models trained with SGD and general activation functions, also accommodating adaptive (i.e., polynomial decay) step-size selection, and provide non-asymptotic results in under-/over-parameterized regimes.  We make the following contributions and findings:

First, we characterize statistical properties of covariance operators/matrices in RF, including $\Sigma_m := \frac{1}{m}\mathbb{E}_{\bm x}[\sigma(\bm W \bm x/\sqrt{d}) \sigma(\bm W \bm x/\sqrt{d})^{\!\top} ]$ and its expectation version $\widetilde{\Sigma}_m := \mathbb{E}_{\bm W} [\Sigma_m]$. We demonstrate that, under Gaussian initialization, if the activation function $\sigma(\cdot): \mathbb{R} \mapsto \mathbb{R}$ is Lipschitz continuous, $\mathrm{Tr}(\Sigma_m)$ is a sub-exponential random variable with $\mathcal{O}(1)$ sub-exponential norm; $\widetilde{\Sigma}_m$ has only two distinct eigenvalues at $\mathcal{O}(1)$ and $\mathcal{O}(1/m)$ order, respectively. Such analysis on the spectra of $\Sigma_m$ and $\widetilde{\Sigma}_m$ (without spectral decay assumption) is helpful to obtain sharp error bounds for excess risk.
This is different from the least squares setting based on effective dimension \cite{bartlett2020benign,zou2021benign}.
\vspace{-0.0cm}
 
Second, based on the bias-variance decomposition in stochastic approximation, we take into account multiple randomness sources of initialization, label noise, and data sampling as well as stochastic gradients. We (partly) disentangle these randomness sources and derive non-asymptotic error bounds under the optimization effect: the error bounds for bias and variance as a function of the radio $m/n$ are monotonic decreasing and unimodal, respectively. Importantly, our analysis holds for both constant and polynomial-decay step-size SGD setting, and is valid under sub-Gaussian data and general activation functions.\vspace{-0.0cm}

Third, our non-asymptotic results show that, RF regression trained with SGD still generalizes well for interpolation learning, and is able to capture the double descent behavior. In addition, we  demonstrate that the constant step-size SGD setting incurs no loss on the convergence rate of excess risk when compared to the exact least-squares closed form solution.
	% which provides the justification of using SGD in model machine learning models. 
	Our empirical evaluations support our theoretical results and findings.  
\vspace{-0.0cm}

Our analysis (technical challenges are discussed in Section~\ref{sec:proofoutline}) sheds light on the effect of SGD on high dimensional RF models in under-/over-parameterized regimes, and bridges the gap between the minimum-norm solution and numerical iteration solution in terms of optimization and generalization on double descent.
It would be helpful for understanding large dimensional machine learning and neural network models more generally.

\vspace{-0.05cm}
\section{Related work and problem setting}\vspace{-1mm}
\label{sec:preli}
\vspace{-0.05cm}

This section reviews relevant works and introduces our problem setting of RF regression with SGD. 

{\bf Notation:} The notation $\bm a \otimes \bm a$ denotes the tensor product of a vector $\bm a$.
For two operators/matrices, $A \preccurlyeq B$ means $B-A$ is positive semi-definite (PSD).
For any two positive sequences $\{a_t\}_{t=1}^s$ and $\{b_t\}_{t=1}^s$, the notation $a_t \lesssim b_t$ means that there exists a positive constant $C$ independent of $s$ such that $a_t \leq C b_t$, and analogously for $\sim$, $\gtrsim$, and $\precsim$.
%In this paper, the notation $A < \infty$ means, there exists a positive constant $C$ independent of $n$, $m$, and $d$ such that $A \leqslant C$.
For any $a, b \in \mathbb{R}$, $a \wedge b$ denotes the minimum of $a$ and $b$.

\vspace{-3mm}
\subsection{Related work}\vspace{-1mm}

A flurry of research papers are devoted to analysis of over-parameterized models on optimization \cite{kawaguchi2019gradient,allen2019convergence,zou2019improved}, generalization (or their combination) under neural tangent kernel \cite{jacot2018neural,arora2019fine,chizat2019lazy} and mean-field analysis regime \cite{mei2019mean,chizat2020implicit}.
We take a unified perspective on optimization and generalization but work in the high-dimensional setting to fully capture the double descent behavior. By high-dimensional setting, we mean that $m$, $n$, and $d$ increase proportionally, large and comparable \cite{mei2019generalization,ba2020generalization,liao2020random,d2020double}. 

{\bf Double descent in random features model:} Characterizing the double descent of the RF model often derives from random matrix theory (RMT) in high dimensional statistics \cite{hastie2019surprises,mei2019generalization,ba2020generalization,liao2020random,li2021towards} and from the replica method \cite{d2020double,rocks2020memorizing,gerace2020generalisation}. Under specific assumptions on data distribution, activation functions, target function, and initialization, these results show that the generalization error/excess risk increase when $m/n < 1$, diverge when $m/n \rightarrow 1$, and then decrease when $m/n > 1$. 
Further, refined results are developed on the \emph{analysis of variance} due to multiple randomness sources \cite{d2020double,adlam2020understanding,lin2020causes}.
%\cite{ba2020generalization} on RF optimized by gradient descent exhibits the double descent behavior under the Gaussian data assumption.
We refer to comparisons in Table~\ref{Tabsetting} in Appendix~\ref{app:setting} for further details.
Technically speaking, since RF (least-squares) regression involves with inverse random matrices, these two classes of methods attempt to achieve a similar target: how to disentangle the nonlinear activation function by the Gaussian equivalence conjecture.
RMT utilizes calculus of deterministic equivalents (or resolvents) for random matrices and replica methods focus on some specific scalar parameters that allows for circumventing the expectation computation.
In fact, most of the above methods can be asymptotically equivalent to the Gaussian covariate model \cite{hu2020universality}.

{\bf Non-asymptotic stochastic approximation:} Many papers on linear least-squares regression \cite{bach2013non,jain2018parallelizing}, kernel regression \cite{dieuleveut2016nonparametric,dieuleveut2017harder}, random features \cite{carratino2018learning} with SGD often work in the under-parameterized regime, where $d$ is finite and much smaller than $n$.
In the over-parameterized regime, under GD setting, the excess risk of least squares is controlled by the smallest positive eigenvalue in \cite{kuzborskij2021role} via random matrix theory.
Under the averaged constant step-size SGD setting, 
the excess risk in \cite{chen2020dimension} on least squares in high dimensions can be independent of $d$, and the convergence rate is built in \cite{zou2021benign}.
This convergence rate is also demonstrated under the minimal-iterate \cite{berthier2020tight} or last-iterate \cite{varre2021last} setting in step-size SGD for noiseless least squares.
We also notice a concurrent work \cite{wu2021last} on last-iterate SGD with decaying step-size on least squares.
Besides, the existence of multiple descent \cite{chen2020multiple,liang2020multiple} beyond double descent and SGD as implicit regularizer \cite{neyshabur2017geometry,smith2020origin} can be traced to the above two lines of work. Our work shares some similar technical tools with \cite{dieuleveut2016nonparametric} and \cite{zou2021benign} but differs from them in several aspects. 
We detail the differences in Section~\ref{sec:proofoutline}. 

\vspace{-0.2cm}
\subsection{Problem setting}
\vspace{-0.cm}
We study the standard problem setting for RF least-squares regression and adopt the relevant terminologies from learning theory: \textit{cf.}, \cite{cucker2007learning,dieuleveut2016nonparametric,carratino2018learning,li2021towards} for details.
Let $X \subseteq \mathbb{R}^d$ be a metric space and $Y \subseteq \mathbb{R}$.  The training data $\{  (\bm x_i, y_i) \}_{i=1}^n $ are assumed to be independently drawn from a non-degenerate unknown Borel probability measure $\rho$ on $X \times Y$. The \emph{target function} of $\rho$ is defined by $f_{\rho}(\bm x) = \int_Y y \,\mathrm{d} \rho(y \mid \bm x)$, where $\rho(\cdot\mid\bm x)$ is the conditional distribution of $\rho$ at $\bm x \in X$. 
%We assume the following data model $y = f_{\rho}(\bm x) + \varepsilon$, where $\varepsilon$ is the noise. Taking $\bm X = [\bm x_1, \bm x_2, \cdots, \bm x_n]^{\!\top} \in \mathbb{R}^{n \times d}$, the data model follows $\bm y = f_{\rho}(\bm X) + \bm \varepsilon$ with $\bm \varepsilon =[\varepsilon_1, \epsilon_2, \cdots, \varepsilon_n]^{\!\top}$.

{\bf RF least squares regression:} We study the RF regression problem with the squared loss as follows:
\begin{equation*}
	\min_{f \in \mathcal{H}} \mathcal{E}(f),\quad \mathcal{E}(f) \!:=\!\! \int (f(\bm x) - y)^2  \mathrm{d} \rho(\bm x, y) \!=\! \| f - f_{\rho} \|^2_{L^2_{\rho_X}} \!\!\,, ~~\mbox{with $f(\bm x) = \langle \bm \theta, \varphi(\bm x) \rangle$}\,,
\end{equation*}
where the optimization vector $\bm \theta \in \mathbb{R}^m$ and the feature mapping $\varphi(\bm x)$ is defined as
\begin{equation}\label{mapping}
	\begin{split}
		\varphi(\bm x) & := 
		\frac{1}{\sqrt{m}} \left[\sigma(\bm \omega^{\!\top}_1 \bm x /\sqrt{d}), \cdots,\sigma(\bm \omega^{\!\top}_m \bm x/\sqrt{d})\right]^{\!\top}  = \frac{1}{\sqrt{m}}\sigma(\bm W \bm x/\sqrt{d}) \in \mathbb{R}^m\,,
	\end{split}
\end{equation}
where $\bm W = [\bm \omega_1, \cdots, \bm \omega_m]^{\!\top} \in \mathbb{R}^{m \times d}$ with $ W_{ij} \sim \mathcal{N}(0,1)$ corresponds to such two-layer neural network initialized with random Gaussian weights.
Then, the corresponding hypothesis space $\mathcal{H}$ is a reproducing kernel Hilbert space
\begin{equation}\label{defrkhs}
	\mathcal{H} := \left\{ f \in {L^2_{\rho_X}} \Big|~~ f(\bm x) = \frac{1}{\sqrt{m}} \langle \bm \theta, \sigma({\bm W \bm x}/{\sqrt{d}}) \rangle \right\}\,,
\end{equation}
with $\| f \|^2_{L^2_{\rho_X}} = \int_{X} | f(\bm x) |^2 \mathrm{d} \rho_X(\bm x)  = \langle f, \Sigma_m f \rangle_{\mathcal{H}}$ with the \emph{covariance} operator $\Sigma_m: \mathbb{R}^m \rightarrow \mathbb{R}^m$
\begin{equation}\label{sigmamdef}
	\Sigma_m = \int_X \varphi(\bm x) \otimes \varphi(\bm x) \mathrm{d} \rho_X(\bm x) \,,
\end{equation}
actually defined in $\mathcal{H}$ that is isomorphic to $\mathbb{R}^m$.
This is the usually (uncentered) covariance matrix in finite dimensions,\footnote{In this paper, we do not distinguish the notations $\Sigma_m$ and $\bm \Sigma_m$. This is also suitable to other operators/matrices, e.g., $\widetilde{\Sigma}_m$.} i.e., $\bm \Sigma_m = \mathbb{E}_{\bm x} [\varphi(\bm x) \otimes \varphi(\bm x)]$. 
Define $J_m: \mathbb{R}^m \rightarrow L^2_{\rho_X}$ such that$(J_m \bm v)(\cdot) = \langle \bm v, \varphi(\cdot) \rangle, \quad \forall \bm v \in \mathbb{R}^m$, we have $\Sigma_m = J_m^*J_m$, where $J_m^*$ denotes the adjoint operator of $J_m$.
Clearly, $\Sigma_m$ is random with respect to $\bm W$, and thus its deterministic version is defined as $	\widetilde{\Sigma}_m = \mathbb{E}_{\bm x, \bm W} [\varphi(\bm x) \otimes \varphi(\bm x)] $.

{\bf SGD with averaging:} Regarding the stochastic approximation, we consider the one pass SGD with iterate averaging and adaptive step-size at each iteration $t$: after a training sample $(\bm x_t, y_t) \sim \rho$ is observed, we update the decision variable as below (initialized at $\bm \theta_0 $)
\begin{equation}\label{eq:sgdf}
	\bm \theta_t = \bm \theta_{t-1} + \gamma_t [ y_t - \langle{\bm \theta_{t-1}, \varphi(\bm x_t)} \rangle ] \varphi(\bm x_t),\qquad t=1,2, \dots ,  n\,,
\end{equation}
where we use the polynomial decay step size $\gamma_t := \gamma_0 t^{-\zeta}$ with $\zeta \in [0,1)$, following \cite{dieuleveut2016nonparametric}.
This setting also holds for the constant step-size case by taking $\zeta=0$. 
Besides, we employ the bath size $=1$ in an online setting style, which is commonly used in theory \cite{dieuleveut2016nonparametric,zou2021benign,nitanda2021optimal} for ease of analysis, which captures the key idea of SGD by combining stochastic gradients and data sampling.

The final output is defined as the average of the iterates: 
$\bar{\bm \theta}_{n} := \frac{1}{n} \sum_{t=0}^{n-1} \bm \theta_t$.
Here we sum up $\{ \theta_t \}_{t=0}^{n-1}$ with $n$ terms for notational simplicity.
The  optimality condition for Eq.~\eqref{eq:sgdf} implies $\mathbb{E}_{(\bm x,y)\sim \rho}[(y-\langle \bm \theta^*,\varphi(\bm x)\rangle)\varphi(\bm x)] = \bm 0$, which corresponds to $f^* = J_m \bm \theta^*$ if we assume that
$f^* = \argmin_{f \in {\mathcal{H}}} \mathcal{E}(f)$ exists (see Assumption~\ref{assexist} in the next section).
Likewise, we have $f_t = J_m \bm \theta_t$ and $\bar{f}_n = J_m \bar{\bm \theta}_{n}$.

In this paper, we study the averaged excess risk $\mathbb{E}\| \bar{f}_n - f^* \|^2_{L^2_{\rho_X}}$ instead of $\mathbb{E}\| \bar{f}_n - f_{\rho} \|^2_{L^2_{\rho_X}}$, that follows \cite{dieuleveut2016nonparametric,Rudi2017Generalization,carratino2018learning,li2021towards}, as $f^*$ is the best possible solution in $\mathcal{H}$ and the mis-specification error $\| f^* - f_{\rho} \|^2_{L^2_{\rho_X}}$ pales into insignificance.
Note that the expectation used here is considered with respect to the random features matrix $\bm W$, and the distribution of the training data $\{ (\bm x_t, y_t) \}_{t=1}^n$ (note that $\| \bar{f}_n - f^* \|^2_{L^2_{\rho_X}}$ is itself a different expectation over $\rho_X$).

\vspace{-0.2cm}
\section{Main results}
\vspace{-0.2cm}

In this section, we present our main theoretical results on the generalization properties employing  error bounds for bias and variance of RF regression in high dimensions optimized by averaged SGD.

\vspace{-0.2cm}
\subsection{Assumptions}
\vspace{-0.2cm}
Before we present our result, we list the assumptions used in this paper, refer to Appendix~\ref{app:assumptions} for more discussions.

\begin{assumption}\label{assumdata} \cite[high dimensional setting]{el2010spectrum,hastie2019surprises}
	We work in the large $d, n, m$ regime with $ c \leqslant \{ d/n, m/n\} \leqslant C$ for some constants $c,C >0$ such that $m,n,d$ are large and comparable. The data point $\bm x \in \mathbb{R}^d$ is assumed to satisfy $\| \bm x \|_2^2 \sim \mathcal{O}(d)$ and the sample covariance operator $ \Sigma_d := \mathbb{E}_{\bm x} [\bm x \otimes \bm x] $ with bounded spectral norm  $\|\Sigma_d\|_2$ (finite and independent of $d$).
\end{assumption}

\begin{assumption}\label{assexist}
	There exists $f^* \in \mathcal{H}$ such that $f^* = \argmin_{f \in \mathcal{H}} \mathcal{E}(f)$ with bounded Hilbert norm.
\end{assumption}
{\bf Remark:} 
This bounded Hilbert norm assumption is commonly used in \cite{liang2020just,liang2020multiple,mei2021generalization} even though $n$ and $d$ tend to infinity. 
It holds true for linear functions with $\| f \|_{\mathcal{H}} \leqslant 4 \pi$ \cite{bach2017breaking}, see Appendix~\ref{app:assumptions} for details.

\begin{assumption}\label{assumact}
	The activation function $\sigma(\cdot)$ is assumed to be Lipschitz continuous.
\end{assumption}
{\bf Remark:} This assumption is quite general to cover commonly-used activation functions used in random features and neural networks, e.g., ReLU, Sigmoid, Logistic, and sine/cosine functions.

Recall $\Sigma_m := \mathbb{E}_{\bm x} [\varphi(\bm x) \otimes \varphi(\bm x)]$ in Eq.~\eqref{sigmamdef} and its expectation $\widetilde{\Sigma}_m := \mathbb{E}_{\bm W} [\Sigma_m]$, we make the following fourth moment assumption that follows \cite{bach2013non,zou2021benign,varre2021last} to analyse SGD for least squares.
\begin{assumption}[Fourth moment condition]\label{assump:bound_fourthmoment}
	Assume there exists some positive constants $r', r \geqslant 1$, such that for any PSD
	operator $A$, it holds that 
	\begin{equation*}\label{assum:fmc}
		\begin{split}
			\mathbb{E}_{\bm W} [\Sigma_m A \Sigma_m]   \! \preccurlyeq \! \mathbb{E}_{\bm W} \bigg( \! \mathbb{E}_{\bm x} \Big(\! [\varphi(\bm x) \otimes \varphi(\bm x) ] A [\varphi(\bm x) \otimes \varphi(\bm x)]  \!\Big) \!\!\bigg) \!\!  \preccurlyeq \!\! r' \mathbb{E}_{\bm W}[\mathrm{Tr}({\Sigma}_m A){\Sigma}_m] \preccurlyeq r \mathrm{Tr}(\widetilde{\Sigma}_m A)\widetilde{\Sigma}_m.
		\end{split}
	\end{equation*}
\end{assumption}
{\bf Remark:} This assumption requires the data are drawn from some not-too-heavy-tailed distribution, e.g., $\Sigma_m^{-\frac{1}{2}}\bm x$ has sub-Gaussian tail, common in high dimensional statistics. This condition is weaker than most previous work on double descent that requires the data to be Gaussian \cite{hastie2019surprises,d2020double,adlam2020understanding,ba2020generalization}, or uniformly spread on a sphere \cite{mei2019generalization,ghorbani2019linearized}, see comparisons in Table~\ref{Tabsetting} in Appendix~\ref{app:setting}.
Note that the assumption for any PSD operator is just for ease of description. In fact some certain PSD operators satisfying this assumption are enough for our proof.
Besides, a special case of this assumption with $A:=I$ is proved by Lemma~\ref{lemma:m2}, and thus this assumption can be regarded as a natural extension, with more discussions in Appendix~\ref{app:assumptions}.

%Our next assumption is a noise condition, where it is helpful to interpret $y - \langle \bm \theta^*,\varphi(\bm x)\rangle$ as the additive noise. 
\begin{assumption}[Noise condition]\label{assump:noise}
	There exists $\tau > 0$ such that
	$	\Xi:= \mathbb{E}_{\bm x} [\varepsilon^2 \varphi(\bm x) \otimes \varphi(\bm x)] \preccurlyeq \tau^2 \Sigma_m$,
	where the noise $\varepsilon := y - f^*(\bm x)$.
\end{assumption}
{\bf Remark:} This noise assumption is standard in \cite{dieuleveut2016nonparametric,zou2021benign} and holds for the standard noise model $y = f^*(\bm x)+ \varepsilon$ with $\mathbb{E}[\varepsilon]=0$ and $\mathbb{V}[\varepsilon] < \infty$ \cite{hastie2019surprises}.

\vspace{-0.2cm}
\subsection{Properties of covariance operators}
\label{sec:statcov}
\vspace{-0.1cm}
Before we present the main results, we study statistical properties of $\Sigma_m$ and $\widetilde{\Sigma}_m$ by the following lemmas (with proof deferred to Appendix~\ref{app:rescov}), that will be needed for our main result.
This is different from the least squares setting \cite{bartlett2020benign,zou2021benign} that introduces the effective dimension to separate the entire space into a “head” subspace where the error decays more quickly than the complement “tail” subspace.
Instead, the following lemma shows that $\widetilde{\Sigma}_m$ has only two distinct eigenvalues at $\mathcal{O}(1)$ and $\mathcal{O}(1/m)$ order, respectively.
Such fast eigenvalue decay can avoid extra data spectrum assumption for tight bound.
For description simplicity, we consider the single-output activation function: $\sigma(\cdot): \mathbb{R} \rightarrow \mathbb{R}$.
Our results can be extended to multiple-output activation functions, see Appendix~\ref{sec:example} for details.

\begin{lemma}\label{thmH}
	Under Assumption~\ref{assumdata} and~\ref{assumact}, the expected covariance operator $\widetilde{\Sigma}_m := \mathbb{E}_{{\bm x}, \bm W }[\varphi({\bm x}) \otimes \varphi({\bm x})] \in \mathbb{R}^{m \times m}$ has the same diagonal elements and the same non-diagonal element
	\begin{equation*}
		(\widetilde{\Sigma}_m)_{ii} = \frac{1}{m} \mathbb{E}_{\bm x} \mathbb{E}_{z \sim \mathcal{N}(0,{\| \bm x \|^2_2}/{d})} [\sigma(z)]^2 \!\sim\! \mathcal{O}(1/m) \,, \quad 	(\widetilde{\Sigma}_m)_{ij} \!=\! \frac{1}{m} \mathbb{E}_{\bm x}\!\! \left( \mathbb{E}_{z \sim \mathcal{N}(0,{\| \bm x \|^2_2}/{d})} [\sigma(z)] \right)^2 \!\!\!\sim\! \mathcal{O}(1/m) \,.
	\end{equation*}
	Accordingly, $\widetilde{\Sigma}_m$ has only two distinct eigenvalues
	\begin{equation*}
		\begin{split}
			&	\widetilde{\lambda}_1 =(\widetilde{\Sigma}_m)_{ii} +(m-1) (\widetilde{\Sigma}_m)_{ij}  \sim \mathcal{O}(1)\,, \quad \widetilde{\lambda}_2 \!=\! (\widetilde{\Sigma}_m)_{ii} - (\widetilde{\Sigma}_m)_{ij} \!=\! \frac{1}{m} \mathbb{E}_{\bm x} \mathbb{V}[\sigma(z)] \sim \mathcal{O}(1/m) \!\,.
		\end{split}
	\end{equation*}
\end{lemma}
{\bf Remark:} Lemma~\ref{thmH} implies $\operatorname{tr}(\widetilde{\Sigma}_m) < \infty$. 
In fact, $ \mathbb{E}_{\bm x} \mathbb{V}[\sigma(z)] > 0$ holds almost surely as $\sigma(\cdot)$ is not a constant, and thus $\widetilde{\Sigma}_m$ is positive definite.
%For description simplicity, we mainly focus on the single-output activation function in this paper, but it can be easily extended to a multiple-output version, as discussed in below.
%Our error bounds will largely depend on $\widetilde{\lambda}_2 = \frac{1}{m} \mathbb{E}_{\bm x} \mathbb{V}[\sigma(z)]$.

Here we take the ReLU activation $\sigma(x) = \max\{x,0 \}$ as one example, RF actually approximates the first-order arc-cosine kernel \cite{cho2009kernel} with $\varphi(\bm x) \in \mathbb{R}^m$. We have $(\widetilde{\Sigma}_m)_{ii} =  \frac{1}{2md} \mathrm{Tr}(\Sigma_d)$ and $(\widetilde{\Sigma}_m)_{ij} = \frac{1}{2md \pi} \mathrm{Tr}(\Sigma_d)$ by recalling $\Sigma_d := \mathbb{E}_{\bm x} [\bm x \bm x^{\!\top}]$ and $\mathrm{Tr}(\Sigma_d)/d \sim \mathcal{O}(1)$. 
More examples can be found in Appendix~\ref{sec:example}.

\begin{lemma}\label{lemsubexp}
	Under Assumptions~\ref{assumdata} and~\ref{assumact}, random variables $\| {\Sigma}_m \|_2$, $\| {\Sigma}_m - \widetilde{\Sigma}_m \|_2$, and $\mathrm{Tr}(\Sigma_m)$ are sub-exponential, and have sub-exponential norm at $\mathcal{O}(1)$ order.
	%\begin{equation*}
		%	\mathbb{P}[\| \bm {\Sigma}_m - \widetilde{\bm \Sigma}_m \|_2 > t] \leq 2 \exp(- ct)\,,
		%\end{equation*}
	%where $c$ is some constant.
\end{lemma}
{\bf Remark:} This lemma characterizes the sub-exponential property of covariance operator $\Sigma_m$, which is a fundamental result for our proof since the bias and variance involve them.

The following lemma demonstrates that the behavior of the fourth moment can be bounded.
\begin{lemma}\label{lemma:m2}
	Under Assumptions~\ref{assumdata},and~\ref{assumact}, there exists a constant $r>0$ such that
	$\mathbb{E}_{\bm W} \left( \Sigma_m^2 \right) \preccurlyeq \mathbb{E}_{\bm x, \bm W}[\varphi(\bm x) \otimes \varphi(\bm x)\otimes \varphi(\bm x) \otimes \varphi(\bm x)] \preccurlyeq r \mathrm{Tr}(\widetilde{\Sigma}_m) \widetilde{\Sigma}_m $.
\end{lemma}
%{\bf Remark:} Lemma~\ref{lemma:m2} is a special case of Assumption~\ref{assum:fmc} if we take $ A:= I$ and $r:= 1+ \mathcal{O}\left( \frac{1}{m} \right)$.

\begin{lemma}\label{trace1}
	Under Assumptions~\ref{assumdata} and~\ref{assumact}, we have  $\mathrm{Tr}[ \widetilde{\Sigma}_m^{-1} \mathbb{E}_{\bm W} (\Sigma_m^2) ] \sim \mathcal{O}(1)$.
\end{lemma}
%{\bf Remark:} This is a direct corollary of Lemma~\ref{lemma:m2}.

We remark here that Lemma~\ref{lemma:m2} is a special case of Assumption~\ref{assum:fmc} if we take $ A:= I$ and $r:= 1+ \mathcal{O}\left( \frac{1}{m} \right)$; and Lemma~\ref{trace1} is a direct corollary of Lemma~\ref{lemma:m2}.

\subsection{Results for error bounds}
\vspace{-0.2cm}
Recall the definition of the noise $\bm \varepsilon = [\varepsilon_1, \cdots, \varepsilon_n]^{\!\top}$ with $\varepsilon_t = y_t - f^*(\bm x_t)$, $t=1,2,\dots,n$, the averaged excess risk can be expressed as
\begin{equation*}
	\begin{split}
		& \mathbb{E}\| \bar{f}_n - f^* \|^2_{L^2_{\rho_X}} := \mathbb{E}_{\bm X, \bm W, \bm \varepsilon}\| \bar{f}_n - f^* \|^2_{L^2_{\rho_X}}  = \mathbb{E}_{\bm X, \bm W, \bm \varepsilon} \langle \bar{f}_n - f^*, \Sigma_m(\bar{f}_n - f^*)  \rangle \!=\! \mathbb{E}_{\bm X, \bm W, \bm \varepsilon} \langle \bar{\eta}_n, \Sigma_m \bar{\eta}_n  \rangle \!\,,
	\end{split}
\end{equation*}
where $\bar{\eta}_{n} := \frac{1}{n}\sum_{t=0}^{n-1} \eta_t$ with the centered SGD iterate $ \eta_t := f_t - f^*$. 
Following the standard bias-variance decomposition in stochastic approximation \cite{dieuleveut2016nonparametric,jain2018parallelizing,zou2021benign}, it admits
\begin{equation*}
	\begin{split}
		\eta_t & = [I - \gamma_t \varphi(\bm x_t) \otimes \varphi(\bm x_t) ] (f_{t-1} - f^* ) + \gamma_t \varepsilon_t \varphi(\bm x_t) \,,
	\end{split}
\end{equation*}
where the first term corresponds to the bias 
\begin{equation}\label{eq:bias_iterates}
	\eta_t^{{\tt bias}} = [I - \gamma_t \varphi(\bm x_t)\otimes \varphi(\bm x_t) ] \eta_{t-1}^{{\tt bias}} , \quad \eta_0^{{\tt bias}} = f^*\,,
\end{equation}
and the second term corresponds to the variance
\begin{equation}\label{eq:variance_iterates}
	\eta_t^{{\tt var}} = [I - \gamma_t \varphi(\bm x_t) \otimes \varphi(\bm x_t) ] \eta_{t-1}^{{\tt var}} + \gamma_t \varepsilon_t \varphi(\bm x_t) , \quad \eta_0^{{\tt var}} = 0\,.
\end{equation}
Accordingly, we have ${f}_t = \eta_t^{{\tt bias}} + \eta_t^{{\tt var}} + f^*$ due to $\mathbb{E}_{\bm \varepsilon} \bar{f}_n = \bar{\eta}^{{\tt bias}}_n + f^*$ and $\| f \|^2_{L^2_{\rho_X}} = \langle f, \Sigma_m f \rangle$.
\begin{proposition}\label{propdefbiasvar}
	Based on the above setting, the averaged excess risk admits the following bias-variance decomposition
	\begin{equation*}
		\begin{split}
			&	\mathbb{E}\| \bar{f}_n  - f^* \|^2_{L^2_{\rho_X}} 
			\!\!=\! \mathbb{E}_{\bm X, \bm W, \bm \varepsilon}\| \bar{f}_n \!-\! \mathbb{E}_{\bm \varepsilon} \bar{f}_n \!+\! \mathbb{E}_{\bm \varepsilon} \bar{f}_n \!-\! f^* \|^2_{L^2_{\rho_X}}  \!\!\!=\! \underbrace{\mathbb{E}_{\bm X, \bm W} \langle \bar{\eta}^{{\tt bias}}_n, \Sigma_m \bar{\eta}^{{\tt bias}}_n  \rangle }_{:= {\tt Bias}} \!+\! \underbrace{\mathbb{E}_{\bm X, \bm W, \bm \varepsilon} \langle \bar{\eta}^{{\tt var}}_n, \Sigma_m \bar{\eta}^{{\tt var}}_n  \rangle}_{:= {\tt Variance}}\,. 
		\end{split}
	\end{equation*}
\end{proposition}
By (partly) decoupling the multiple randomness sources of initialization, label noise, and data sampling (as well as stochastic gradients), we give precise non-asymptotic error bounds for bias and variance as below. 
%In Section~\ref{sec:proofoutline}, we provide the proof framework with high level ideas, discuss our derived result, and compare it with previous work.
\begin{theorem} \label{promainba} (Error bound for bias)
	Under Assumptions~\ref{assumdata},~\ref{assexist},~\ref{assumact},~\ref{assump:bound_fourthmoment} with $r' \geqslant 1$, 
	if the step-size $\gamma_t := \gamma_0 t^{-\zeta}$ with $\zeta \in [0,1)$ satisfies $\gamma_0 \lesssim \frac{1}{r'\mathrm{Tr}(\widetilde{\Sigma}_m)} \sim \mathcal{O}(1) $, the ${\tt Bias}$ in Proposition~\ref{propdefbiasvar} holds by
	\begin{equation*}\label{thmbias}
		\begin{split}
			{\tt Bias}
			& \lesssim \gamma_0 r' n^{\zeta-1}  \| f^* \|^2 \sim \mathcal{O}\left( n^{\zeta-1} \right)  \,.
		\end{split}
	\end{equation*}
	%Here $\widetilde{C_1}$, $\widetilde{C_2}$, $\widetilde{C_3}$ are some constants that are independent of $n$ and $d$.
\end{theorem}
{\bf Remark:} 
The error bound for ${\tt Bias}$ is monotonically decreasing at $\mathcal{O}(n^{\zeta-1})$ rate.
For the constant step-size setting, it converges at $\mathcal{O}(1/n)$ rate, which is better than  $\mathcal{O}(\sqrt{\log n/n})$ in \cite{li2021towards} relying on closed-form solution under correlated features with polynomial decay on $\Sigma_d$.
Besides, our result on bias matches the exact formulation in \cite{d2020double} under the closed-form solution, i.e., monotonically decreasing bias. One slight difference is, their result on bias tends to a constant under the over-parameterized regime while our bias result can converge to zero.

\begin{theorem} \label{promainvar} (Error bound for variance)
	Under Assumptions~\ref{assumdata},~\ref{assumact},~\ref{assump:bound_fourthmoment} with $r' \geqslant 1$, and Assumption \ref{assump:noise} with $\tau > 0$, if the step-size $\gamma_t := \gamma_0 t^{-\zeta}$ with $\zeta \in [0,1)$ satisfies $\gamma_0 \lesssim \frac{1}{r'\mathrm{Tr}(\widetilde{\Sigma}_m)} \sim \mathcal{O}(1)$, the ${\tt Variance}$ defined in Proposition~\ref{propdefbiasvar}  holds 
	\begin{equation*}\label{thmvar}
		\begin{split}
			{\tt Variance}
			& \lesssim {\gamma_0 r' \tau^2} \left\{ \begin{array}{rcl}
				\begin{split}
					&   mn^{\zeta-1} ,~\mbox{if $m \leqslant n$} \\
					&  1+n^{\zeta-1} + \frac{n}{m}  ,~\mbox{if $m > n$}  \\
				\end{split}
			\end{array} \right. 
		% \sim \left\{ \begin{array}{rcl}
		%		\begin{split}
		%			&  \mathcal{O}\left( m n^{\zeta -1} \right) ,~\mbox{if $m \leqslant n$} \\
		%			&  \mathcal{O}\left(  1 \right) ,~\mbox{if $m > n$}   \,.
		%		\end{split}
		%	\end{array} \right. 
		\end{split}
	\end{equation*}
\end{theorem}
{\bf Remark:} We make the following remarks:\\
\textit{i}) The error bound for ${\tt Variance}$ is demonstrated to be unimodal: increasing with $m$ in the under-parameterized regime and decreasing with $m$ in the over-parameterized regime, and finally converge to a constant order (that depends on noise parameter $\tau^2$), which matches recent results relying on closed-form solution for (refined) variance, e.g., \cite{d2020double,adlam2020understanding,lin2020causes}.\\
\textit{ii}) When compared to least squares, our result can degenerate to this setting by choosing $m:=d$. Our upper bound is able to match the lower bound in \cite[Corollary 1]{hastie2019surprises} with the same order, which demonstrates the tightness of our upper bound.
Besides, our results can recover the result of \cite{zou2021benign} by taking the effective dimension $k^* = \min\{n,d\}$ (no data spectrum assumption is required here).
More discussion on our derived results refers to Appendix~\ref{app:setting}.

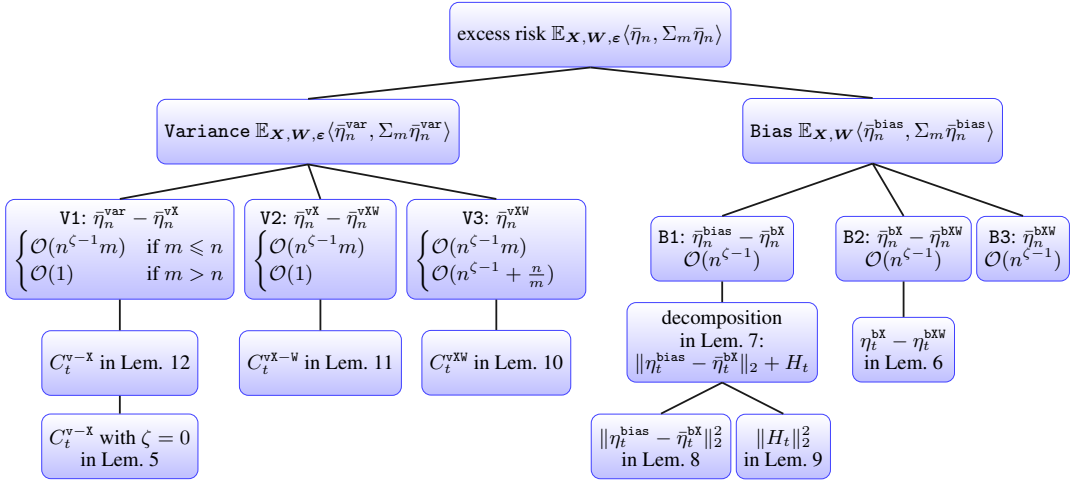
\begin{figure*}[t]
	\centering
	\tikzset{font=\small,
		%edge from parent fork down,
		edge from parent,
		every node/.style=
		{top color=white,
			bottom color=blue!25,
			rectangle,rounded corners,
			minimum height=10mm,
			draw=blue!75,
			%very thick,
			%drop shadow,
			align=center
		},
		edge from parent/.style=
		{draw=black!90,
			thick
	}}
	\tikzset{level 1/.style={level distance=1.5cm}}
	\tikzset{level 2/.style={level distance=2cm}}
	\tikzset{level 3/.style={level distance=1.6cm}}
	\tikzset{level 4/.style={level distance=1.5cm}}
	\tikzset{level 5/.style={level distance=1.2cm}}
	\scalebox{.86}{
		\begin{tikzpicture}
			\Tree [.{excess risk $\mathbb{E}_{\bm X, \bm W, \bm \varepsilon} \langle \bar{\eta}_n, \Sigma_m \bar{\eta}_n  \rangle$}
			[.{${\tt Variance}$ $\mathbb{E}_{\bm X, \bm W, \bm \varepsilon} \langle \bar{\eta}^{{\tt var}}_n, \Sigma_m \bar{\eta}^{{\tt var}}_n  \rangle$}
			[.{${\tt V1}$: $\bar{\eta}^{{\tt var}}_n - \bar{\eta}^{{\tt vX}}_n$ \\
				$\begin{cases}
					\mathcal{O}(n^{\zeta-1}m) & \mbox{if}~m \leqslant n \\
					\mathcal{O}(1) &  \mbox{if}~m > n \\
				\end{cases}$ } 
			[.{$C^{{\tt v-X}}_t $ in Lem. \ref{lemcv-xada}} {$C^{{\tt v-X}}_t$ with $\zeta=0$\\ in Lem.~\ref{dinfvx}}
			]
			]
			[.{${\tt V2}$: $\bar{\eta}^{{\tt vX}}_n - \bar{\eta}^{{\tt vXW}}_n$ \\
				$\begin{cases}
					\mathcal{O}(n^{\zeta-1}m) \\
					\mathcal{O}(1) \\
				\end{cases}$ } 
			[.{$C^{{\tt vX-W}}_t $ in Lem.~\ref{lemcvx-wada}}
			]
			]
			[.{${\tt V3}$: $\bar{\eta}^{{\tt vXW}}_n$ \\
				$\begin{cases}
					\mathcal{O}(n^{\zeta-1}m) \\
					\mathcal{O}(n^{\zeta-1} + \frac{n}{m}) \\
				\end{cases}$}
			[.{$C^{{\tt vXW}}_t $ in Lem.~\ref{lemcvtwtada}}
			]
			] ] 
			[.{${\tt Bias}$ $\mathbb{E}_{\bm X, \bm W} \langle \bar{\eta}^{{\tt bias}}_n, \Sigma_m \bar{\eta}^{{\tt bias}}_n \rangle$}
			[.{${\tt B1}$: $\bar{\eta}^{{\tt bias}}_n - \bar{\eta}^{{\tt bX}}_n$ \\
				$\mathcal{O}(n^{\zeta-1})$ } 
			[.{decomposition\\ in Lem.~\ref{lemstorecada}:\\$\|{\eta}^{{\tt bias}}_t - \bar{\eta}^{{\tt bX}}_t\|_2 + H_t$} [.{$\|{\eta}^{{\tt bias}}_t - \bar{\eta}^{{\tt bX}}_t\|^2_2$\\ in Lem.~\ref{lemaxtb2ada}}  ] [.{$\| H_t\|^2_2$\\ in Lem.~\ref{lemahtb2ada}}  ]
			]
			]
			[.{${\tt B2}$: $\bar{\eta}^{{\tt bX}}_n - \bar{\eta}^{{\tt bXW}}_n$ \\
				$\mathcal{O}(n^{\zeta-1})$} {${\eta}^{{\tt bX}}_t - {\eta}^{{\tt bXW}}_t  $ \\in Lem.~\ref{lemintb2}}  ] 
			[.{${\tt B3}$: $\bar{\eta}^{{\tt bXW}}_n$ \\
				$\mathcal{O}(n^{\zeta-1})$} ]]
			]
		\end{tikzpicture}
	}
	\caption{The roadmap of proofs.}\label{figproof}
\end{figure*}

\vspace{-0.05cm}
\section{Proof outline and discussion}
\vspace{-0.05cm}
\label{sec:proofoutline}

In this section, we first introduce the structure of the proofs with high level ideas, and then discuss our work with previous literature in terms of the used techniques and the obtained results.

\vspace{-0.15cm}
\subsection{Proof outline}
We (partly) disentangle the multiple randomness sources on the data $\bm X$, the random features matrix $\bm W$, the noise $\bm \varepsilon$, make full use of statistical properties of covariance operators $\Sigma_m$ and $\widetilde{\Sigma}_m$ in Section~\ref{sec:statcov}, and provide the respective (bias and variance) upper bounds in terms of multiple randomness sources, as shown in Figure~\ref{figproof}.
%We work in the adaptive step-size setting, which naturally holds for the constant step-size setting by taking $\zeta =0$.

{\bf Bias}: To bound ${\tt Bias}$, we need some auxiliary notations. Recall $\Sigma_m = \mathbb{E}_{\bm x} [\varphi(\bm x)\otimes \varphi(\bm x)]$ and $\widetilde{\Sigma}_m = \mathbb{E}_{\bm x, \bm W} [\varphi(\bm x)\otimes \varphi(\bm x)]$, define
\begin{equation}\label{eq:bias_xada}
	\begin{split}
			\eta_t^{{\tt bX}} = (I - \gamma_t \Sigma_m ) \eta_{t-1}^{{\tt bX}} , \quad  \eta_0^{{\tt bX}} = f^*\,, \qquad	\eta_t^{{\tt bXW}} = (I - \gamma_t \widetilde{ \Sigma}_m ) \eta_{t-1}^{{\tt bXW}} , \quad \eta_0^{{\tt bXW}} = f^*\,,
	\end{split}
\end{equation}
with the average $\bar{\eta}_n^{{\tt bX}} := \frac{1}{n} \sum_{t=0}^{n-1} \bar{ \eta}_t^{{\tt bX}}$ and $\bar{\eta}_n^{{\tt bXW}} := \frac{1}{n} \sum_{t=0}^{n-1} \bar{\eta}_t^{{\tt bXW}}$.
Accordingly, $\eta_t^{{\tt bX}}$ can be regarded as a ``deterministic'' version of $\eta_t^{{\tt bias}}$: we omit the randomness on $\bm X$ (data sampling, stochastic gradients) by replacing $[\varphi(\bm x)\varphi(\bm x)^{\!\top}]$ with its expectation ${\Sigma}_m$. Likewise, $\eta_t^{{\tt bXW}}$ is a deterministic version of $\eta_t^{{\tt vX}}$ by replacing $\Sigma_m$ with its expectation $\widetilde{\Sigma}_m$ (randomness on initialization).

By Minkowski inequality, the ${\tt Bias}$ can be decomposed as ${\tt Bias} \lesssim {\tt B1} + {\tt B2} + {\tt B3}$, where ${\tt B1} := \mathbb{E}_{\bm X, \bm W} \big[ \langle \bar{\eta}^{{\tt bias}}_n - \bar{\eta}^{{\tt bX}}_n, \Sigma_m ( \bar{\eta}^{{\tt bias}}_n - \bar{\eta}^{{\tt bX}}_n ) \rangle \big]$ and ${\tt B2} := \mathbb{E}_{\bm W} \big[ \langle \bar{\eta}^{{\tt bX}}_n \!-\! \bar{\eta}^{{\tt bXW}}_n, \Sigma_m ( \bar{\eta}^{{\tt bX}}_n \!-\! \bar{\eta}^{{\tt bXW}}_n ) \rangle \big]$ and ${\tt B3} := \langle \bar{\eta}^{{\tt bXW}}_n, \widetilde{\Sigma}_m \bar{\eta}^{{\tt bXW}}_n \rangle$.
Here ${\tt B3}$ is a deterministic quantity that is closely connected to model (intrinsic) bias without any randomness; while ${\tt B1}$ and ${\tt B2}$ evaluate the effect of randomness from $\bm X$ and $\bm W$ on the bias, respectively.
The error bounds for them can be directly found in Figure~\ref{figproof}. 

To bound ${\tt B3}$, we directly focus on its formulation by virtue of spectrum decomposition and integral estimation.
To bound ${\tt B2}$, we have
${\tt B2} = \frac{1}{n^2} \mathbb{E}_{\bm W} \Big\| \Sigma_m^{\frac{1}{2}}  \sum_{t=0}^{n-1} ({\eta}^{{\tt bX}}_t - {\eta}^{{\tt bXW}}_t) \Big\|^2$, where the key part ${\eta}^{{\tt bX}}_t - {\eta}^{{\tt bXW}}_t$ can be estimated by Lemma~\ref{lemintb2}.
To bound ${\tt B1}$, it can be further decomposed as (here we use inaccurate expression for description simplicity) ${\tt B1} \lesssim \sum_t \|{\eta}^{{\tt bX}}_t - {\eta}^{{\tt bXW}}_t \|^2_2 + \sum_t \mathbb{E}_{\bm X} \| H_t \|^2$ in Lemma~\ref{lemstorecada}, where $H_{t-1} := [{\Sigma}_m - \varphi(\bm x_t) \otimes \varphi(\bm x_t) ] {\eta}^{{\tt bX}}_{t-1}$.
The first term can be upper bounded by $\sum_t \|{\eta}^{{\tt bX}}_t - {\eta}^{{\tt bXW}}_t \|^2_2 \lesssim \mathrm{Tr}(\Sigma_m) n^{\zeta} \| f^*\|^2$ in Lemma~\ref{lemaxtb2ada}, and the second term admits $\sum_t \mathbb{E}_{\bm X} \| H_t \|^2 \lesssim \mathrm{Tr}(\Sigma_m) \| f^*\|^2$ in Lemma~\ref{lemahtb2ada}.

{\bf Variance}: To bound ${\tt Variance}$, we need some auxiliary notations.
\begin{align}
	\eta_t^{{\tt vX}} := (I - \gamma_t \Sigma_m ) \eta_{t-1}^{{\tt vX}} + \gamma_t \varepsilon_t \varphi(\bm x_t) , ~~~ \eta_0^{{\tt vX}} = 0\,,\label{eq:var_xada} \\
	\eta_t^{{\tt vXW}} := (I - \gamma_t \widetilde{ \Sigma}_m ) \eta_{t-1}^{{\tt vXW}} +\gamma_t \varepsilon_t \varphi(\bm x_t) , ~~~ \eta_0^{{\tt vXW}} = 0\,, 		\label{eq:var_xwada}
\end{align}
with the averaged quantities
$\bar{\eta}_n^{{\tt vX}} := \frac{1}{n} \sum_{t=0}^{n-1} \bar{ \eta}_t^{{\tt vX}}$, $ \bar{\eta}_n^{{\tt vXW}} := \frac{1}{n} \sum_{t=0}^{n-1} \bar{\eta}_t^{{\tt vXW}}$.
Accordingly, $\eta_t^{{\tt vX}}$ can be regarded as a ``semi-stochastic'' version of $\eta_t^{{\tt var}}$: we keep the randomness due to the noise $\varepsilon_t$ but omit the randomness on $\bm X$ (data sampling) by replacing $[\varphi(\bm x)\varphi(\bm x)^{\!\top}]$ with its expectation ${\Sigma}_m$. Likewise, $\eta_t^{{\tt vXW}}$ can be regarded as a ``semi-stochastic'' version of $\eta_t^{{\tt vX}}$ by replacing $\Sigma_m$ with its expectation $\widetilde{\Sigma}_m$ (randomness on initialization).

By virtue of Minkowski inequality, the ${\tt Variance}$ can be decomposed as ${\tt Variance} \lesssim {\tt V1} + {\tt V2} + {\tt V3}$, where ${\tt V1} := \mathbb{E}_{\bm X, \bm W,\bm \varepsilon} \big[ \langle \bar{\eta}^{{\tt var}}_n - \bar{\eta}^{{\tt vX}}_n, \Sigma_m ( \bar{\eta}^{{\tt var}}_n - \bar{\eta}^{{\tt vX}}_n ) \rangle \big]$, ${\tt V2} := \mathbb{E}_{\bm X, \bm W, \bm \varepsilon} \big[ \langle \bar{\eta}^{{\tt vX}}_n \!-\! \bar{\eta}^{{\tt vXW}}_n, \Sigma_m ( \bar{\eta}^{{\tt vX}}_n \!-\! \bar{\eta}^{{\tt vXW}}_n ) \rangle \big] $, and ${\tt V3} := \mathbb{E}_{\bm X, \bm W, \bm \varepsilon} \langle   \bar{\eta}^{{\tt vXW}}_n, {\Sigma}_m \bar{\eta}^{{\tt vXW}}_n \rangle$.
Though ${\tt V1}$, ${\tt V2}$, ${\tt V3}$ still interact the multiple randomness, ${\tt V1}$ disentangles some randomness on data sampling, ${\tt V2}$ discards some randomness on initialization, and ${\tt V3}$ focuses on the ``minimal'' interaction between data sampling, label noise, and initialization.  The error bounds for them can be found in Figure~\ref{figproof}. 

To bound ${\tt V3}$, we focus on the formulation of the covariance operator $C^{{\tt vXW}}_t := \mathbb{E}_{\bm X, \bm \varepsilon} [{\eta}^{{\tt vXW}}_t \otimes {\eta}^{{\tt vXW}}_t]$ in Lemma~\ref{lemcvtwtada} and the statistical properties of $\widetilde{\Sigma}_m$ and $\Sigma_m$.
To bound ${\tt V2}$, we need study the covariance operator $C^{{\tt vX-W}}_t := \mathbb{E}_{\bm X, \bm \varepsilon} [({\eta}^{{\tt vX}}_t - {\eta}^{{\tt vXW}}_t) \otimes ({\eta}^{{\tt vX}}_t - {\eta}^{{\tt vXW}}_t)]$ admitting $\|C^{{\tt vX-W}}_t\| \lesssim \| \Sigma_m^2 \|_2 \| \widetilde{\Sigma}_m  \|_2$ in Lemma~\ref{lemcvx-wada}.
To bound ${\tt V1}$, we need study the covariance operator $C^{{\tt v-X}}_t := \mathbb{E}_{\bm X, \bm \varepsilon} [({\eta}^{{\tt var}}_t - {\eta}^{{\tt vX}}_t) \otimes ({\eta}^{{\tt var}}_t - {\eta}^{{\tt vX}}_t)]$, as a function of $\zeta \in [0,1)$, admitting $\mathrm{Tr}[C^{{\tt v-X}}_t(\zeta)] \lesssim \mathrm{Tr}[C^{{\tt v-X}}_t(0)]$ in Lemma~\ref{dinfvx}, and further
$C^{{\tt v-X}}_t  \precsim \mathrm{Tr}(\Sigma_m) I$ in Lemma~\ref{lemcv-xada}.

\vspace{-0.05cm}
\subsection{Discussion on techniques}
\vspace{-0.05cm}

Our proof framework follows \cite{dieuleveut2016nonparametric} that focuses on kernel regression with stochastic approximation in the under-parameterized regimes ($d$ is regarded as finite and much smaller than $n$).
Nevertheless, even in the under-parameterized regime, their results can not be directly extended to random features model due to the extra randomness on $\bm W$.
%coupling with other randomness sources on noise and data sampling, which makes their proof framework invalid: 
For instance, their results depend on \cite[Lemma 1]{bach2013non} by taking conditional expectation to bridge the connection between $\mathbb{E}[\| \alpha_t\|_2]$ and $\mathbb{E} \langle \alpha_t, \Sigma_m \alpha_t \rangle$. This is valid for ${\tt B1}$ but expires on other quantities.

Some technical tools used in this paper follow \cite{zou2021benign} that focuses on linear regression with constant step-size SGD for benign overfitting.
%However, coping with multiple randomness sources and adaptive step-size setting (no longer a homogeneous markov chain) make our analysis intractable and largely different.
However, our results differ from it in
1) tackling multiple randomness, e.g., stochastic gradients, random features (Gaussian initialization), by introducing another type of error decomposition and several deterministic/randomness covariance operators. We prove nice statistical properties of them for proof, which gets rid of data spectrum assumption in \cite{zou2021benign}. 2) tackling non-constant step-size SGD setting by introducing new integral estimation techniques. Original techniques on constant step-size in \cite{zou2021benign} are invalid due to non-homogeneous update rules.
The above two points make our proof relatively more intractable and largely different.
Besides, their results demonstrate that linear regression with SGD generalizes well (converges with $n$) but has few findings on double descent. 
Instead, our result depends on $n$ and $m$ (where $d$ is implicitly included in $m$), and is able to explain double descent.

Here we take the estimation for the variance in \cite{zou2021benign} under the least squares setting as an example to illustrate this.
\begin{equation*}
    \text{Variance} \lesssim \sum_{t=0}^{n-1} \Big\langle  I - (I-\gamma \Sigma_d)^{n-t}, I - (I - \gamma \Sigma_d)^t  \Big\rangle \qquad [\mbox{Eq. (4.10) in \cite{zou2021benign}}]
\end{equation*}
In this setting, the effective dimension to tackle $I - (I-\gamma \Sigma_d)^{n-t}$; while our result is based on fast eigenvalue decay of $\widetilde{\Sigma}_m$ in Lemma~\ref{thmH} can direct to bound this.
Besides, the homogeneous markov chain under the constant step-size setting is employed  \cite{zou2021benign} for $(I-\gamma \Sigma_d)^{n-t}$, which is naturally invalid under our decaying step-size setting. Instead, we introduce integral estimation techniques to tackle adaptive step-size, see Appendix \ref{app:int} for details.

\vspace{-0.15cm}
\section{Numerical Validation}
\label{sec:exp}
\vspace{-0.15cm}

In this section, we provide some numerical experiments in Figure~\ref{fig-gaussiankernel} to support our theoretical results and findings.
Note that our results go beyond Gaussian data assumption and can be empirically validated on real-world datasets. 
More experiments can be found in Appendix~\ref{app:experiment}.

\begin{figure*}[t]
	\centering
	\subfigure[SGD vs. min-norm solution]{\label{fig:opt}
		\includegraphics[width=0.28\linewidth]{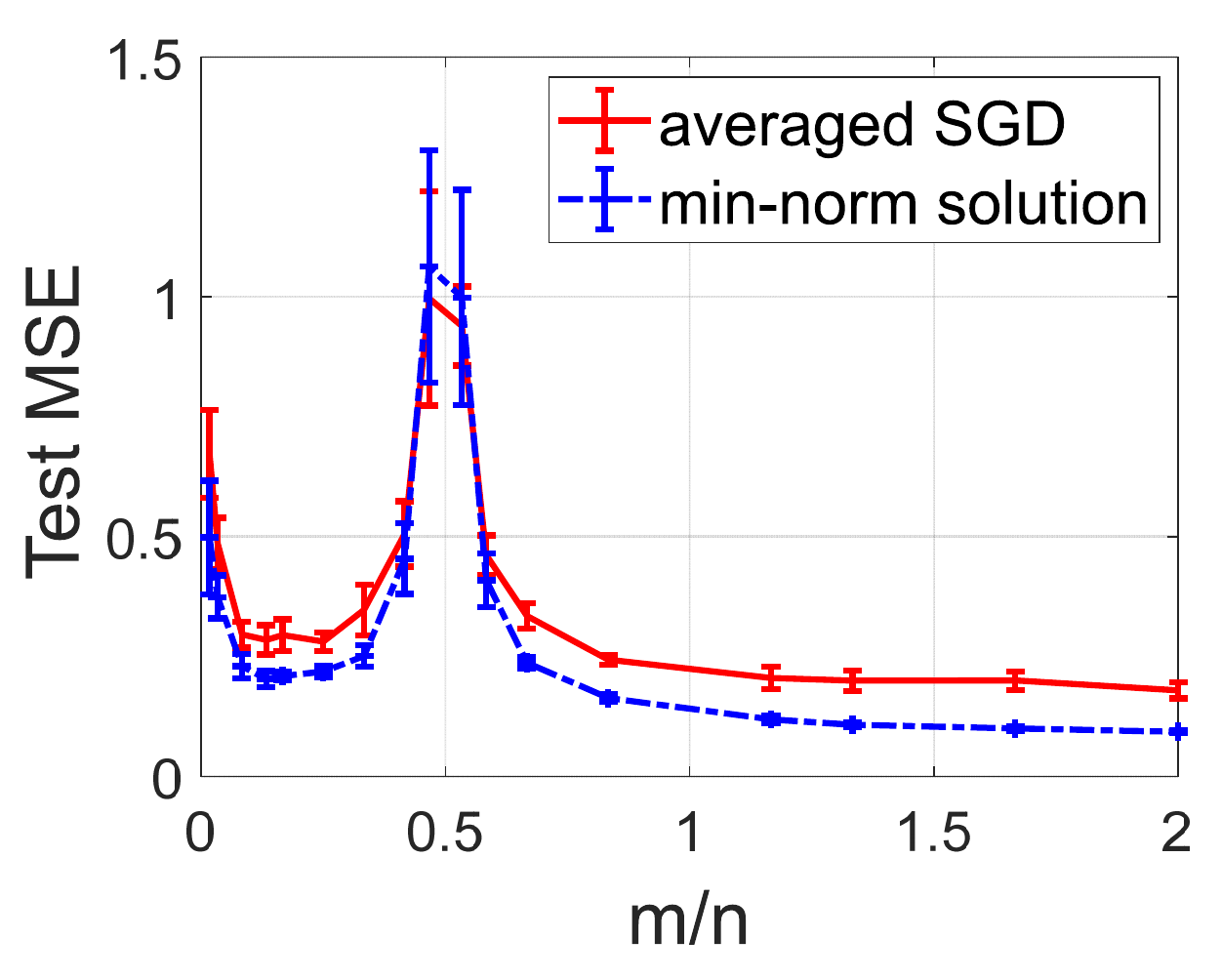}}
	\subfigure[${\tt Bias}$]{\label{fig:bias}
		\includegraphics[width=0.28\linewidth]{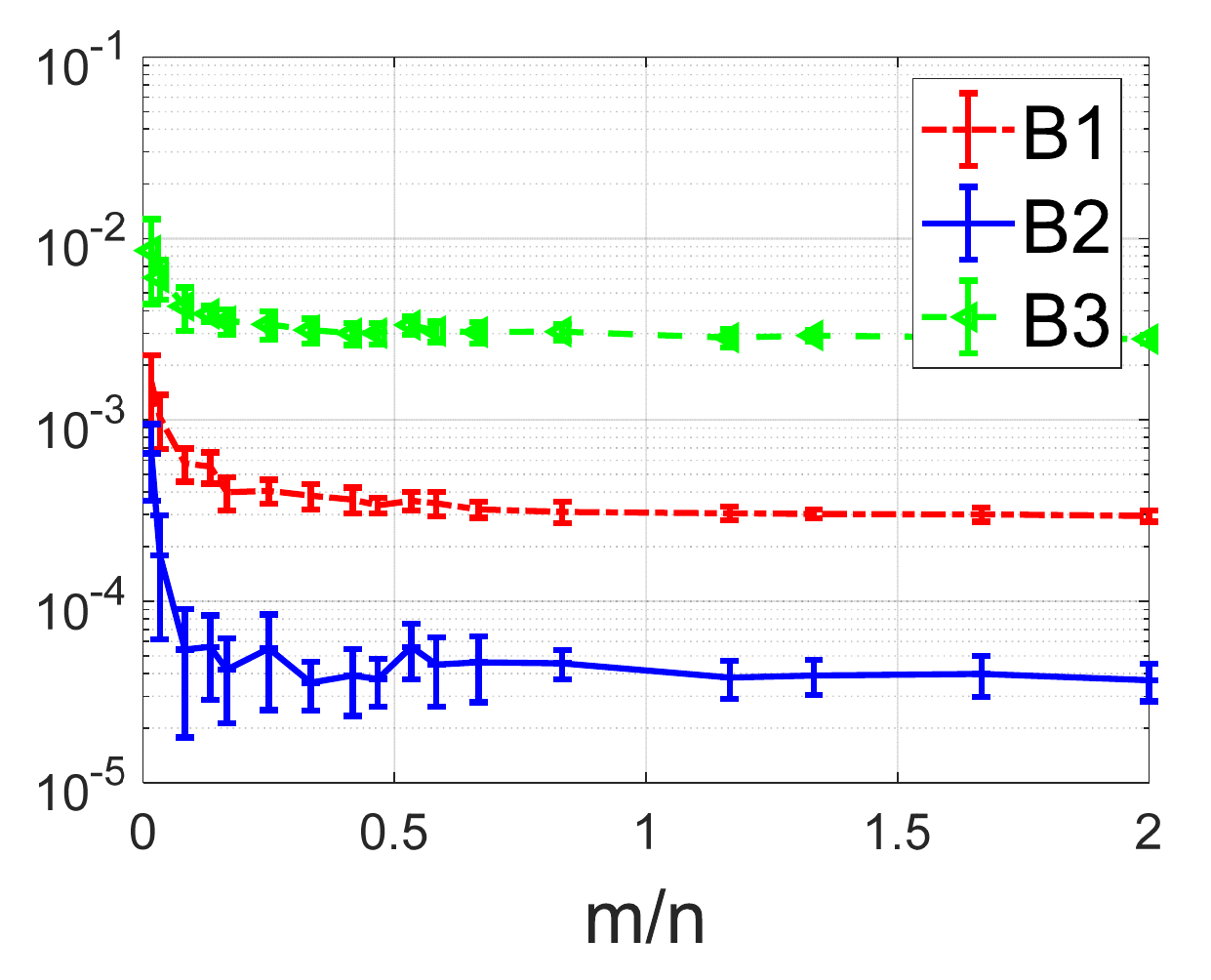}}
	\subfigure[${\tt Variance}$]{\label{fig:variance}
		\includegraphics[width=0.28\linewidth]{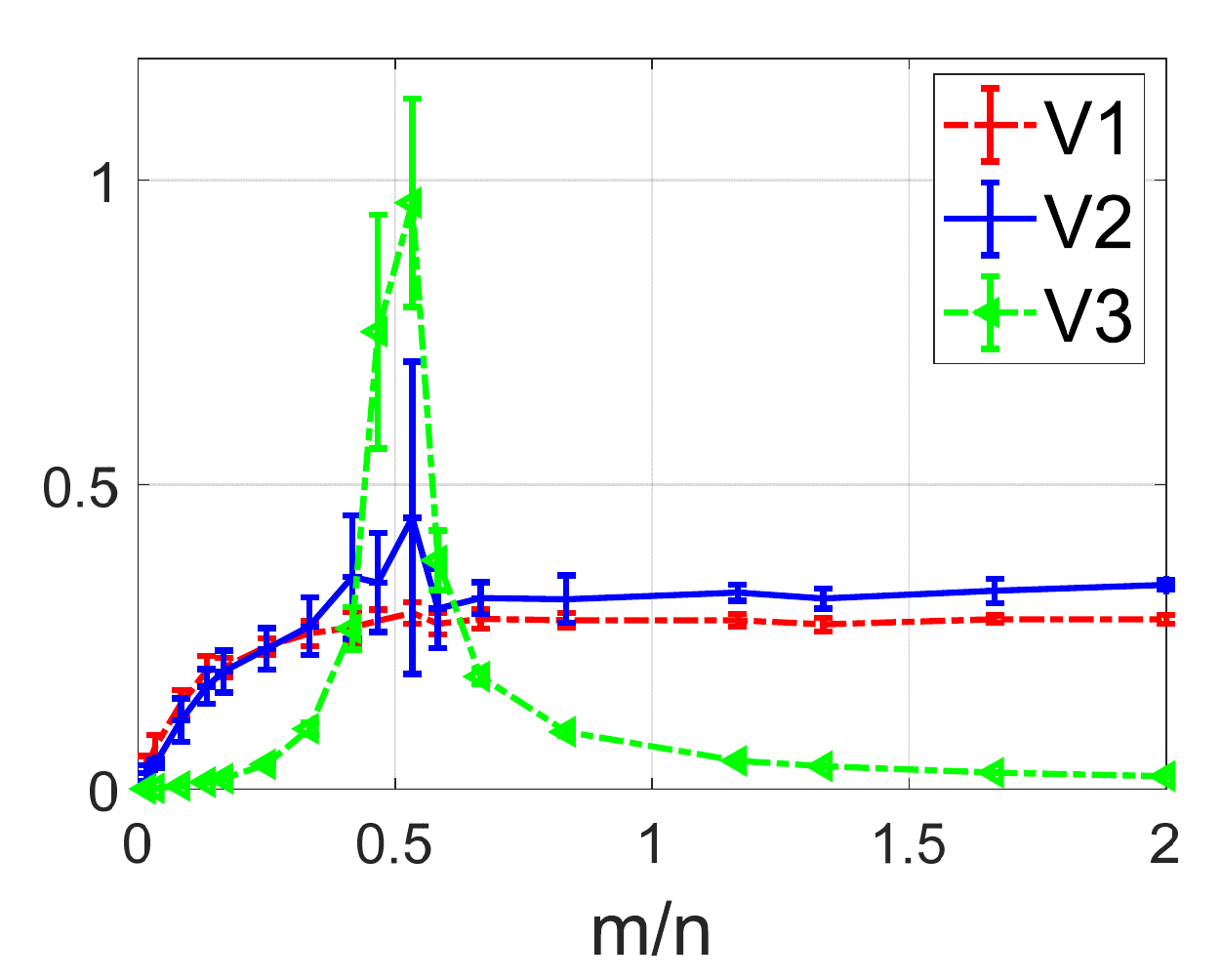}}
	\caption{Test MSE (mean$\pm$std.) of RF regression as a function of the ratio $m/n$ on MNIST data set (digit 3 vs. 7) across the Gaussian kernel, for $d=784$ and $n=600$ in (a). The interpolation threshold occurs at $m/n=0.5$ as the Gaussian kernel outputs the $2m$-feature mapping (instead of $m$), i.e., $\sigma(\bm W \bm x) \in \mathbb{R}^{2m}$. Under this setting, the trends of ${\tt Bias}$ and ${\tt Variance}$ are empirically given in (b) and (c).  }\label{fig-gaussiankernel}
	\vspace{-0.2cm}
\end{figure*}

\vspace{-0.15cm}
\subsection{Behavior of RF for interpolation learning}
\vspace{-0.15cm}
\label{sec:expdd}
Here we evaluate the test mean square error (MSE) of RFF regression on the MNIST data set \cite{L1998Gradient}, following the experimental setting of \cite{liao2020random,derezinski2020exact}, to study the generalization performance of minimum-norm solution, see Figure~\ref{fig:opt}.
More results on regression dataset refer to Appendix~\ref{app:experiment}.

{\bf Experimental settings:} We take digit 3 vs. 7 as an example, and randomly select 300 training data in these two classes, resulting in $n=600$ for training.
Hence, our setting with $n=600$, $d=784$, and tuning $m$ satisfies our realistic high dimensional assumption. The Gaussian kernel $k(\bm x, \bm x') =\exp(-\| \bm x - \bm x'\|_2^2/(2\sigma_0^2))$ is used, where the kernel width $\sigma_0$ is chosen as $\sigma_0^2 = d$ in high dimensional settings such that $\| \bm x \|^2_2/d \sim \mathcal{O}(1)$ in Assumption~\ref{assumdata}. 
In our experiment, the initial step-size is set to $\gamma_0=1$ and we take the initial point $\bm \theta_0$ near the min-norm solution\footnote{In our numerical experiments, we only employ single-pass SGD, and thus the initialization is chosen close to minimum norm solution, with more discussion in Appendix~\ref{app:experiment}.} corrupted with zero-mean, unit-variance Gaussian noise. The experiments are repeated 10 times and 
the test MSE (mean$\pm$std.) can be regarded as a function of the ratio $m/n$ by tuning $m$.
Results on different initialization and more epochs of SGD refer to Appendix~\ref{app:experiment}. 

{\bf SGD vs. minimal-norm solution:} Figure~\ref{fig:opt} shows the test MSE of RF regression with averaged SGD (we take $\zeta = 0.5$ as an example; {\color{red}red} line) and minimal-norm solution ({\color{blue}blue} line). First, we observe the double descent phenomenon: a phase transition on the two sides of the interpolation threshold at $2m = n$ when these two algorithms are employed. 
Second, in terms of test error, RF with averaged SGD is slightly inferior to that with min-norm solution, but still generalizes well.
%Accordingly, there is no loss for RFF regression with SGD on generalization properties when compared to the min-norm solution.

\vspace{-0.2cm}
\subsection{Behavior of our error bounds}
\label{sec:expbiasvar}
We have experimentally validated the phase transition and corresponding double descent in the previous section, and here we aim to semi-quantitatively assess our derived bounds for ${\tt Bias}$ and ${\tt Variance}$, see Figure~\ref{fig:bias} and \ref{fig:variance}, respectively.
Results of these quantities on different step-size refer to Appendix~\ref{app:experiment}.

{\bf Experimental settings:} Since the target function $f^*$, the covariance operators $\Sigma_d$, $\Sigma_m$, and the noise $\bm \varepsilon$ are unknown on the MNIST data set, our experimental evaluation need some assumptions to calculate ${\tt Bias}$ and ${\tt Variance}$.
First, we assume the label noise $\varepsilon \sim \mathcal{N}(0,1)$, which can in turn obtain $f^*(\bm x)$ on both training and test data due to $f^*(\bm x) = y - \varepsilon$. 
Second, the covariance matrices $\Sigma_d$ and $\Sigma_m$ are estimated by the related sample covariance matrices. 
When using the Gaussian kernel, the covariance matrix $\widetilde{\Sigma}_m$ can be directly computed, see the remark in Lemma~\ref{thmH}, where the expectation on $\bm x$ is approximated by Monte Carlo sampling with $n$ training samples.
Accordingly, based on the above results, we are ready to calculate $\eta_t^{{\tt bias}}$ in Eq.~\eqref{eq:bias_iterates}, $\eta_t^{{\tt bX}}$, and $\eta_t^{{\tt bXW}}$ in Eq.~\eqref{eq:bias_xada}, respectively, which is further used to approximately compute ${\tt B1} := \mathbb{E}_{\bm X, \bm W} \big[ \langle \bar{\eta}^{{\tt bias}}_n - \bar{\eta}^{{\tt bX}}_n, \Sigma_m ( \bar{\eta}^{{\tt bias}}_n - \bar{\eta}^{{\tt bX}}_n ) \rangle \big]$ ({\color{red}red} line) and ${\tt B2} := \mathbb{E}_{\bm W} \big[ \langle \bar{\eta}^{{\tt bX}}_n \!-\! \bar{\eta}^{{\tt bXW}}_n, \Sigma_m ( \bar{\eta}^{{\tt bX}}_n \!-\! \bar{\eta}^{{\tt bXW}}_n ) \rangle \big]$ ({\color{blue}blue} line) and ${\tt B3} := \langle \bar{\eta}^{{\tt bXW}}_n, \widetilde{\Sigma}_m \bar{\eta}^{{\tt bXW}}_n \rangle$ ({\color{green}green} line).
The (approximate) computation for ${\tt Variance}$ can be similar achieved by this process.

{\bf Error bounds for bias:} Figure~\ref{fig:bias} shows the trends of (scaled) ${\tt B1}$, ${\tt B2}$, and ${\tt B3}$.
Recall our error bound: ${\tt B1}$, ${\tt B2}$, ${\tt B3}\sim \mathcal{O}(n^{\zeta-1})$, we find that, all of them monotonically decreases at a certain convergence rate when $m$ increases from the under-parameterized regime to the over-parameterized regime. 
%When $m<n$, ${\tt B2}$ converges faster than ${\tt B1}$ and ${\tt B3}$.
These experimental results coincide with our error bound on them.
%, i.e., converging with $n$ at some certain rate ($m$ and $n$ are in the same order in our experiment).

{\bf Error bounds for variance:} Figure~\ref{fig:variance} shows the trends of (scaled) ${\tt V1}$, ${\tt V2}$, and ${\tt V3}$.
Recall our error bound: in the under-parameterized regime, ${\tt V1}$, ${\tt V2}$, and ${\tt V3}$ increase with $m$ at a certain $\mathcal{O}(n^{\zeta-1}m)$ rate; and in the over-parameterized regime, ${\tt V1}$ and ${\tt V2}$ are in $\mathcal{O}(1)$ order while ${\tt V3}$ decreases with $m$.
Figure~\ref{fig:variance} shows that, when $2m<n$, ${\tt V1}$ and ${\tt V2}$ monotonically increases with $m$ and then remain unchanged when $2m > n$. 
Besides, ${\tt V3}$ is observed to be unimodal: firstly increasing when $2m<n$, reaching to the peak at $2m=n$, and then decreasing when $2m>n$, which admits the phase transition at $2m = n$.
Accordingly, these findings accord with our theoretical results, and also matches refined results in \cite{d2020double,adlam2020understanding,lin2020causes}: the unimodality of variance is a prevalent phenomenon.

\vspace{-0.3cm}
\section{Conclusion}
\vspace{-0.25cm}
We present non-asymptotic results for RF regression under the averaged SGD setting for understanding double descent under the optimization effect.
Our theoretical and empirical results demonstrate that, the error bounds for variance and bias can be unimodal and monotonically decreasing, respectively, which is able to recover the double descent phenomenon.
Regarding to constant/adaptive step-size setting, there is no difference between the constant step-size case and the exact minimal-norm solution on the convergence rate; while the polynomial-decay step-size case will slow down the learning rate, but does not change the error bound for variance in over-parameterized regime that converges to $\mathcal{O}(1)$ order, that depends on noise parameter(s).

Our work centers around the RF model, which is still a bit far away from practical neural networks. 
Theoretical understanding the generalization properties of over-parameterized neural networks is a fundamental but difficult problem.
We believe that a comprehensive and thorough understanding of shallow neural networks, e.g., the RF model, is a necessary first step.
Besides, we consider the single-pass SGD in our work for simplicity rather than multiple-pass SGD used in practice.
This is also an interesting direction for understanding the optimization effect of SGD in the double descent.

\section*{Acknowledgment}
The research leading to these results has received funding from the European Research Council under the European Union's Horizon 2020 research and innovation program: ERC Advanced Grant E-DUALITY (787960) and grant agreement n$^\circ$ 725594 - time-data. 
This paper reflects only the authors' views and the Union is not liable for any use that may be made of the contained information.
This work was supported by SNF project – Deep Optimisation of the Swiss National Science Foundation (SNSF) under grant number 200021\_205011; Research Council KU Leuven: Optimization frameworks for deep kernel machines C14/18/068; Flemish Government: FWO projects: GOA4917N (Deep Restricted Kernel Machines: Methods and Foundations), PhD/Postdoc grant. This research received funding from the Flemish Government (AI Research Program). 
This work was supported in part by Ford KU Leuven Research Alliance Project KUL0076 (Stability analysis and performance improvement of deep reinforcement learning algorithms), EU H2020 ICT-48 Network TAILOR (Foundations of Trustworthy AI - Integrating Reasoning, Learning and Optimization), Leuven.AI Institute.\\
We also thank Zhenyu Liao and Leello Dadi for their helpful discussions on this work.

%% file: checklist.tex
\section*{Checklist}

\begin{enumerate}

	\item For all authors...
	\begin{enumerate}
		\item Do the main claims made in the abstract and introduction accurately reflect the paper's contributions and scope?
		\answerYes
		\item Did you describe the limitations of your work?
		\answerYes{We clearly discuss the limitation of this work in Conclusion.}
		\item Did you discuss any potential negative societal impacts of your work?
		\answerNo{Our work is theoretical and generally will have no negative societal impacts.}
		\item Have you read the ethics review guidelines and ensured that your paper conforms to them?
		\answerYes{}
	\end{enumerate}

	\item If you are including theoretical results...
	\begin{enumerate}
		\item Did you state the full set of assumptions of all theoretical results?
		\answerYes{The assumptions are clearly stated and well discussed.}
		\item Did you include complete proofs of all theoretical results?
		\answerYes{All of the proofs can be found in the Appendix.}
	\end{enumerate}

	\item If you ran experiments...
	\begin{enumerate}
		\item Did you include the code, data, and instructions needed to reproduce the main experimental results (either in the supplemental material or as a URL)?
		\answerYes
		\item Did you specify all the training details (e.g., data splits, hyperparameters, how they were chosen)?
		\answerYes
		\item Did you report error bars (e.g., with respect to the random seed after running experiments multiple times)?
		\answerYes
		\item Did you include the total amount of compute and the type of resources used (e.g., type of GPUs, internal cluster, or cloud provider)?
		\answerYes{}
	\end{enumerate}

	\item If you are using existing assets (e.g., code, data, models) or curating/releasing new assets...
	\begin{enumerate}
		\item If your work uses existing assets, did you cite the creators?
		\answerNA{}
		\item Did you mention the license of the assets?
		\answerNA{}
		\item Did you include any new assets either in the supplemental material or as a URL?
		\answerNA{}
		\item Did you discuss whether and how consent was obtained from people whose data you're using/curating?
		\answerNA{}
		\item Did you discuss whether the data you are using/curating contains personally identifiable information or offensive content?
		\answerNA{}
	\end{enumerate}

	\item If you used crowdsourcing or conducted research with human subjects...
	\begin{enumerate}
		\item Did you include the full text of instructions given to participants and screenshots, if applicable?
		\answerNA{}
		\item Did you describe any potential participant risks, with links to Institutional Review Board (IRB) approvals, if applicable?
		\answerNA{}
		\item Did you include the estimated hourly wage paid to participants and the total amount spent on participant compensation?
		\answerNA{}
	\end{enumerate}

\end{enumerate}

%% file: appendix.tex
\newpage

\appendix

The outline of the appendix is stated as follows. 
\begin{itemize}
	\item Appendix~\ref{app:setting} summarizes representative results on random features regarding to double descent under various settings and discusses the tightness of our derived upper bounds.
	\item Appendix~\ref{app:assumptions}
presents more discussion on the used assumptions, especially  Assumptions~\ref{assexist} and~\ref{assump:bound_fourthmoment}, demonstrating the rationale behind these assumptions.
	\item Appendix~\ref{app:rescov} provides the proofs of lemmas in Section~\ref{sec:statcov} on statistical properties of $\Sigma_m$ and $\widetilde{\Sigma}_m$.
	\item Appendix~\ref{app:psd} introduces preliminaries on PSD operators in stochastic approximation.
	\item Appendix~\ref{app:int} provides estimation for several typical integrals that are needed for our proof.
	\item Appendix~\ref{app:bias} gives error bounds for ${\tt Bias}$.
	\item Appendix~\ref{app:variance} provides the error bounds for ${\tt Variance}$. 
	\item Appendix~\ref{app:experiment} provides more experiments including different initialization, step-size on various datasets to support our theory.
 \end{itemize}

\section{Comparisons with previous work}
\label{app:setting}

\subsection{Problem settings}
Here we summarize various representative approaches in Table~\ref{Tabsetting} according to the used data assumption, the type of solution, and the derived results.

\begin{table*}[!htb]
	% \centering
	\fontsize{9}{8}\selectfont
	\begin{threeparttable}
		\caption{Comparison of problem settings on analysis of high dimensional random features on double descent.}
		\label{Tabsetting}
		\begin{tabular}{cccccccccccc}
			\toprule
			&data assumption & \centering{solution} & \centering{result} \cr
			\midrule
			\cite{hastie2019surprises} &Gaussian & closed-form & variance $\nearrow$ $\searrow$  \cr
			\midrule
			\cite{ba2020generalization} &Gaussian  & GD  & variance $\nearrow$ $\searrow$  \cr
			\midrule
			\cite{mei2019generalization}  &i.i.d on sphere & closed-form & variance, bias $\nearrow$ $\searrow$  \cr 
			\midrule
			\cite{d2020double} &Gaussian  & closed-form &  refined~\tnote{2}  \cr
			\midrule
			\cite{gerace2020generalisation}  &Gaussian  & closed-form &  $\nearrow$ $\searrow$ \cr
			\midrule
			\cite{adlam2020understanding}  &Gaussian & closed-form &  refined  \cr
			\midrule
			\cite{dhifallah2020precise}  &Gaussian & closed-form &  $\nearrow$ $\searrow$ \cr
			\midrule
			\cite{hu2020universality} &Gaussian  & closed-form &  $\nearrow$ $\searrow$ \cr
			\midrule
			\cite{liao2020random} &general & closed-form &  $\nearrow$ $\searrow$ \cr
			\midrule
			\cite{lin2020causes} &isotropic features with finite moments & closed form &  refined \cr
			\midrule
			\cite{li2021towards} &correlated features with polynomial decay on $\Sigma_d$ & closed form &  interpolation learning \cr
			\midrule
			Ours & sub-Gaussian data & SGD & variance $\nearrow$ $\searrow$, bias $\searrow$ \cr
			\bottomrule
		\end{tabular}
		\begin{tablenotes}
			\footnotesize
			\item[1] A refined decomposition on variance is conducted by sources of randomness on data sampling, initialization, label noise to possess each term
			\cite{d2020double} or their full decomposition in \cite{adlam2020understanding,lin2020causes}.
		\end{tablenotes}
	\end{threeparttable}
\end{table*}

Here we discuss the used assumption on data distribution and the discussion on other assumptions is deferred to \cref{app:assumptions}. It can be found that, most papers assume the data to be Gaussian or uniformly distributed on the sphere. The following papers admit weaker assumption on data.
Given a correlated features model that is commonly used in high dimensional statistics \cite{hastie2019surprises}
\begin{equation}\label{data:model}
	\bm x = \Sigma_d^{\frac{1}{2}} \bm t\,,\quad \mathbb{E}[t_i] = 0, \mathbb{V}[t_i] = 1, \quad  \mbox{with}~ \Sigma_d := \mathbb{E}_{\bm x} [\bm x \bm x^{\!\top}]\,,
\end{equation}
where $\bm t \in \mathbb{R}^d$ has i.i.d entries $t_i$ ($i=1,2,\dots,d$) with zero mean and unit variance.
In \cite{li2021towards}, they further require that each entry is i.i.d sub-Gaussian and $\Sigma_d$ admits polynomial decay on eigenvalues.
In \cite{lin2020causes}, the authors consider isotropic features with finite moment, i.e., taking $\Sigma_d := I$ in Eq.~\eqref{data:model} and $\mathbb{E}[t_i^{8+\eta}] < \infty$ for any arbitrary positive constant $\eta > 0$.
Our model holds for sub-Gaussian, and thus the used data assumption~\ref{assump:bound_fourthmoment} is weaker than them.
We also remark that, no assumption on data distribution is employed \cite{liao2020random} but they require that test data ``behave'' statistically like the training data by concentrated random vectors. 
Indeed, their data assumption is weaker than ours, but their analysis framework builds on the exact closed-form solution from random matrix theory.
Instead, we focus on the SGD setting and thus take a unified perspective on optimization and generalization.

Here we briefly discuss our result with previous work.
Compared to \cite{ba2020generalization} on RF optimized by gradient descent under the Gaussian data in an asymptotic view, our non-asymptotic result holds for more general data distribution under the SGD setting.
In fact, our data assumption is weaker than most previous work assuming the data to be Gaussian,  uniformly spread on a sphere, or isotropic/correlated features (with spectral decay assumption), except \cite{liao2020random}. 
Nevertheless, we extend their asymptotic results relying on the least-squares closed-form solution to non-asymptotic results under the SGD setting, which takes the effect of optimization into consideration.
Besides, our result coincides several findings with refined variance decomposition in \cite{d2020double,adlam2020understanding,lin2020causes}, e.g., the interaction effect can dominate the variance (between samples and initialization); the unimodality of variance is a prevalent phenomenon.

\subsection{Discussion on the tightness of our results}
\label{app:limitation}

We present the upper bounds of excess risk in this work, and it is natural to ask whether the lower bound can be derived by our proof framework.
Unfortunately, the first step in our proof is based on Minkowski inequality such that ${\tt Bias} \leqslant 3({\tt B1} + {\tt B2} + {\tt B3})$ and ${\tt Variance} \leqslant 3({\tt V1} + {\tt V2} + {\tt V3})$.
This could be a limitation of this work, but our derived results are still tight when compared to
previous work in both under- and over-parameterized regimes.

First, we compare our result with 
classical random features regression with SGD in the under-parameterized regime \cite{carratino2018learning}.
Under the same standard assumptions, e.g., $f^* \in H$ and label noise with bounded variance,
without refined assumptions, e.g., source condition describing the smoothness of $f^*$ and capacity condition describing the ``size'' of the corresponding $\mathcal{H}$ \cite{Rudi2017Generalization}, by taking one-pass over the data (the same setting with our result) and the random features $m=\mathcal{O}(\sqrt{n})$, the excess risk \cite{carratino2018learning} achieves at a certain $\mathcal{O}(1/\sqrt{n})$ rate. 
Under the same setting with the constant-step size, i.e., $\gamma = 0$, we have
\begin{equation*}
    \mathbb{E}\| \bar{f}_n  - f^* \|^2_{L^2_{\rho_X}} = \underbrace{{\tt Bias}}_{\mathcal{O}(\frac{1}{n})} + \underbrace{{\tt Variance}}_{\mathcal{O}(\frac{1}{\sqrt{n}})} \lesssim \frac{1}{\sqrt{n}}\,,
\end{equation*}
which achieves the same learning rate with \cite{carratino2018learning}, and has been proved to be optimal in a minimax sense \cite{caponnetto2007optimal} under the standard assumptions.
That means, the constant step-size SGD setting incurs no loss in convergence rate when compared to the exact kernel ridge regression.

Second, in the over-parameterized regime, previous work using random matrix theory and replica method provide an exact formulation of the excess risk.
Nevertheless, it appears difficult to compare the specific convergence rate due to their complex formulations, and thus we in turn study the tendency.
Here we take \cite{d2020double} as an example for comparison.
They use conditional expectations to split the variance into label noise, initialization, and data sampling, and the first two terms dominates the variance.\\
$(i)$ Our result on bias matches their exact formulation, i.e., monotonically decreasing bias. One slight difference is, their result on bias tends to a constant under the over-parameterized regime while our bias result can converge to zero.\\
$(ii)$ Our result on variance admits the same tendency with their result, leading to unimodal variance, where some part(s) are with phase transition and some part(s) firstly monotonically increase during the under-parameterized regime and then remain unchanged during the over-parameterized regime.
More importantly, both of the above two results demonstrate that, the variance will finally converge to a constant order, that depends on the variance of label noise $\tau^2$.
That means, our (upper bound) result is tight to describe phase transition and the final convergence state (depending on the noise level) when compared to the exact formulation results. 

Third, though convergence rates of random features for double descent is non-easy to compare, results on
least squares \cite{bartlett2020benign,zou2021benign} under the over-parameterized regime or interpolation are possible for comparison.
Here we take our result by choosing $m:=d$ and the constant step-size  for least squares setting, and compare their lower bound results to demonstrate the tightness of our result.
By virtue of Lemma~\ref{thmH}, we can reformulate our result as
	\begin{equation*}
		\begin{split}
		  \mathbb{E}\| \bar{f}_n  - f^* \|^2_{L^2_{\rho_X}}
			& \lesssim \gamma_0 \tau^2 \left\{ \begin{array}{rcl}
				\begin{split}
					&   \frac{1}{n} + \frac{d}{n} ,~\mbox{if $d \leqslant n$} \\
					&  1+\frac{1}{n} + \frac{n}{d} ,~\mbox{if $d > n$}  \\
				\end{split}
			\end{array} \right. 
		\end{split}
	\end{equation*}
	which matches the same order with \cite[Corollary 1]{hastie2019surprises}.

Besides, when compared to \cite{zou2021benign}, if taking the effective dimension $k^* = \min\{n,d\}$ (no data spectrum assumption is required here), we can recover their result.
In fact, our result is able to match their lower bound \cite{bartlett2020benign,zou2021benign}: ${\tt excess~risk} \gtrsim \tau^2 (\frac{1}{n} + \frac{n}{d})$ with only one difference on an extra constant when $d > n$.
	
	Based on the above discussion, our upper bound matches previous work with exact formulation or lower bound under various settings, which demonstrates the tightness of our upper bound, and accordingly, our result is able to recover the double descent phenomenon.

\section{Discussion on the used assumptions}
\label{app:assumptions}

Here we give more discussion on the used assumptions, especially  Assumptions~\ref{assexist} and~\ref{assump:bound_fourthmoment}, which are fair and attainable.

{\bf Discussion on Assumption~\ref{assexist}:} 
$i)$ \emph{bounded Hilbert norm:} In high-dimensional asymptotics, this bounded Hilbert norm assumption is commonly used in kernel regression \cite{liang2020just,liang2020multiple,liu2020kernelreg}, and RF model \cite{mei2021generalization} even though $n$ and $d$ tend to infinity. Here we give an example satisfying this assumption, which is provided by \cite[Proposition 4]{bach2017breaking}, i.e., linear functions on the sphere can have bounded Hilbert norm for all $d$.

To be specific, assume $f: \mathbb{S}^{d} \rightarrow \mathbb{R}$ such that $f(\bm x) = \bm v^{\!\top} \bm x$ for a certain $\bm v \in \mathbb{S}^d$, if we consider the following reproducing kernel
\begin{equation*}
    k(\bm x, \bm x') = \int_{\mathbb{S}^d} 1_{\{\bm \omega^{\!\top} \bm x \geq 0\}} 1_{\{\bm \omega^{\!\top} \bm x' \geq 0\}} \mathrm{d} \mu(\bm \omega)\,,
\end{equation*}
where $\mu$ is the probability measure of $\bm \omega$, leading to a zero-order arc-cosine kernel \cite{cho2009kernel} by taking Gaussian measure.
Then we have 
\begin{equation*}
    \| f \|_{\mathcal{H}} = \frac{2d\pi}{d-1} \leqslant 4 \pi \,,
\end{equation*}
which verifies that our assumption on bounded Hilbert norm is attainable.

We also need to remark that, unbounded Hilbert norm of functions can be achieved \cite{bach2017breaking,donhauser2021rotational} when $d \rightarrow \infty$ in some cases. For example, if we consider the above problem setting but employ the first-order arc-cosine kernel, we have $\| f \|_{\mathcal{H}} \asymp C \sqrt{d}$ for some constant $C$ independent of $d$.
%i.e.
%\begin{equation*}
%    k(\bm x, \bm x') = \int_{\mathbb{S}^d} \max\{\bm \omega^{\!\top} \bm x ,0\} \max\{\bm \omega^{\!\top} \bm x',0\} \mathrm{d} \mu(\bm \omega)\,,
%\end{equation*}

%we have $\| f \|_{\mathcal{H}} = \frac{4\sqrt{\pi} d}{d-1} \frac{\Gamma(\frac{d}{2}+ \frac{1}{2})}{\Gamma(\frac{d}{2})}$.
Accordingly, apart from directly regarding it as an assumption, we also give an example such that a function can have bounded Hilbert norm.
In fact, in practice $d$ is fixed (larger or smaller than $n$), and accordingly it is reasonable for a fixed ground truth with bounded Hilbert norm. 

$ii)$ \emph{optimal solution:} We assume that $\mathcal{E}(f)$ admits a unique global optimum. 
If multiple solutions exist, we choose the minimum norm solution of $\mathbb{E}(f)$, i.e., 
\begin{equation*}
f^* = \argmin_{f \in \mathcal{H}} \| f \|_{\mathcal{H}}\quad \mbox{s.t.}~ f \in \argmin_{f \in \mathcal{H}} \mathcal{E}(f)\,,
\end{equation*}
which follows the setting \cite{zou2021benign,wu2021last}.

{\bf Discussion on Assumption~\ref{assump:bound_fourthmoment}:} This assumption follows the spirit of \cite[Assumption 2.2]{zou2021benign}.
According to \cite[Theorem 5.2.15]{vershynin2018high},
assume $\bm x$ is a sub-Gaussian random vector with density of the form $p(\bm x) = \exp(-U(\bm x))$ for the strongly convex function $U: \mathbb{R}^d \rightarrow \mathbb{R}$.
Accordingly, $\Sigma_m^{-\frac{1}{2}}\bm x$ is sub-Gaussian, and then for any fixed $\bm W$ and PSD operator $A$, we have
\begin{equation*}
    \mathbb{E}_{\bm x} [\varphi^{\!\top}(\bm x) A \varphi(\bm x) \Sigma_m] \lesssim \mathrm{Tr}(A\Sigma_m) \Sigma_m \,.
\end{equation*}
The proof is similar to \cite[Lemma A.1]{zou2021benign}, and thus we omit the proof for simplicity.
The sub-Gaussian assumption is common in high dimensional statistics \cite{bartlett2020benign}, which is weaker than most previous work on double descent that requires the data to be Gaussian \cite{hastie2019surprises,d2020double,adlam2020understanding,ba2020generalization}, or uniformly spread on a sphere \cite{mei2019generalization,ghorbani2019linearized}, as discussed in Appendix~\ref{app:setting}.

In fact, our proof only requires Assumption~\ref{assump:bound_fourthmoment} valid to some specific PSD operators, e.g., $S^{\tt W}$, $\mathbb{E}_{\bm X}[{\Sigma}_m - \varphi(\bm x_t) \otimes \varphi(\bm x_t) ]^2$ defined in Appendix~\ref{app:psd}.
For description simplicity, we present the requirement on all PSD operators in Assumption~\ref{assump:bound_fourthmoment}.
Besides, one special case of Assumption~\ref{assump:bound_fourthmoment} by taking $A:=I$ is proved by Lemma~\ref{lemma:m2}, and accordingly this assumption can be regarded as a natural extension.

This assumption is also similar to the bounded fourth moment in stochastic approximation, see \cite{bach2013non,dieuleveut2016nonparametric,jain2018parallelizing,berthier2020tight,varre2021last} for details.

\section{Results on covariance operators}
\label{app:rescov}
In this section, we present the proofs of Lemmas~\ref{thmH},~\ref{lemsubexp},~\ref{lemma:m2},~\ref{trace1} on statistical properties of $\Sigma_m$ and $\widetilde{\Sigma}_m$.
\subsection{Proof of Lemma~\ref{thmH} and examples}
\label{sec:lemthmh}
Here we present the proof of Lemma~\ref{thmH} and then give two examples by taking different activation functions.

\subsubsection{Proof of Lemma~\ref{thmH}} 
\begin{proof}
	%For description simplicity, we assume $\sigma(\cdot) : \mathbb{R} \rightarrow \mathbb{R}$ with single-output, the results can be easily extended to multiple-output cases, e.g., the discussed Gaussian kernel later corresponding to $\sigma(x) = [\cos(x), \sin (x)]^{\!\top}$.
	
	Recall the definition of $\widetilde{\Sigma}_m$, we have	
	\begin{equation*}
		\widetilde{\Sigma}_m := \mathbb{E}_{{\bm x}, \bm W }[\varphi({\bm x}) \otimes \varphi({\bm x})] = \frac{1}{m} \mathbb{E}_{{\bm x}, W_{ij} \sim \mathcal{N}(0,1) } \left[ \sigma \left(\frac{\bm W \bm x}{\sqrt{d}} \right)\sigma \Big(\frac{\bm W \bm x}{\sqrt{d}} \Big)^{\!\top}   \right] \in \mathbb{R}^{m \times m}\,.
	\end{equation*}
	We consider the diagonal and non-diagonal elements of $\widetilde{\Sigma}_m$ separately.
	%In the next, we analysis the diagonal element and non-diagonal elements of $\widetilde{\Sigma}_m$, respectively.
	
	{\bf Diagonal element:} The diagonal entry $(\widetilde{\Sigma}_m)_{ii} = \frac{1}{m} \mathbb{E}_{\bm x, \bm \omega_i} [\sigma(\frac{\bm \omega_i^{\!\top} \bm x}{\sqrt{d}})\sigma(\frac{\bm \omega_i^{\!\top} \bm x}{\sqrt{d}}) ] = \frac{1}{m}\mathbb{E}_{\bm x} \mathbb{E}_{\bm \omega} [\sigma(\frac{\bm \omega^{\!\top} \bm x}{\sqrt{d}})]^2 $ is the same.
	In fact, $\mathbb{E}_{\bm \omega} \left[\sigma\left(\frac{\bm \omega^{\!\top} \bm x}{\sqrt{d}}\right) \right]^2$ is actually a one-dimensional integration by considering the basis $(e_1, e_2, \cdots, e_d)$ with $\bm e_1 = \bm x/\| \bm x \|_2$, and $\bm e_2, \cdots, \bm e_d$ any completion of the basis. This technique is commonly used in \cite{williams1998computation,louart2018random}.
	The random feature $\bm \omega$ admits the coordinate representation $\bm \omega = \bar{\omega}_1 \bm e_1 + \bar{\omega}_2 \bm e_2 + \cdots + \bar{\omega}_d \bm e_d$, and thus
	\begin{equation*}
		\bm \omega^{\!\top} \bm x = (\bar{\omega}_1 \bm e_1 + \bar{\omega}_2 \bm e_2 + \cdots + \bar{\omega}_d \bm e_d)^{\!\top} ( \| \bm x \| \bm e_1 ) = \| \bm x \| \bar{\omega}_1 \,,
	\end{equation*}
	which implies
	\begin{equation*}
		\begin{split}
			\mathbb{E}_{\bm \omega} \left[\sigma \left(\frac{\bm \omega^{\!\top} \bm x}{\sqrt{d}} \right) \right]^2
			& = (2\pi)^{-\frac{d}{2}} \int_{\mathbb{R}^d} \left[ \sigma \left(\frac{\bm \omega^{\!\top} \bm x}{\sqrt{d}} \right) \right]^2 \exp \left(-\frac{1}{2} \| \bm \omega \|_2^2 \right) \mathrm{d} \bm \omega \\
			& = \frac{1}{\sqrt{2\pi}} \int_{\mathbb{R}} \left[ \sigma \left( \frac{\bar{\omega}_1 \| \bm x \|_2}{\sqrt{d}} \right) \right]^2 \exp\left(-\frac{\bar{\omega}_1^2}{2}\right) \mathrm{d} \bar{\omega}_1 \\
			& =  \frac{1}{\sqrt{2\pi}} \int_{\mathbb{R}} [\sigma(z)]^2 \exp \left( - \frac{z^2}{2 \| \bm x \|^2/d} \right) \frac{\sqrt{d}}{\| \bm x \|_2} \mathrm{d} z \\
			& = \mathbb{E}_{z \sim \mathcal{N}(0,{\| \bm x \|^2_2}/{d})} [\sigma(z)]^2\,,
		\end{split}
	\end{equation*}
	where we change the integral variable $z:=\frac{\bar{\omega}_1 \| \bm x \|_2}{\sqrt{d}}$.
	Hence we have $(\widetilde{\Sigma}_m)_{ii} =\frac{1}{m} \mathbb{E}_{\bm x} \mathbb{E}_{z \sim \mathcal{N}(0,{\| \bm x \|^2_2}/{d})} [\sigma(z)]^2$.
	
	{\bf Non-diagonal element:} The non-diagonal entry $(\widetilde{\Sigma}_m)_{ij} = \frac{1}{m}\mathbb{E}_{\bm x, \bm \omega_i, \bm \omega_j} [\sigma(\frac{\bm \omega_i^{\!\top} \bm x}{\sqrt{d}})\sigma(\frac{\bm \omega_j^{\!\top} \bm x}{\sqrt{d}})^{\!\top} ] = \frac{1}{m}\mathbb{E}_{\bm x} [\mathbb{E}_{\bm \omega} \sigma(\frac{\bm \omega^{\!\top} \bm x}{\sqrt{d}})]^2 $ is the same due to the independence between $\bm \omega_i$ and $\bm \omega_j$. Likewise, it can be represented as a one-dimensional integration
	\begin{equation*}
		(\widetilde{\Sigma}_m)_{ij} = \frac{1}{m}\mathbb{E}_{\bm x} \left[\mathbb{E}_{\bm \omega} \sigma \left(\frac{\bm \omega^{\!\top} \bm x}{\sqrt{d}}\right)\right]^2 = \frac{1}{m}\mathbb{E}_{\bm x} \left[\mathbb{E}_{z \sim \mathcal{N}(0,1)} \sigma \left(\frac{ z \|\bm x\|}{\sqrt{d}} \right)\right]^2 = \frac{1}{m}\mathbb{E}_{\bm x} \left( \mathbb{E}_{z \sim \mathcal{N}(0,{\| \bm x \|^2_2}/{d})} [\sigma(z)] \right)^2\,.
	\end{equation*}
	Accordingly, by denoting $a: = (\widetilde{\Sigma}_m)_{ii}$ and $b := (\widetilde{\Sigma}_m)_{ij}$, the covariance operator $\widetilde{\Sigma}_m$ can be represented as
	\begin{equation}\label{tilesigmamdef}
		\widetilde{\Sigma}_m = (a-b) I_m + b \bm 1 \bm 1^{\!\top} \in \mathbb{R}^{m \times m}\,,
	\end{equation}
	with its determinant $\det(\widetilde{\Sigma}_m) = (1+ \frac{mb}{a-b}) (a-b)^{m}$.
	Hence, the eigenvalues of $\widetilde{\Sigma}_m$ can be naturally obtained by the matrix determinant lemma: $\widetilde{\lambda}_1(\widetilde{\Sigma}_m) = a - b + bm$ and the remaining eigenvalues are $a-b$.
	
	According to \cite[Theorem 2.26]{wainwright2019high}, by virtue of the Lipschitz function $\sigma(\cdot)$ of Gaussian variables, we have
	\begin{equation*}
		\mathbb{P} \left[ \bigg| \sigma\left(\frac{\bm \omega^{\!\top} \bm x}{\sqrt{d}}\right) - \mathbb{E}_{\bm \omega \sim \mathcal{N}(\bm 0, \bm I_d)} \sigma \left(\frac{\bm \omega^{\!\top} \bm x}{\sqrt{d}} \right) \bigg| \geqslant t \right] \leqslant c \exp(-t^2),~~\forall t \geqslant 0\,,
	\end{equation*}
	which implies that $\sigma \big(\frac{\bm \omega^{\!\top} \bm x}{\sqrt{d}} \big) $ is a sub-Gaussian random variable due to its expectation in the $\mathcal{O}(1)$ order.
	Accordingly, for $z \sim \mathcal{N}(0,{\| \bm x \|^2_2}/{d})$, we have  $\mathbb{E}_{\bm x} \mathbb{V}[\sigma(z)] \sim \mathcal{O}(1)$ as $\sigma(z)$ is sub-Gaussian with $\mathcal{O}(1)$ norm and its finite second moment, i.e., $\mathbb{V}[\sigma(z)] \sim \mathcal{O}(1)$, which implies $\widetilde{\lambda}_2 = \frac{1}{m} \mathbb{E}_{\bm x} \mathbb{V}[\sigma(z)] \sim \mathcal{O}(1/m)$.
	Finally, we  conclude the proof.
\end{proof}

\subsubsection{Examples}
\label{sec:example}

In our analysis, we assume $\sigma(\cdot) : \mathbb{R} \rightarrow \mathbb{R}$ with single-output for description simplicity.
In fact, our results can be easily extended to multiple-output cases, e.g., the Gaussian kernel corresponding to $\sigma(x) = [\cos(x), \sin (x)]^{\!\top}$.
Here we give two examples, including single- and multiple-output: arc-cosine kernel that corresponds to the ReLU function $\sigma(x) = \max\{0,x\}$; and the Gaussian kernel.
%Here we give the calculation details for the Gaussian kernel that corresponds to the $\sin/\cos$ activation function $\sigma(x) = [\cos(x), \sin (x)]^{\!\top}$ and arc-cosine kernel that corresponds to the ReLU function $\sigma(x) = \max\{0,x\}$.

{\bf Arc-cosine kernel:} 
We begin with calculation of arc-cosine kernels due to its related single-output activation function.
Denote $\widetilde{z} := \max \{ 0, z \}$ with $z \sim \mathcal{N}(0,{\| \bm x \|^2_2}/{d})$, it is subject to the Rectified Gaussian distribution admitting (refer to \cite{li2021towards})
\begin{equation*}
	\mathbb{E}[\widetilde{z}] = \frac{\| \bm x \|_2}{\sqrt{2d \pi} }\,,~\quad \mathbb{E}[\widetilde{z}]^2 =  \frac{\| \bm x \|_2^2}{2d}\,,~\quad  \mathbb{V}[\widetilde{z}] = \frac{\| \bm x \|_2^2}{2d} \left(1 - \frac{1}{\pi} \right).
\end{equation*}
Accordingly, recall the sample covariance operator  $\Sigma_d := \mathbb{E}_{\bm x} [\bm x \bm x^{\!\top}]$, the diagonal elements are the same 
\begin{equation*}
	(\widetilde{\Sigma}_m)_{ii} = \frac{1}{m} \mathbb{E}_{\bm x}\mathbb{E}_{z \sim \mathcal{N}(0,{\| \bm x \|^2_2}/{d})} [\sigma(z)]^2 = \frac{1}{2md} \mathbb{E}_{\bm x} \| \bm x \|_2^2 = \frac{1}{2md} \mathrm{Tr}(\Sigma_d)\,, \quad i=1,2,\dots,m \,,
\end{equation*}
and the non-diagonal elements are the same
\begin{equation*}
	(\widetilde{\Sigma}_m)_{ij} = \frac{1}{m}\mathbb{E}_{\bm x} \left( \mathbb{E}_{z \sim \mathcal{N}(0,{\| \bm x \|^2_2}/{d})} [\sigma(z)] \right)^2 = \frac{1}{2md \pi} \mathrm{Tr}(\Sigma_d)\,, \quad i,j=1,2,\dots,m ~\mbox{with} ~ i \neq j \,.
\end{equation*}

{\bf Gaussian kernels:}
Briefly, if we choose $\sigma(x) = [\cos(x), \sin (x)]^{\!\top}$, a multiple-output version, RF actually approximates the Gaussian kernel with $\varphi(\bm x) \in \mathbb{R}^{2m}$ in Eq.~\eqref{mapping}, resulting in $\widetilde{\Sigma}_m \in \mathbb{R}^{2m \times 2m}$.
In this case, $\widetilde{\Sigma}_m = S_1 \oplus S_2$ is a block diagonal matrix, where $\oplus$ is the direct sum.
By denoting $\vartheta := \| \bm x \|_2^2/d$, the matrix $S_1 \in \mathbb{R}^{m \times m}$ has the same diagonal elements $[S_1]_{ii} = \frac{1}{2m} \mathbb{E}_{\bm x} \left[ 1+ \exp \left(-2 \vartheta \right) \right]$, and the same non-diagonal elements $\frac{1}{m} \mathbb{E}_{\bm x} \left[ \exp \left(- \vartheta \right) \right]$. The matrix $S_2$ is diagonal with $ [S_2]_{ii} =  \frac{1}{2m} \mathbb{E}_{\bm x} \left[ 1- \exp \left(-2 \vartheta \right) \right]$.
In this case, $\widetilde{\Sigma}_m$ admits three distinct eigenvalues: the largest eigenvalue at $\mathcal{O}(1)$ order, and the remaining two eigenvalues at $\mathcal{O}(1/m)$ order.

According to Bochner's theorem \cite{bochner2005harmonic},  we have $\mathbb{E}[\cos(\bm \omega^{\!\top} \bm z)] = \exp(-{z^2}/{2})$ and $\mathbb{E}[\cos^2(\bm \omega^{\!\top} \bm z)] = \frac{1 + \exp(-2z^2)}{2}$ for $\bm \omega \sim \mathcal{N}(\bm 0, \bm I_d)$ and $z:= \| \bm z \|_2$.
In fact, this can be computed by two steps: first transforming the $d$-dimensional integration to a one-dimensional integral as discussed before; and then computing the integral by virtue of the Euler's formula $\exp(-\mathrm{i}x) = \cos x + \mathrm{i} \sin x$. For instance,
\begin{equation*}
	\begin{split}
		\mathbb{E}[\cos(\bm \omega^{\!\top} \bm z)] &= \mathbb{E}_{x \sim \mathcal{N}(0, \| \bm z \|_2^2)} \cos x 
		= \frac{1}{\sqrt{2 \pi} \| \bm z \|_2} \mathrm{Re} \left[ \int_{\mathbb{R}} \exp(-\frac{x^2}{2\| \bm z \|_2^2}) \exp(\mathrm{i}x) \mathrm{d} x \right] \\
		&= \exp \left(  -\frac{\| \bm z \|_2^2}{2} \right) \mathrm{Re} \left[ \frac{1}{\sqrt{2 \pi} \| \bm z \|_2}  \int_{\mathbb{R}} \exp \left( -\frac{(x- \mathrm{i} \| \bm z \|_2^2)}{2 \| \bm z \|_2^2} \right) \mathrm{d}x \right] \\
		&= \exp \left(  -\frac{\| \bm z \|_2^2}{2} \right)\,.
	\end{split}
\end{equation*}
Similarly, we have $\mathbb{E}[\sin(\bm \omega^{\!\top} \bm z)] = 0$ and $\mathbb{E}[\sin^2(\bm \omega^{\!\top} \bm z)] = \frac{1 - \exp(-2z^2)}{2}$ for $\bm \omega \sim \mathcal{N}(\bm 0, \bm I_d)$ and $z:= \| \bm z \|_2$.

Based on the above results, for the Gaussian kernel, the expected covariance operator $\widetilde{\Sigma}_m$ is a block diagonal matrix
\begin{equation*}
	\widetilde{\Sigma}_m =
	\left[
	\begin{array}{c|c}
		\bm S_1 & \bm 0 \\ \hline 
		\bm 0& \bm S_2
	\end{array}
	\right] \in \mathbb{R}^{2m \times 2m}\,,
\end{equation*}
where $\bm S_1 \in \mathbb{R}^{m \times m}$ has the same diagonal elements and the same non-diagonal elements:
\begin{equation*}
	\begin{split}
		[\bm S_1]_{ii} & = \frac{1}{m} \mathbb{E}_{\bm x, \bm \omega} \left[\cos \left( \frac{\bm \omega^{\!\top} \bm x}{\sqrt{d}} \right) \right]^2 = \frac{1}{2m} \mathbb{E}_{\bm x} \left[ 1+ \exp \left(-2 \frac{\| \bm x \|_2^2}{d} \right) \right]\,, \quad i = 1,2,\dots,m \,, \\
		[\bm S_1]_{ij} & = \frac{1}{m} \mathbb{E}_{\bm x} \left[ \mathbb{E}_{\bm \omega} \cos \left( \frac{\bm \omega^{\!\top} \bm x}{\sqrt{d}} \right) \right]^2 = \frac{1}{m} \mathbb{E}_{\bm x} \left[ \exp \left(- \frac{\| \bm x \|_2^2}{d} \right) \right]\,, \quad i,j=1,2,\dots,m ~\mbox{with} ~ i \neq j \,.
	\end{split}
\end{equation*}
The matrix $\bm S_2 \in \mathbb{R}^{m \times m}$ is diagonal with 
\begin{equation*}
	[\bm S_2]_{ii}  = \frac{1}{m} \mathbb{E}_{\bm x, \bm \omega} \left[\sin \left( \frac{\bm \omega^{\!\top} \bm x}{\sqrt{d}} \right) \right]^2 = \frac{1}{2m} \mathbb{E}_{\bm x} \left[ 1- \exp \left(-2 \frac{\| \bm x \|_2^2}{d} \right) \right]\,, \quad i = 1,2,\dots,m \,.
\end{equation*}
Accordingly, $\widetilde{\Sigma}_m$ has three distinct eigenvalues
\begin{equation*}
	\begin{split}
		\widetilde{\lambda}_1 &= \mathbb{E}_{\bm x} \left[ \exp \left(- \frac{\| \bm x \|_2^2}{d} \right) \right] + \frac{1}{2m} \mathbb{E}_{\bm x} \left[ 1 - \exp \left(- \frac{\| \bm x \|_2^2}{d} \right) \right]^2 \sim \mathcal{O}(1) \,, \\
		\widetilde{\lambda}_2 & = \frac{1}{2m} \mathbb{E}_{\bm x}  \left[ 1 - \exp \left(-2 \frac{\| \bm x \|_2^2}{d} \right) \right] \sim \mathcal{O}\left( \frac{1}{m}\right)\,, \\
		\widetilde{\lambda}_3 & = \frac{1}{2m} \mathbb{E}_{\bm x}  \left[ 1 - \exp \left(- \frac{\| \bm x \|_2^2}{d} \right) \right]^2 \sim \mathcal{O}\left( \frac{1}{m}\right)\,.
	\end{split}
\end{equation*}
In this case, we can also get the similar claim on spectra of $\widetilde{\Sigma}_m$ with the single-output version:  $\widetilde{\Sigma}_m$ admits the largest eigenvalue at $\mathcal{O}(1)$ order, and the remaining eigenvalues are at $\mathcal{O}(1/m)$ order.

%Regarding to the Gaussian kernel, by virtue of $\mathbb{E}[\cos(\bm a^{\!\top} \bm z)] = \cos(\bm \mu^{\!\top} \bm z) \exp(-\frac{1}{2} \bm z^{\!\top} \bm A \bm z)$ for $\bm a \sim \mathcal{N}(\bm \mu, \bm A)$ and $\bm \omega_i - \bm \omega_j \sim \mathcal{N}(\bm 0, 2\bm I_d)$

\subsection{Proof of Lemma~\ref{lemsubexp}}
\begin{proof}
	As discussed before, $\sigma \big(\frac{\bm \omega^{\!\top} \bm x}{\sqrt{d}} \big) $ is a sub-Gaussian random variable with the $\mathcal{O}(1)$ sub-Gaussian norm order.
	Hence, $\| {\Sigma}_m - \widetilde{\Sigma}_m \|_2$ is a sub-exponential random variable with
	\begin{equation*}
		\begin{split}
			\| {\Sigma}_m - \widetilde{\Sigma}_m \|_2 & \leqslant \| {\Sigma}_m \|_2 + \| \widetilde{ \Sigma}_m \|_2  = \frac{1}{m} \left\| \mathbb{E}_{{\bm x} } \Big[ \sigma \left(\frac{\bm W \bm x}{\sqrt{d}} \right) \sigma \left(\frac{\bm W \bm x}{\sqrt{d}} \right)^{\!\top}   \Big] \right\|_2  + \mathcal{O}(1) \quad \mbox{[using Lemma~\ref{thmH}]} \\
			&  \leqslant \frac{1}{m} \mathbb{E}_{{\bm x} } \left\| \sigma \left(\frac{\bm W \bm x}{\sqrt{d}}\right) \right\|_2^2 +  \mathcal{O}(1) \quad \mbox{[Jensen's inequality]} \\
			& \lesssim \frac{1}{m} \left( \mathbb{E}_{{\bm x} } \| \sigma(\bm 0_m) \|^2_2 + \mathbb{E}_{{\bm x} } \left\| \frac{\bm W \bm x}{\sqrt{d}} \right\|^2_2 \right) + \mathcal{O}(1) \quad\mbox{[$\sigma$: Lipschitz continuous]}\\
			& \lesssim \mathcal{O}(1) + \frac{1}{md} \sum_{i=1}^m \bm \omega_i^{\!\top} \mathbb{E}_{\bm x} [\bm x \bm x^{\!\top}] \bm \omega_i \quad \mbox{[using $\|\Sigma_d\|_2 < \infty$]} \\
			& \lesssim \frac{1}{d}\| \bm \omega \|^2_2 \quad \mbox{[here $\bm \omega \sim \mathcal{N}(\bm 0, \bm I_d)$]}\,, 
		\end{split}
	\end{equation*}
	where $\| \bm \omega \|^2_2$ is a $\chi^2(d)$ random variable, and thus $\| {\Sigma}_m - \widetilde{\Sigma}_m \|_2$ has sub-exponential norm at $\mathcal{O}(1)$ order.
	Accordingly, the high moment $\mathbb{E}\| \Sigma_m \|^p_2 < \infty$ holds for finite $p$.
	Following the above derivation, we can also conclude that $\mathrm{Tr}(\Sigma_m)$ has the sub-exponential norm $\mathcal{O}(1)$, i.e.
	\begin{equation*}
		\mathrm{Tr}(\Sigma_m) = \frac{1}{m}  \mathbb{E}_{{\bm x} } \mathrm{Tr}\left[ \sigma \left(\frac{\bm W \bm x}{\sqrt{d}} \right) \sigma \left(\frac{\bm W \bm x}{\sqrt{d}} \right)^{\!\top}   \right] = \frac{1}{m} \mathbb{E}_{{\bm x} } \left\| \sigma \left(\frac{\bm W \bm x}{\sqrt{d}}\right) \right\|_2^2 \lesssim \frac{1}{d}\| \bm \omega \|^2_2\,.
	\end{equation*}
	Likewise, we can derive $\mathrm{Tr}(\Sigma_m^2) < \infty$ in the similar fashion.
	%Besides, our work needs the error bound for the smallest eigenvalue of $\Sigma_m$, i.e., $\lambda_m$. By virtue of Schur–Horn theorem \cite{horn1954doubly}, $\lambda_m$ admits
	%\begin{equation*}
	%	\lambda_m \leqslant \min_{i \in \{1,2,\dots, m \} } (\Sigma_m )_{ii} \sim \mathcal{O}\left( \frac{1}{m}\right)\,,
	%\end{equation*}
	%due to the sub-exponential random variable $(\Sigma_m)_{ii}$.
	%It is clear that $\Sigma_m$ is a PSD operator and we have $\mathbb{E}_{\bm W} [\lambda_m] \sim \mathcal{O}\left( \frac{1}{m}\right)$ due to the sub-exponential random variable $(\Sigma_m)_{ii}$. Besides, we have 
	%\begin{equation*}
	 %   \mathbb{E}_{\bm W} \left[ \frac{1}{\lambda_m} \right] \geqslant \mathbb{E}_{\bm W} \frac{1}{\min_{i \in \{1,2,\dots, m \} } (\Sigma_m )_{ii}} := \mathbb{E}_{\bm W} \frac{1}{(\Sigma_m)_{i^* i^*}} \geqslant \frac{1}{\mathbb{E}_{\bm W} (\Sigma_m)_{i^* i^*}}
	%\end{equation*}
\end{proof}

\subsection{Proof of Lemma~\ref{lemma:m2}}
\begin{proof}
	The first inequality naturally holds, and so we focus on the second inequality.
	Denote $\Phi := \mathbb{E}_{\bm x, \bm W}[\varphi(\bm x) \otimes \varphi(\bm x)\otimes \varphi(\bm x) \otimes \varphi(\bm x)]$, its diagonal elements are the same
	\begin{equation*}
		\Phi_{ii} =  \frac{m-1}{m^2}\mathbb{E}_{\bm x} \left( \mathbb{E}_{z \sim \mathcal{N}(0,{\| \bm x \|^2_2}/{d})} [\sigma(z)]^2 \right)^2  + \frac{1}{m^2}\mathbb{E}_{\bm x} \mathbb{E}_{z \sim \mathcal{N}(0,{\| \bm x \|^2_2}/{d})} [\sigma(z)]^4 \sim \mathcal{O}\left( \frac{1}{m}\right)\,.
	\end{equation*}
	Its non-diagonal elements $\Phi_{ij}$ with $i \neq j$ are the same
	\begin{equation*}
		\begin{split}
			\Phi_{ij} & = \frac{m-3}{m^2}\mathbb{E}_{\bm x} \bigg[ \left( \mathbb{E}_{z \sim \mathcal{N}(0,{\| \bm x \|^2_2}/{d})} [\sigma(z)] \right)^2 \mathbb{E}_{z \sim \mathcal{N}(0,{\| \bm x \|^2_2}/{d})} [\sigma(z)]^2 \bigg] \\ 
			& + \frac{2}{m^2} \mathbb{E}_{\bm x} \bigg[ \mathbb{E}_{z \sim \mathcal{N}(0,{\| \bm x \|^2_2}/{d})} [\sigma(z)]^3 \mathbb{E}_{z \sim \mathcal{N}(0,{\| \bm x \|^2_2}/{d})} [\sigma(z)] \bigg]\,,
		\end{split}
	\end{equation*}
	where the first term is in $\mathcal{O}(\frac{1}{m})$ order and the second term is in $\mathcal{O}(\frac{1}{m^2})$ order.
	By denoting $a: = (\widetilde{\Sigma}_m)_{ii}$, $b := (\widetilde{\Sigma}_m)_{ij}$ as given by Lemma~\ref{thmH}, $A:= \Phi_{ii}$, and $B:=\Phi_{ij}$, the operator $r \mathrm{Tr}(\widetilde{\Sigma}_m) \widetilde{\Sigma}_m - \Phi$ can be represented as
	\begin{equation*}
		r \mathrm{Tr}(\widetilde{\Sigma}_m) \Sigma_m - \Phi = \left[ rm(a-b) - A + B \right] I_m + (rmab - B) \bm 1 \bm 1^{\!\top}\,,
	\end{equation*} 
	of which the smallest eigenvalue is $rma(a-b) - A + B$.
	Accordingly, to ensure the positive definiteness of $r \mathrm{Tr}(\widetilde{\Sigma}_m) \widetilde{\Sigma}_m - \Phi$, which implies $\mathbb{E}_{\bm W} \bigg( \mathbb{E}_{\bm x} \Big( [\varphi(\bm x) \otimes \varphi(\bm x) ] A [\varphi(\bm x) \otimes \varphi(\bm x)]  \Big) \bigg) \preccurlyeq r \mathrm{Tr}(\widetilde{\Sigma}_m )\widetilde{\Sigma}_m$, we require its smallest eigenvalue is non-negative, i.e., $	rma(a-b) - A + B \geqslant 0$.
	That means, $r$ should satisfies
	\begin{equation}\label{conditionr}
		r \geqslant \frac{A-B}{ma(a-b)} = \frac{A-B}{\frac{1}{m}\mathbb{E}_{\bm x} \mathbb{E}_{z \sim \mathcal{N}(0,{\| \bm x \|^2_2}/{d})} [\sigma(z)]^2 \mathbb{E}_{\bm x}\mathbb{V}[\sigma(z)]}\,.
	\end{equation}
	Since $A - B$ admits
	\begin{equation*}
		A - B \leqslant \frac{1}{m} \mathbb{E}_{\bm x} \mathbb{E}_{z} [\sigma(z)]^2 \mathbb{E}_{\bm x}\mathbb{V}[\sigma(z)] + \mathcal{O} \left( \frac{1}{m^2} \right) \,,
	\end{equation*}
	then by taking $r:= 1 + \mathcal{O}\left( \frac{1}{m} \right)$, the condition in Eq.~\eqref{conditionr} satisfies, and thus $r \mathrm{Tr}(\widetilde{\Sigma}_m) \widetilde{\Sigma}_m - \Phi$ is positive definite, which concludes the proof.
\end{proof}

\section{Preliminaries on PSD operators}
\label{app:psd}
In this section, we first define some stochastic/deterministic PSD operators that follow \cite{jain2017markov,zou2021benign} in stochastic approximation, and then present Lemma~\ref{dinfvx} that is based on PSD operators and is needed to estimate ${\tt B1}$ and ${\tt V1}$.
Note that, the PSD operators will make the notation in our proof simple and clarity but do not change the proof itself. 

Following \cite{jain2017markov,zou2021benign}, we define several stochastic PSD operators as below.
Given the random features matrix $\bm W$ and any PSD operator $A$, define
\begin{equation*}
	\begin{split}
		& S^{\tt W} := \mathbb{E}_{\bm x} [\varphi(\bm x) \otimes \varphi(\bm x) \otimes \varphi(\bm x) \otimes \varphi(\bm x)], \quad \widetilde{S}^{\tt W} := \Sigma_m \otimes \Sigma_m\,, \\
		& S^{\tt W} \circ A := \mathbb{E}_{\bm x} \left[ \varphi(\bm x)^{\!\top} \varphi(\bm x) A \varphi(\bm x) \otimes \varphi(\bm x) \right], \quad \widetilde{S}^{\tt W} \circ A := {\Sigma}_m A {\Sigma}_m \,,
	\end{split}
\end{equation*}
where the superscript ${\tt W}$ denotes the randomness dependency on the random feature matrix $\bm W$.
Besides, for any $\gamma_i$ ($i=1,2,\dots,n$), define the following operators
\begin{equation*}
	\begin{split}
		& (I - \gamma_i T^{\tt W}) \circ A :=  \mathbb{E}_{\bm x} \left( [I - \gamma_i \varphi(\bm x) \otimes \varphi(\bm x)] A [I - \gamma_i \varphi(\bm x) \otimes \varphi(\bm x)] \right) \\ 
		& (I - \gamma_i \widetilde{T}^{\tt W}) \circ A :=  (I - \gamma_i {\Sigma}_m) A (I - \gamma_i {\Sigma}_m) \,,
	\end{split}
\end{equation*}
associated with two corresponding operators (that depend on $\gamma_i$)
\begin{equation*}
	T^{\tt W} := {\Sigma}_m \otimes I + I \otimes {\Sigma}_m - \gamma_i S^{\tt W}, \quad \widetilde{T}^{\tt W} := {\Sigma}_m \otimes I + I \otimes {\Sigma}_m - \gamma_i \widetilde{S}^{\tt W} \,.
\end{equation*}
Clearly, the above operators $S^{\tt W}$, $\widetilde{S}^{\tt W}$, $(I - \gamma_i T^{\tt W})$, $(I - \gamma_i \widetilde{T}^{\tt W})$, $T^{\tt W}$, and $\widetilde{T}^{\tt W}$ are PSD, and $S^{\tt W} \succcurlyeq \widetilde{S}^{\tt W}$.
The proof is similar to \cite[Lemma~B.1]{zou2021benign} and thus we omit it here. 

Further, if $\gamma_0 < 1/\mathrm{Tr}(\Sigma_m)$, the PSD operator $I - \gamma_i \Sigma_m$ ($i=1,2,\dots,n$) is a contraction map, and thus for any PSD operator $A$ and step-size $\gamma_i$, the following exists
\begin{equation*}
	\sum_{t=0}^{\infty} (I - \gamma_i \widetilde{T}^{\tt W})^t \circ A = \sum_{t=0}^{\infty} (I - \gamma_i \Sigma_m)^t A (I - \gamma_i \Sigma_m)^t\,.
\end{equation*}
Hence, $(\widetilde{T}^{\tt W})^{-1} := \gamma_i \sum_{t=0}^{\infty} (I - \gamma_i \widetilde{T}^{\tt W})^t$ exists and is PSD.
We need to remark that, though $\mathrm{Tr}(\Sigma_m)$ is a random variable, it is with a sub-exponential $\mathcal{O}(1)$ norm. That means, this holds with exponentially high probability.
%We have $\widetilde{T} - T = \gamma (S - \widetilde{S})$, and they are PSD mapping.

Based on the above stochastic operators, we define several deterministic PSD ones by taking the expectation over $\bm W$ as below. For any given $\gamma_i$ ($i=1,2,\dots,n$), we have the following PSD operators
\begin{equation*}
	\begin{split}
		& S := \mathbb{E}_{\bm W}[\Sigma_m \otimes \Sigma_m], \quad \widetilde{S} := \widetilde{\Sigma}_m \otimes \widetilde{\Sigma}_m \,,\\
		& T := \widetilde{\Sigma}_m \otimes I + I \otimes \widetilde{\Sigma}_m - \gamma_i S, \quad \widetilde{T} := \widetilde{\Sigma}_m \otimes I + I \otimes \widetilde{\Sigma}_m - \gamma_i \widetilde{S} \,,\\
		& S \circ A := \mathbb{E}_{\bm W} [\Sigma_m A \Sigma_m], \quad \widetilde{S} \circ A := \widetilde{\Sigma}_m A \widetilde{\Sigma}_m\,, \\
		& (I - \gamma_i T) \circ A := \mathbb{E}_{\bm W} [ (I - \gamma_i \Sigma_m) A (I - \gamma_i \Sigma_m) ],\quad (I - \gamma_i \widetilde{T}) \circ A :=  (I - \gamma_i \widetilde{\Sigma}_m) A (I - \gamma_i \widetilde{\Sigma}_m) \,,
	\end{split}
\end{equation*}
which implies $\widetilde{T} - T = \gamma_i (S - \widetilde{S})$.

Based on the above PSD operators, we present a lemma here that is used to estimate ${\tt B1}$ and ${\tt V1}$.\footnote{Our proofs on the remaining quantities including ${\tt V2}$, ${\tt V3}$, ${\tt B2}$, ${\tt B3}$ do not use PSD operators.} 
\begin{lemma}\label{dinfvx}
	Under Assumptions~\ref{assumdata},~\ref{assexist},~\ref{assumact},~\ref{assump:bound_fourthmoment} with $r' \geqslant 1$, denote 
	\begin{equation}\label{defdvxt}
		D^{{\tt v-X}}_t := \sum_{s=1}^{t} \prod_{i=s+1}^t (I - \gamma_i T^{\tt W}) \circ  \gamma_s^2 B \Sigma_m\,,
	\end{equation}
	with a scalar $B$ independent of $k$,
	if the step-size $\gamma_t := \gamma_0 t^{-\zeta}$ with $\zeta \in [0,1)$ satisfies
	\begin{equation*}
		\gamma_0 < \min \left\{ \frac{1}{r'\mathrm{Tr}(\Sigma_m)}, \frac{1}{c'\mathrm{Tr}(\Sigma_m)} \right\} \,,
	\end{equation*}
	where the constant $c'$ is defined as
	\begin{equation}\label{eqconstant}
		\begin{split}
			c':= \left\{ \begin{array}{rcl}
				\begin{split}
					&  1 ,~\mbox{if $\zeta = 0$}\,, \\
					&   \frac{1}{1-2^{-\zeta}} ,~\mbox{if $\zeta \in (0,1)$}   \,.
				\end{split}
			\end{array} \right. 
		\end{split}
	\end{equation}
	Then $D^{{\tt v-X}}_t$ can be upper bounded by
	\begin{equation*}
		D^{{\tt v-X}}_t \preccurlyeq \frac{\gamma_0 B}{1-\gamma_0 r' \mathrm{Tr}(\Sigma_m)} I \,.
	\end{equation*}
\end{lemma}
{\bf Remark:} The PSD operator $I - \gamma_i T^{\tt W}$ cannot be guaranteed as a contraction map since we cannot directly choose $\gamma_0 < \frac{1}{\mathrm{Tr}[\varphi(\bm x) \varphi(\bm x)^{\!\top}]}$ for general data $\bm x$. However, its summation in Eq.~\eqref{defdvxt} can be still bounded by our lemma. In our work, we set $B:= r' \mathrm{Tr}(\Sigma_m) $ for estimate ${\tt B1}$, and $B:= \tau^2 r'  \gamma_0  [\mathrm{Tr}(\Sigma_m) + \gamma_0 \mathrm{Tr}(\Sigma_m^2)] $ to bound ${\tt V1}$, respectively.

\begin{proof}
	Our proof can be divided into two parts: one is to prove $\mathrm{Tr}[D^{{\tt v-X}}_t(\zeta)] \leqslant \mathrm{Tr}[D^{{\tt v-X}}_t(0)]$ for any $\zeta \in [0,1)$; the other is to provide the upper bound of $D^{{\tt v-X}}_t(0)$.
	We focus on the first part and the proof in the second part follows \cite[Lemmas 3 and 5]{jain2017markov} and \cite[Lemma B.4]{zou2021benign}. %We just present it here for completeness. 
	
	%We consider $B$ as a scalar here for description simplicity and the following proof can be easily extended to the operator/matrix case for $B$.
	%	Denote the constant step-size setting (special case) with $\zeta=0$ for $D^{{\tt v-X}}_t$ as 
	%	\begin{equation*}
		%		D^{{\tt v-X}}_t(0) := \sum_{s=1}^{t} (I - \gamma_0 T^{\tt W})^{t-s} \circ  \gamma_0^2 B \Sigma_m \,.
		%	\end{equation*}
	The quantity $\mathrm{Tr}[	D^{{\tt v-X}}_t(\zeta)]$ admits the following representation by the definition of $I - \gamma_i T^{\tt W}$
	\begin{equation*}
		\begin{split}
			\mathrm{Tr}[	D^{{\tt v-X}}_t(\zeta)] &= \sum_{s=1}^{t} \prod_{i=s+1}^t \mathrm{Tr} \left[ (I - \gamma_i T^{\tt W}) \circ  \gamma_s^2 B\Sigma_m \right] \\
			& = \sum_{s=1}^{t} B\gamma_s^2 \prod_{i=s+1}^t \mathrm{Tr} \Bigg( \mathbb{E}_{\bm x}[I - \gamma_i \varphi(\bm x) \otimes \varphi(\bm x)] \Sigma_m [I - \gamma_i \varphi(\bm x) \otimes \varphi(\bm x)]  \Bigg) \\
			& = B \sum_{s=1}^{t} \gamma_s^2 \prod_{i=s+1}^t \mathrm{Tr} \Bigg( \Sigma_m - 2 \gamma_i \Sigma^2_m + \gamma_i^2 \Sigma_m \mathbb{E}_{\bm x} \left[ \varphi(\bm x) \otimes \varphi(\bm x) \otimes \varphi(\bm x) \otimes \varphi(\bm x) \right]  \Bigg)  \,.
		\end{split}
	\end{equation*}
	Based on the above results, we have
	\begin{equation*}
		\begin{split}
			\mathrm{Tr}[D^{{\tt v-X}}_t(0)] - \mathrm{Tr}[D^{{\tt v-X}}_t(\zeta)] & = B \sum_{s=1}^{t}  \prod_{i=s+1}^t \mathrm{Tr} \bigg( \Sigma_m \Big[(\gamma_0^2 - \gamma_s^2)I   - 2(\gamma_0^3 - \gamma_s^2 \gamma_i) \Sigma_m \\
			& \quad + (\gamma_0^4 - \gamma_i^2 \gamma_s^2)\mathbb{E}_{\bm x} \left[ \varphi(\bm x) \otimes \varphi(\bm x) \otimes \varphi(\bm x) \otimes \varphi(\bm x) \right] \Big] \bigg) \\
			& \geqslant B \sum_{s=1}^{t}  \prod_{i=s+1}^t \mathrm{Tr} \bigg( \Sigma_m \Big[(\gamma_0^2 - \gamma_s^2)I   - 2(\gamma_0^3 - \gamma_s^2 \gamma_i) \Sigma_m  + (\gamma_0^4 - \gamma_i^2 \gamma_s^2) \Sigma_m^2 \Big]  \bigg)\\
			& = B \sum_{s=1}^{t}  \prod_{i=s+1}^t \sum_{j=1}^m \bigg( \lambda_j \Big[(\gamma_0^2 - \gamma_s^2)   - 2(\gamma_0^3 - \gamma_s^2 \gamma_i) \lambda_j  + (\gamma_0^4 - \gamma_i^2 \gamma_s^2) \lambda_j^2 \Big]  \bigg) \\
			& = B \sum_{s=1}^{t}  \prod_{i=s+1}^t \sum_{j=1}^m \bigg( \lambda_j \Big[ (\gamma_0^4 - \gamma_i^2 \gamma_s^2) \left( \lambda_j - \frac{\gamma_0^3 - \gamma_s^2 \gamma_i}{\gamma_0^4 - \gamma_i^2 \gamma_s^2} \right)^2 -   \frac{\gamma_0^2 \gamma_s^2 (\gamma_0 - \gamma_i)^2}{\gamma_0^4 - \gamma_i^2 \gamma_s^2} \Big] \bigg)\,.
		\end{split}
	\end{equation*}
	Accordingly, $\mathrm{Tr}[D^{{\tt v-X}}_t(0)] - \mathrm{Tr}[D^{{\tt v-X}}_t(\zeta)] \geqslant 0$ naturally holds when $\zeta=0$. When $\zeta \in (0,1)$, it holds if $\lambda_j \leqslant \frac{\gamma_0^3 - \gamma_s^2 \gamma_i - \gamma_0^2 \gamma_s + \gamma_0 \gamma_s \gamma_i}{\gamma_0^4 - \gamma_s^2 \gamma_i^2}$ with $j=1,2,\dots,m$.
	This condition can be satisfied by
	\begin{case}[if $s=1$]
		In this case, $\gamma_1 = \gamma_0$ and we have 
		\begin{equation*}
			\lambda_j \leqslant \mathrm{Tr}(\Sigma_m) \leqslant \frac{1}{2 \gamma_0} \leqslant \frac{1}{\gamma_0 + \gamma_i} = \frac{\gamma_0^3 - \gamma_s^2 \gamma_i - \gamma_0^2 \gamma_s + \gamma_0 \gamma_s^2}{\gamma_0^4 - \gamma_s^2 \gamma_i^2} \,, \quad \mbox{when}~s=1\,.
		\end{equation*}
	\end{case}
	\begin{case}[if $s=2,3,\dots$] In this case, notice  
		\begin{equation*}
			\frac{\gamma_0^4 - \gamma_s^2 \gamma_i^2}{\gamma_0^3 - \gamma_s^2 \gamma_i - \gamma_0^2 \gamma_s + \gamma_0 \gamma_s^2} \leqslant \frac{\gamma_0^4 }{\gamma_0(\gamma_0 - \gamma_s)(\gamma_0^2 + \gamma_s \gamma_i)} \leqslant \frac{\gamma_0^3}{(\gamma_0 - \gamma_2)(\gamma_0^2 + \gamma_2 \gamma_3)} = \frac{1}{1-2^{-\zeta}} \,,
		\end{equation*}
	\end{case}
	Accordingly, we have
	\begin{equation*}
		\lambda_j \leqslant \mathrm{Tr}(\Sigma_m) \leqslant \frac{1-2^{-\zeta}}{ \gamma_0} \leqslant \frac{\gamma_0^3 - \gamma_s^2 \gamma_i - \gamma_0^2 \gamma_s + \gamma_0 \gamma_s^2}{\gamma_0^4 - \gamma_s^2 \gamma_i^2} \,,
	\end{equation*}
	where the second inequality holds by Eq.~\eqref{eqconstant}.
	Accordingly, combining the above two cases, if we choose
	\begin{equation*}
		\gamma_0 \leqslant \frac{1}{ \frac{1}{1-2^{-\zeta}} \mathrm{Tr}(\Sigma_m) } \,, \quad \mbox{for}~~ \zeta \in (0,1)\,,
	\end{equation*}
	we have $\mathrm{Tr}[D^{{\tt v-X}}_t(0)] - \mathrm{Tr}[D^{{\tt v-X}}_t(\zeta)] \geqslant 0$.
	
	In the next, we give the upper bound for $D^{{\tt v-X}}_t(0)$. The proof follows \cite[Lemmas 3 and 5]{jain2017markov} and \cite[Lemma B.4]{zou2021benign}. We just present it here for completeness. 
	We firstly demonstrate that $D^{{\tt v-X}}_t(0)$ is increasing and bounded, which implies that the limit $D^{{\tt v-X}}_{\infty}(0)$ exists, and then we seek for the upper bound of this limit. To be specific, $D^{{\tt v-X}}_t(0)$ admits the following expression
	\begin{equation*}
		D^{{\tt v-X}}_t(0) := \sum_{k=1}^{t} (I - \gamma_0 T^{\tt W})^{k-1} \circ  \gamma_0^2 B \Sigma_m = D^{{\tt v-X}}_{t-1}(0) +  (I - \gamma_0 T^{\tt W})^{t-1} \circ \gamma_0^2 B \Sigma_m \succcurlyeq D^{{\tt v-X}}_{t-1}(0)\,,
	\end{equation*}
	which implies that $D^{{\tt v-X}}_t(0)$ is increasing.
	
	Let $A_t := (I - \gamma_0 T^{\tt W})^{t-1} \circ B\Sigma_m $, and then $A_t = (I - \gamma_0 T^{\tt W}) \circ A_{t-1}$. We have
	\begin{equation*}
		\begin{split}
			\mathrm{Tr}(A_t) & = \mathrm{Tr}[(I - \gamma_0 T^{\tt W})\circ A_{t-1}] = \mathrm{Tr}(A_{t-1}) - 2\gamma_0 \mathrm{Tr}(\Sigma_m A_{t-1}) + \gamma_0^2 \mathrm{Tr}(S^{\tt W} \circ A_{t-1}) \\
			& \leqslant \mathrm{Tr}(A_{t-1}) - 2\gamma_0 \mathrm{Tr}(\Sigma_m A_{t-1}) + \gamma_0^2 r' \mathrm{Tr}(\Sigma_m A_{t-1}) \mathrm{Tr}(\Sigma_m) \quad \mbox{[using Assumption~\ref{assump:bound_fourthmoment}]} \\
			& \leqslant \mathrm{Tr}[(I - \gamma_0 \Sigma_m)A_{t-1}] \leqslant (1-\gamma_0 \lambda_m) \mathrm{Tr}(A_{t-1})\,, \quad \mbox{[using $\gamma_0 \leqslant \frac{1}{r' \mathrm{Tr}(\Sigma_m)}$]}
		\end{split}
	\end{equation*}
	which implies
	\begin{equation*}
		\mathrm{Tr}[D^{{\tt v-X}}_{t}(0)] \leqslant \gamma_0^2 \sum_{t=0}^{\infty} \mathrm{Tr} \left( (I - \gamma_0 {T}^{\tt W})^t \circ B\Sigma_m \right) \leqslant  \mathrm{Tr}(B\Sigma_m) \sum_{t=0}^{\infty} (1-\gamma_0 \lambda_m)^t \leqslant \frac{\gamma_0 \mathrm{Tr}(B\Sigma_m)}{\lambda_m} < \infty \,.
	\end{equation*}
	Accordingly, the monotonicity and boundedness of $\{ D^{{\tt v-X}}_{t}(0) \}_{t=0}^{\infty}$ implies that the limit exists, denoted as $D^{{\tt v-X}}_{\infty}(0)$ with
	\begin{equation*}
		D^{{\tt v-X}}_{\infty}(0) = (I - \gamma_0 {T}^{\tt W}) \circ D^{{\tt v-X}}_{\infty}(0) + \gamma_0^2 B \Sigma_m \,,
	\end{equation*}
	which implies $D^{{\tt v-X}}_{\infty}(0) = \gamma_0({T}^{\tt W})^{-1} \circ B\Sigma_m$
	Further, we have
	\begin{equation}\label{Twdvinf}
		\begin{split}
			\widetilde{T}^{\tt W} \circ D^{{\tt v-X}}_{\infty}(0) &= {T}^{\tt W} \circ D^{{\tt v-X}}_{\infty}(0) + \gamma_0 S^{\tt W} \circ D^{{\tt v-X}}_{\infty}(0) - \gamma_0 \widetilde{S}^{\tt W} \circ D^{{\tt v-X}}_{\infty}(0)  \quad \mbox{[definition of $\widetilde{T}^{\tt W}$]}\\
			& = \gamma_0 B\Sigma_m + \gamma_0 S^{\tt W} \circ D^{{\tt v-X}}_{\infty}(0) - \gamma_0 \widetilde{S}^{\tt W} \circ D^{{\tt v-X}}_{\infty}(0) \\
			& \preccurlyeq \gamma_0 B \Sigma_m + \gamma_0 S^{\tt W} \circ D^{{\tt v-X}}_{\infty}(0)\,. \quad \mbox{[using $S^{\tt W} \succcurlyeq \widetilde{S}^{\tt W}$]} \\
		\end{split}
	\end{equation}
	Besides, $(\widetilde{T}^{\tt W})^{-1} \circ \Sigma_m$ can be bounded by
	\begin{equation}\label{tw-Sigmam}
		\begin{split}
			(\widetilde{T}^{\tt W})^{-1} \circ \Sigma_m & = \gamma_0 \sum_{t=0}^{\infty} (I - \gamma_0 \widetilde{T}^{\tt W}) \circ \Sigma_m = \gamma_0 \sum_{t=0}^{\infty} (I - \gamma_0 \Sigma_m)^t \Sigma_m(I - \gamma_0 \Sigma_m)^t \\
			& \preccurlyeq \gamma_0 \sum_{t=0}^{\infty} (I - \gamma_0 \Sigma_m)^t\Sigma_m = I \,. \quad \mbox{[using $\gamma_0 \leqslant 1/\mathrm{Tr}(\Sigma_m)$]}
		\end{split}
	\end{equation}
	
	Therefore, $D^{{\tt v-X}}_{\infty}(0) $ can be further upper bounded by
	\begin{equation}\label{dtv-xinf}
		\begin{split}
			D^{{\tt v-X}}_{\infty}(0) & \preccurlyeq \gamma_0 ( \widetilde{T}^{\tt W} )^{-1} \circ B \Sigma_m + \gamma_0 ( \widetilde{T}^{\tt W} )^{-1} \circ S^{\tt W}  \circ D^{{\tt v-X}}_{\infty}(0) \quad \mbox{[using Eq.~\eqref{Twdvinf}]} \\
			& \preccurlyeq \gamma_0 B + \gamma_0 ( \widetilde{T}^{\tt W} )^{-1} \circ S^{\tt W}  \circ D^{{\tt v-X}}_{\infty}(0) \quad \mbox{[using Eq.~\eqref{tw-Sigmam}]} \\
			& = \gamma_0 B \sum_{t=0}^{\infty} [\gamma_0 ( \widetilde{T}^{\tt W} )^{-1} \circ S^{\tt W} ]^t  \circ I \quad \mbox{[solving the recursion]} \\
			& \preccurlyeq \gamma_0 B \sum_{t=0}^{\infty} \left( \gamma_0 ( \widetilde{T}^{\tt W} )^{-1} \circ S^{\tt W} \right)^{t-1} \circ \gamma_0 ( \widetilde{T}^{\tt W} )^{-1} \circ S^{\tt W} \circ I \\
			& \preccurlyeq \gamma_0 B \sum_{t=0}^{\infty}  \left( \gamma_0 ( \widetilde{T}^{\tt W} )^{-1} \circ S^{\tt W} \right)^{t-1} \circ  \gamma_0 ( \widetilde{T}^{\tt W} )^{-1} \circ \mathrm{Tr}(\Sigma_m) \Sigma_m \quad \mbox{[using Assumption~\ref{assump:bound_fourthmoment}]} \\
			& \preccurlyeq \gamma_0 B \sum_{t=0}^{\infty}  \left[  \gamma_0 r' \mathrm{Tr}(\Sigma_m) \right]^{t} \circ I \quad \mbox{[using Eq.~\eqref{tw-Sigmam}]} \\
			& \preccurlyeq \frac{\gamma_0 B}{1 - \gamma_0 r' \mathrm{Tr}(\Sigma_m)} I\,. \quad \mbox{[using $\gamma_0 < \frac{1}{r'\mathrm{tr}(\Sigma_m)}$]}
		\end{split}
	\end{equation}
	Hence, based on the above results, $D^{{\tt v-X}}_t(0)$ can be further upper bounded by
	\begin{equation}\label{dvxtfinal}
		\begin{split}
			D^{{\tt v-X}}_t(0)	& = (I - \gamma_0 {T}^{\tt W}) \circ D^{{\tt v-X}}_{t-1}(0)  + \gamma_0^2 B \Sigma_m  \\
			& = (I - \gamma_0 \widetilde{T}^{\tt W}) \circ D^{{\tt v-X}}_{t-1}(0) + \gamma_0^2 ({S}^{\tt W} - \widetilde{S}^{\tt W} ) \circ D^{{\tt v-X}}_{t-1} +\gamma_0^2 B \Sigma_m \\
			& \preccurlyeq (I - \gamma_0 \widetilde{T}^{\tt W}) \circ D^{{\tt v-X}}_{t-1}(0) + \gamma_0^2 {S}^{\tt W} \circ D^{{\tt v-X}}_{\infty}(0) + \gamma_0^2 B\Sigma_m \\
			& \preccurlyeq (I - \gamma_0 \widetilde{T}^{\tt W}) \circ D^{{\tt v-X}}_{t-1}(0) + \gamma_0^2 r' \mathrm{Tr}[D^{{\tt v-X}}_{\infty}(0)] \mathrm{Tr}(\Sigma_m) \Sigma_m + \gamma_0^2 B \Sigma_m \quad \mbox{[using Assumption~\ref{assump:bound_fourthmoment}]} \\
			&  \preccurlyeq (I - \gamma_0 \widetilde{T}^{\tt W}) \circ D^{{\tt v-X}}_{t-1}(0) + \gamma_0^2 B \Sigma_m \left( \frac{\mathrm{Tr}(\Sigma_m)r'\gamma_0}{1-\gamma_0 r' \mathrm{Tr}(\Sigma_m)} + 1 \right) \quad \mbox{[using Eq.~\eqref{dtv-xinf}]} \\
			& \preccurlyeq \gamma_0^2 B \left( \frac{\mathrm{Tr}(\Sigma_m)r'\gamma_0}{1-\gamma_0 r' \mathrm{Tr}(\Sigma_m)} + 1 \right) \sum_{k=0}^{\infty} (I - \gamma_0 \Sigma_m)^k \Sigma_m \\
			& \preccurlyeq \gamma_0 B \left( \frac{\mathrm{Tr}(\Sigma_m)r'\gamma_0}{1-\gamma_0 r' \mathrm{Tr}(\Sigma_m)} + 1 \right) I\,,
		\end{split}
	\end{equation}
	which concludes the proof.
\end{proof}

\section{Some useful integrals estimation}
\label{app:int}
In this section, we present the estimation for the following integrals that will be needed in our proof by denoting $\kappa := 1-\zeta \in (0,1]$.

{\bf Integral 1:}
We consider the following integral admitting an exact estimation
\begin{equation}\label{eq:int1}
    \int_1^{t} u^{-\zeta} \exp \bigg( -c  \frac{u^{1-\zeta} - 1}{1-\zeta} \bigg) \mathrm{d}u \leqslant t\,.
\end{equation}
Besides, we also calculate this integral as below:
by changing the integral variable $v^\kappa := c  \frac{u^{1-\zeta} - 1}{1-\zeta} $ and
\begin{equation*}
	\frac{\mathrm{d}v}{\mathrm{d}u} = u^{1-\kappa} \left( \frac{\kappa}{c} \right)^{\frac{1}{\kappa}} \left( u^\kappa -1 \right)^{\frac{\kappa - 1}{\kappa}} = \frac{1}{c}u^{1-\kappa} \kappa v^{\kappa-1} \,,
\end{equation*}
and then we have
\begin{equation}\label{Iifirst}
	\begin{split}
		\int_1^{t} u^{-\zeta} \exp \bigg( -c  \frac{u^{1-\zeta} - 1}{1-\zeta} \bigg) \mathrm{d}u 
		& = \frac{1}{c} \int_0^{[\frac{c}{\kappa} (t^{\kappa}-1)]^{\frac{1}{\kappa}}} u^{-\zeta} u^{1-\kappa} \kappa v^{\kappa-1} \exp(-v^\kappa) \mathrm{d} v \\
		& \leqslant \frac{1}{c} \int_{0}^{\infty} \exp(-x) \mathrm{d} x = \left( \frac{1}{c} \wedge t \right)\,,
	\end{split}
\end{equation}
where the last equality uses \cref{eq:int1} and takes the smaller one via the notation $\wedge$.
Accordingly, if we take $\zeta = 0$ in Eq.~\eqref{Iifirst}, we have
\begin{equation}\label{Iexp}
	\int_1^{t} \exp \bigg( -c  \frac{u^{1-\zeta} - 1}{1-\zeta} \bigg) \mathrm{d}u \leqslant \left( \frac{1}{c}t^{\zeta} \wedge t \right)\,.
\end{equation}

Similar to Eq.~\eqref{Iexp}, we have
\begin{equation}\label{intut}
	\begin{split}
		&	\int_t^{n} \exp \bigg( -\widetilde{\lambda}_i \gamma_0  \frac{u^{1-\zeta} - t^{1-\zeta}}{1-\zeta} \bigg) \mathrm{d}u  \leqslant (n-t) \wedge \frac{n^\zeta}{\widetilde{\lambda}_i \gamma_0}\,.
	\end{split}
\end{equation}

{\bf Integral 2:} we consider the following integral 
\begin{equation}\label{Iisecond}
	\begin{split}
		&  \int_1^{t} u^{-\zeta} \exp \bigg( -c  \frac{(t+1)^{1-\zeta} - (u+1)^{1-\zeta}}{1-\zeta} \bigg) \mathrm{d}u \\
		& =  \frac{(t+1)^{1-\kappa}}{c} \int_{0}^{C} [(t+1)(1-x)^{\frac{1}{\kappa}} - 1]^{\kappa - 1} (1-x)^{\frac{1-\kappa}{\kappa}} \kappa v^{\kappa - 1} \exp(-v^{\kappa}) \mathrm{d} v \quad \mbox{with}~~x:= (\frac{v}{t+1} )^{\kappa} \frac{\kappa}{c}   \\
		& \leqslant \frac{2^{\zeta}}{c}\int_0^{\infty} \kappa v^{\kappa-1} \exp(-v^\kappa ) \mathrm{d} v  \\
		& = \left( \frac{2^{\zeta}}{c} \wedge t \right)\,,
	\end{split}
\end{equation}
where we change the integral variable $v^{\kappa} := c \frac{(t+1)^{1-\zeta} - (u+1)^{1-\zeta}}{1-\zeta}$ with $\kappa := 1 - \zeta$ such that
\begin{equation*}
	\mathrm{d}u = - \frac{\kappa^{1/\kappa}}{c^{1/\kappa}} \big( \frac{u+1}{t+1} \big)^{1-\kappa} \left[ 1 - \big( \frac{u+1}{t+1}\big)^{\kappa} \right]^{1 - \frac{1}{\kappa}} \mathrm{d} v = - \frac{\kappa}{c} \left[ 1 - \big( \frac{v}{t+1}\big)^{\kappa} \frac{\kappa}{c} \right]^{\frac{1-\kappa}{\kappa}} \left( \frac{v}{t+1} \right)^{\kappa - 1} \mathrm{d} v \,,
\end{equation*}
with $(\frac{u+1}{t+1})^{\kappa} = 1 - (v/(t+1))^{\kappa} \kappa/c$ and the upper limit of integral is $C:= c^{1/\kappa} [(t+1)^\kappa - (u+1)^\kappa]^{1/\kappa}$. Due to $u = (t+1)(1-x)^{\frac{1}{\kappa}}-1 \in [1,t]$, we have $(1-x)^{\frac{1}{\kappa}} \in [2/(t+1),1]$ and accordingly
\begin{equation*}
	g(x):= [(t+1)(1-x)^{\frac{1}{\kappa}} - 1]^{\kappa - 1} (1-x)^{\frac{1-\kappa}{\kappa}} \leqslant 2^{1-\kappa} (t+1)^{\kappa -1} \quad \mbox{with}~~ x \in \left[0, 1- \left(\frac{2}{t+1} \right)^\kappa \right]\,,
\end{equation*}
as an increasing function of $x$.

Similar to Eq.~\eqref{Iisecond}, we have the following estimation
\begin{equation}\label{intu2zetatu}
	\int_1^t \gamma_0^2 u^{-2\zeta} \exp \bigg( -2\widetilde{\lambda}_i \gamma_0  \frac{(t+1)^{1-\zeta} - (u+1)^{1-\zeta}}{1-\zeta} \bigg) \mathrm{d}u \lesssim \left( \frac{\gamma_0}{\widetilde{\lambda}_i} \wedge \gamma_0^2 t \right)\,.
\end{equation}

\section{Proofs for ${\tt Bias}$}
\label{app:bias}
In this section, we present the error bound for ${\tt Bias}$.
By virtue of Minkowski inequality, we have
\begin{equation}\label{biasdecom}
	\begin{split}
		\Big( \mathbb{E}_{\bm X, \bm W} \big[ \langle \bar{\eta}^{{\tt bias}}_n, \Sigma_m \bar{\eta}^{{\tt bias}}_n \rangle \big] \Big)^{\frac{1}{2}} &  \!\!\leqslant \!\! \Big( \underbrace{ \mathbb{E}_{\bm X, \bm W} \big[ \langle \bar{\eta}^{{\tt bias}}_n - \bar{\eta}^{{\tt bX}}_n, \Sigma_m ( \bar{\eta}^{{\tt bias}}_n - \bar{\eta}^{{\tt bX}}_n ) \rangle \big] }_{ \triangleq \tt B1} \Big)^{\frac{1}{2}} + \Big( \mathbb{E}_{\bm W} \big[ \langle \bar{\eta}^{{\tt bX}}_n, \Sigma_m \bar{\eta}^{{\tt bX}}_n \rangle \big] \Big)^{\frac{1}{2}} \\
		& \!\!\leqslant \!\! ({\tt B1})^{\frac{1}{2}} \!+\! \Big( \underbrace{ \mathbb{E}_{\bm W} \big[ \langle \bar{\eta}^{{\tt bX}}_n \!-\! \bar{\eta}^{{\tt bXW}}_n, \Sigma_m ( \bar{\eta}^{{\tt bX}}_n \!-\! \bar{\eta}^{{\tt bXW}}_n ) \rangle \big] }_{ \triangleq \tt B2} \Big)^{\frac{1}{2}} \!\!+\! [ \underbrace{ \langle   \bar{\eta}^{{\tt bXW}}_n, \widetilde{\Sigma}_m \bar{\eta}^{{\tt bXW}}_n \rangle}_{\triangleq {\tt B3}} ]^{\frac{1}{2}} \,.
	\end{split}
\end{equation}

In the next, we give the error bounds for ${\tt B3}$, ${\tt B2}$, and ${\tt B1}$, respectively.

\subsection{Bound for  ${\tt B3}$}
\label{app:b3}
In this section, we aim to bound $ {\tt B3} := \langle \bar{\eta}^{{\tt bXW}}_n, \widetilde{\Sigma}_m \bar{\eta}^{{\tt bXW}}_n \rangle$.
\begin{proposition}\label{propb3}
	Under Assumption~\ref{assumdata},~\ref{assexist},~\ref{assumact}, if the step-size $\gamma_t := \gamma_0 t^{-\zeta}$ with $\zeta \in [0,1)$ satisfies
	$\gamma_0 \leqslant \frac{1}{\mathrm{Tr}(\widetilde{\Sigma}_m)}$, 
	then ${\tt B3}$ can be bounded by
	\begin{equation*}
		{\tt B3} \lesssim  \frac{n^{\zeta-1}}{\gamma_0} \| f^* \|^2 \,.
	\end{equation*}
\end{proposition}

\begin{proof}
	Due to $\gamma_0 \leqslant \frac{1}{\mathrm{Tr}(\widetilde{\Sigma}_m)}$, the operator $I - \gamma_t \widetilde{\Sigma}_m$ is a contraction map for $t=1,2,\dots,n$.
	Take spectral decomposition $\widetilde{\Sigma}_m = \widetilde{U} \widetilde{\Lambda} \widetilde{U}^{\!\top}$ where $\widetilde{U}$ is an orthogonal matrix and $\widetilde{\Lambda}$ is a diagonal matrix with $(\widetilde{\Lambda})_{11} = \widetilde{\lambda}_1$ and $( \widetilde{\Lambda})_{ii} = \widetilde{\lambda}_2$ ($i=2,3,\dots,m$) as $\widetilde{\Sigma}_m$ has only two distinct eigenvalues in Lemma~\ref{thmH}. Accordingly, we have
	\begin{equation}\label{boundB3ada}
		\begin{split}
			\langle \bar{\eta}^{{\tt bXW}}_n, \widetilde{\Sigma}_m \bar{\eta}^{{\tt bXW}}_n \rangle &  = \frac{1}{n^2} \left \langle  \sum_{t=0}^{n-1} \prod_{i=1}^{t}(I - \gamma_i \widetilde{\Sigma}_m) f^*, \widetilde{\Sigma}_m \sum_{t=0}^{n-1} \prod_{i=1}^{t}(I - \gamma_i \widetilde{\Sigma}_m) f^*  \right \rangle \\
			& = \frac{1}{n^2} \left\| \sum_{t=0}^{n-1} \prod_{i=1}^{t}(I - \gamma_i \widetilde{\Sigma}_m) \widetilde{\Sigma}_m^{\frac{1}{2}} f^* \right\|^2 \\
			& \leqslant \frac{1}{n^2} \left\| \sum_{t=0}^{n-1} \prod_{i=1}^{t}(I - \gamma_i \widetilde{\Lambda}) \widetilde{\Lambda}^{\frac{1}{2}}  \right\|^2_2  \left\| f^* \right\|^2 \quad \mbox{[using $\widetilde{ \Sigma}_m = \widetilde{U} \widetilde{\Lambda} \widetilde{U}^{\!\top}$]} \\
			%		& = \frac{1}{n^2} \left \| \begin{bmatrix}
				%			\sum \limits_{t=0}^{n-1} \prod_{i=1}^{t}(1-\gamma_i \widetilde{\lambda}_1) \widetilde{\lambda}_1^{\frac{1}{2}} &   & &  \\
				%			& \sum \limits_{t=0}^{n-1} \prod_{i=1}^{t}(1-\gamma_i \widetilde{\lambda}_2) \widetilde{\lambda}_2^{\frac{1}{2}} &  & \\
				%			&   & \cdots&\\
				%			&   &  &\sum \limits_{t=0}^{n-1} \prod_{i=1}^{t}(1-\gamma_i \widetilde{\lambda}_2) \widetilde{\lambda}_2^{\frac{1}{2}}
				%		\end{bmatrix} \right \|_2^2 \left\| f^* \right\|^2 \\
			& \leqslant \frac{1}{n} \max_{k=1,2} \sum\limits_{t=0}^{n-1} \prod_{i=1}^{t}(1-\gamma_i \widetilde{\lambda}_k)^2 \widetilde{\lambda}_k \| f^* \|^2
			\\
			& \leqslant \frac{1}{n} \sum\limits_{t=0}^{n-1} \prod_{i=1}^{t}(1-\gamma_i \widetilde{\lambda}_1)^2 \widetilde{\lambda}_1 \| f^* \|^2 + \frac{1}{n} \sum\limits_{t=0}^{n-1} \prod_{i=1}^{t}(1-\gamma_i \widetilde{\lambda}_2)^2 \widetilde{\lambda}_2 \| f^* \|^2 \,.
		\end{split} 
	\end{equation}
	Note that
	\begin{equation}\label{lambda12exp}
		\begin{split}
			\sum_{t=0}^{n-1} \prod_{i=1}^{t}(1-\gamma_i \widetilde{\lambda}_j)^2 & \leqslant  	\sum_{t=0}^{n-1} \exp \left(-2\gamma_0 \widetilde{\lambda}_j \sum_{i=1}^t i^{-\zeta}\right) \leqslant  \sum_{t=0}^{n-1} \exp \left(-2\gamma_0 \widetilde{\lambda}_j \int_{1}^{t+1} \frac{1}{x^{\zeta}} \mathrm{d} x\right) \\
			& = \sum_{t=0}^{n-1} \exp \left(-2\gamma_0 \widetilde{\lambda}_j \frac{(t+1)^{1-\zeta}-1}{1-\zeta} \right) \\
			& \leqslant 1 + \int_0^n \exp \left(-2\gamma_0 \widetilde{\lambda}_j \frac{(t+1)^{1-\zeta}-1}{1-\zeta} \right) \mathrm{d} x \\
			& \leqslant 1 + \left( \frac{n^{\zeta}}{2\gamma_0 \widetilde{\lambda}_j} \wedge n \right) \,, \quad \mbox{[using Eq.~\eqref{Iexp}]}
		\end{split}
	\end{equation}
	here according to Lemma~\ref{thmH}, for $\widetilde{\lambda}_1$, the upper bound $\frac{n^{\zeta}}{2\gamma_0 \widetilde{\lambda}_1}$ is tighter than $n$ due to $\widetilde{\lambda}_1 \sim \mathcal{O}(1)$; while this conclusion might not hold for $\widetilde{\lambda}_2$ due to $\widetilde{\lambda}_2 \sim \mathcal{O}(1/m)$.
	Then, taking Eq.~\eqref{lambda12exp} back to Eq.~\eqref{boundB3ada}, we have
	\begin{equation}\label{boundB3adafin}
		\begin{split}
			\langle \bar{\eta}^{{\tt bXW}}_n, \widetilde{\Sigma}_m \bar{\eta}^{{\tt bXW}}_n \rangle &
			\lesssim \frac{n^{\zeta-1}}{\gamma_0} \| f^* \|^2 + \frac{\widetilde{\lambda}_2}{n} \left( \frac{n^{\zeta}}{\gamma_0 \widetilde{\lambda}_2} \wedge n \right) \| f^*\|^2  \\
			& \lesssim \frac{n^{\zeta-1}}{\gamma_0} \| f^* \|^2 \sim \mathcal{O}(n^{\zeta -1} ) \,,
		\end{split} 
	\end{equation}
	which concludes the proof.
\end{proof}

\subsection{Bound for  ${\tt B2}$}
\label{app:b2}

Here we aim to bound $ {\tt B2} := \mathbb{E}_{\bm W} \big[ \langle \bar{\eta}^{{\tt bX}}_n \!-\! \bar{\eta}^{{\tt bXW}}_n, \Sigma_m ( \bar{\eta}^{{\tt bX}}_n \!-\! \bar{\eta}^{{\tt bXW}}_n ) \rangle \big] = \mathbb{E}_{\bm W} \big[ \langle \bar{\alpha}^{{\tt W}}_n, \widetilde{\Sigma}_m \bar{\alpha}^{{\tt W}}_n  \rangle \big] + \mathbb{E}_{\bm W} \big[ \langle \bar{\alpha}^{{\tt W}}_n, (\Sigma_m - \widetilde{\Sigma}_m) \bar{\alpha}^{{\tt W}}_n  \rangle \big]$, where 
\begin{equation}\label{bxwdiffada}
	\alpha^{{\tt W}}_t :=	{\eta}^{{\tt bX}}_t - {\eta}^{{\tt bXW}}_{t} = (I - \gamma_t \Sigma_m) ({\eta}^{{\tt bX}}_{t-1} - {\eta}^{{\tt bXW}}_{t-1}) + \gamma_t (\widetilde{\Sigma}_m - \Sigma_m) {\eta}^{{\tt bXW}}_{t-1}\,,
\end{equation}
with $\alpha^{{\tt W}}_0 = 0$. %Here the superscript ${\tt W}$ in $\alpha^{{\tt W}}_t$ denotes the dependency on $\bm W$.
Here $\alpha^{{\tt W}}_t$ can be further formulated as 
\begin{equation}\label{awtada}
	\alpha^{{\tt W}}_t = \sum_{k=1}^t \gamma_k \prod_{j=k+1}^{t} (I - \gamma_j \Sigma_m) (\widetilde{\Sigma}_m - \Sigma_m) \prod_{s=1}^{k-1} (I- \gamma_s \widetilde{\Sigma}_m) f^* \,,
\end{equation}
where we use the recursion
\begin{equation*}
	A_t := (I - \gamma_t \Sigma_m) A_{t-1} + B_t = \sum_{s=1}^t \prod_{i=s+1}^t (I - \gamma_i \Sigma_m) B_s\,.
\end{equation*}
Accordingly, ${\tt B2}$ admits
\begin{equation}\label{b2equ}
	{\tt B2} = \mathbb{E}_{\bm W} \big[ \langle \bar{\alpha}^{{\tt W}}_n, {\Sigma}_m \bar{\alpha}^{{\tt W}}_n  \rangle \big] = \frac{1}{n^2} \mathbb{E}_{\bm W} \left\langle \sum_{t=0}^{n-1} \alpha_t^{{\tt W}}, \Sigma_m \sum_{t=0}^{n-1} \alpha_t^{{\tt W}}  \right\rangle = \frac{1}{n^2} \mathbb{E}_{\bm W} \left\| \Sigma_m^{\frac{1}{2}}  \sum_{t=0}^{n-1} \alpha_t^{{\tt W}} \right\|^2\,,
\end{equation}
and we have the following error bound for ${\tt B2}$.
\begin{proposition}\label{propb2}
	Under Assumption~\ref{assumdata},~\ref{assexist},~\ref{assumact}, if the step-size $\gamma_t := \gamma_0 t^{-\zeta}$ with $\zeta \in [0,1)$ satisfies
	\begin{equation*}
		\gamma_0 \leqslant \min \left\{ \frac{1}{\mathrm{Tr}(\Sigma_m)}, \frac{1}{\mathrm{Tr}(\widetilde{\Sigma}_m)} \right\} \,,
	\end{equation*}
	then ${\tt B2}$ can be bounded by
	\begin{equation*}
		\begin{split}
			{\tt B2} \lesssim \frac{ \| f^*\|^2}{\gamma_0} n^{\zeta-1}\,.
		\end{split}
	\end{equation*}
\end{proposition}
{\bf Remark:} In our paper, we require $I - \gamma_t \Sigma_m$ ($t=1,2,\dots m$) to be a contraction map.
Though $\mathrm{Tr}(\Sigma_m)$ is a random variable, it is with a sub-exponential $\mathcal{O}(1)$ norm, that means, the condition $\gamma_0 < 1/\mathrm{Tr}(\Sigma_m)$ can be equivalently substituted by $\gamma_0 < 1/[c\mathrm{Tr}(\widetilde{\Sigma}_m)]$ for some large $c$ (independent of $n$, $m$, $d$) with exponentially high probability.
This is also used for estimating other quantities.
%For example, the probability with $\exp(-10) < 10^{-4}$ and $\exp(-100) < 10^{-43}$ by taking $c=10$, or $100$ is enough small in practice.

Before we present the error bounds for ${\tt B2}$, we need the following lemma.
\begin{lemma}\label{lemintb2}
	Under Assumption~\ref{assumdata},~\ref{assexist},~\ref{assumact}, if the step-size $\gamma_t := \gamma_0 t^{-\zeta}$ with $\zeta \in [0,1)$ satisfies
	\begin{equation*}
		\gamma_0 \leqslant \min \left\{ \frac{1}{\mathrm{Tr}(\Sigma_m)}, \frac{1}{\mathrm{Tr}(\widetilde{\Sigma}_m)} \right\} \,,
	\end{equation*}
	denote $\varUpsilon_{i}:=  \sum_{t=0}^{n-1} \sum_{k=1}^{t} \gamma_k ( \widetilde{\lambda}_i - \lambda_i)\lambda_i^{\frac{1}{2}} \prod_{j=k+1}^{t}(1-\gamma_j {\lambda}_i) \prod_{s=1}^{k-1}(1-\gamma_s \widetilde{\lambda}_i)$, $\forall i \in [m]$, we have
	\begin{equation*}
		\begin{split}
			\varUpsilon_{i}  \lesssim \lambda_i^{\frac{1}{2}} \left( \frac{n^{\zeta}}{\gamma_0 \lambda_i} \wedge n \right)\,, \quad \mbox{if}~\lambda_i \neq 0\,; \quad \mbox{and}~ \varUpsilon_{i} = 0 \,, \quad \mbox{if}~\lambda_i = 0\,.
		\end{split}
	\end{equation*}
\end{lemma}
\begin{proof}
	Following the derivation in \cref{app:int}, we consider the index $i$ with $\lambda_i \neq 0$ such that
	\begin{equation}\label{Iib2}
		\begin{split}
			\varUpsilon_{i} & := \sum_{t=0}^{n-1}\sum_{k=1}^{t} \gamma_k ( \widetilde{\lambda}_i - \lambda_i)\lambda_i^{\frac{1}{2}} \prod_{j=k+1}^{t}(1-\gamma_j {\lambda}_i) \prod_{s=1}^{k-1}(1-\gamma_s \widetilde{\lambda}_i) \\
			&  \leqslant \sum_{t=0}^{n-1}( \widetilde{\lambda}_i - \lambda_i)\lambda_i^{\frac{1}{2}} \sum_{k=1}^{t} \gamma_k \exp \big( - \sum_{j=k+1}^{t} \gamma_j \lambda_i  \big) \exp\big( -\sum_{s=1}^{k-1}\gamma_s \widetilde{\lambda}_i \big) \\
			& \leqslant \sum_{t=0}^{n-1} ( \widetilde{\lambda}_i - \lambda_i)\lambda_i^{\frac{1}{2}} \sum_{k=1}^{t} \gamma_k \exp \bigg( -{\lambda}_i  \int_{k+1}^{t+1} \frac{\gamma_0}{x^{\zeta}} \mathrm{d}x \bigg) \exp \bigg( -\widetilde{\lambda}_i  \int_{1}^{k} \frac{\gamma_0}{x^{\zeta}} \mathrm{d}x \bigg) \\
			& = \sum_{t=0}^{n-1} ( \widetilde{\lambda}_i - \lambda_i)\lambda_i^{\frac{1}{2}} \sum_{k=1}^{t} \gamma_0 k^{-\zeta} \exp \bigg( -{\lambda}_i \gamma_0 \frac{(t+1)^{1-\zeta} - (k+1)^{1-\zeta}}{1-\zeta} \bigg) \exp \bigg( -\widetilde{\lambda}_i \gamma_0 \frac{k^{1-\zeta} - 1}{1-\zeta} \bigg) \\
			& \lesssim \sum_{t=0}^{n-1} \gamma_0( \widetilde{\lambda}_i - \lambda_i)\lambda_i^{\frac{1}{2}} \Bigg[ \int_{1}^{t} u^{-\zeta} \exp \bigg( - \gamma_0 \frac{{\lambda}_i(t+1)^{1-\zeta} - {\lambda}_iu^{1-\zeta} + \widetilde{\lambda}_i u^{1-\zeta} - \widetilde{\lambda}_i}{1-\zeta} \bigg) \mathrm{d}u \\
			& \hspace{3.5cm} + t^{-\zeta} \exp \bigg( -\widetilde{\lambda}_i \gamma_0 \frac{t^{1-\zeta} - 1}{1-\zeta} \bigg) \Bigg]\,,  \\
		\end{split}
	\end{equation}
	Denote $\kappa:= 1- \zeta$ and
	\begin{equation*}
		v^\kappa := \gamma_0 \frac{{\lambda}_i(t+1)^{1-\zeta} - {\lambda}_iu^{1-\zeta} + \widetilde{\lambda}_i u^{1-\zeta} - \widetilde{\lambda}_i}{1-\zeta}\,,
	\end{equation*}
	by changing the integral variable $u$ to $v$, we have 
	\begin{equation*}
		\begin{split}
			\frac{\mathrm{d}u}{\mathrm{d}v} & = \frac{u^{1-\kappa}}{\widetilde{\lambda}_i - \lambda_i} \left( \frac{\gamma_0}{\kappa} \right)^{-1/\kappa} [ (\widetilde{\lambda}_i - \lambda_i)u^{\kappa} + \lambda_i (t+1)^\kappa - \widetilde{\lambda}_i ] = \frac{u^{1-\kappa}}{\widetilde{\lambda}_i - \lambda_i} \frac{\kappa}{\gamma_0} v^{\kappa-1}\,,
		\end{split}
	\end{equation*}
	Accordingly, Eq.~\eqref{Iib2} can be upper bounded by
	\begin{equation*}
		\begin{split}
			\varUpsilon_{i} & \lesssim \sum_{t=0}^{n-1} \gamma_0( \widetilde{\lambda}_i - \lambda_i)\lambda_i^{\frac{1}{2}} \bigg[ \int_{1}^{t} u^{-\zeta} \exp \bigg( - \gamma_0 \frac{{\lambda}_i(t+1)^{1-\zeta} - {\lambda}_iu^{1-\zeta} + \widetilde{\lambda}_i u^{1-\zeta} - \widetilde{\lambda}_i}{1-\zeta} \bigg) \mathrm{d}u \\
			& \qquad \qquad \qquad \qquad \qquad + t^{-\zeta}\exp \bigg( -\widetilde{\lambda}_i \gamma_0 \frac{t^{1-\zeta} - 1}{1-\zeta} \bigg) \bigg] \\
			& \leqslant \sum_{t=0}^{n-1} \gamma_0( \widetilde{\lambda}_i - \lambda_i)\lambda_i^{\frac{1}{2}} \left[ \int_{c_1^{\frac{1}{\kappa}}}^{c_2^{\frac{1}{\kappa}}} u^{-\zeta} \exp(-v^{\kappa}) \frac{1}{\widetilde{\lambda}_i - \lambda_i} u^{1-\kappa} \frac{\kappa}{\gamma_0} v^{\kappa-1} \mathrm{d} v + t^{-\zeta} \exp \bigg( -\widetilde{\lambda}_i \gamma_0 \frac{t^{1-\zeta} - 1}{1-\zeta} \bigg) \right] \\
			& = \sum_{t=0}^{n-1} \left[  \lambda_i^{\frac{1}{2}} \int_{c_1}^{c_2}  \exp(-x) \mathrm{d}x + \gamma_0( \widetilde{\lambda}_i - \lambda_i)\lambda_i^{\frac{1}{2}} t^{-\zeta} \exp \bigg( -\widetilde{\lambda}_i \gamma_0 \frac{t^{1-\zeta} - 1}{1-\zeta} \bigg) \right]  \\
			& \lesssim \lambda_i^{\frac{1}{2}} \int_{0}^n \exp \bigg( -{\lambda}_i \gamma_0 \frac{(u+1)^{1-\zeta} - 1}{1-\zeta} \bigg) \mathrm{d} u \\
			& \leqslant \lambda_i^{\frac{1}{2}} \left( \frac{n^{\zeta}}{\gamma_0 \lambda_i} \wedge n \right)\,, \quad \mbox{[using Eq.~\eqref{intut}]}
		\end{split}
	\end{equation*}
	where $c_1:= \frac{\gamma_0}{\kappa} \lambda_i [(t+1)^{\kappa} - 1]$ and $c_2 :=  \frac{\gamma_0}{\kappa} \widetilde{\lambda}_i [t^{\kappa} - 1]$.
	Finally, we conclude the proof.
\end{proof}

In the next, we are ready to present the error bounds for ${\tt B2}$.
\begin{proof}[Proof of Proposition~\ref{propb2}]
	
	According to Eq.~\eqref{b2equ},  we need estimation for $\left\| \Sigma_m^{\frac{1}{2}} \sum_{t=0}^{n-1} \alpha^{{\tt W}}_t \right\|_2$ for estimating ${\tt B2}$.
	By spectrum decomposition, we have $\prod_{j=k+1}^{t} (I - \gamma_j \Sigma_m) = U \left( \prod_{j=k+1}^{t} (I - \gamma_j \Lambda) \right) U^{\!\top}$ and $\prod_{s=1}^{k-1} (I- \gamma_s \widetilde{\Sigma}_m) = \widetilde{U} \prod_{s=1}^{k-1} (I- \gamma_s \widetilde{\Lambda} ) \widetilde{U}^{\!\top} $.
	%Note that the matrix $\prod_{j=k+1}^{t} (I - \gamma_j \Sigma_m) \prod_{s=1}^{k-1} (I- \gamma_s \widetilde{\Sigma}_m)$ can be simultaneously diagonalized as $(I- \gamma_s \widetilde{\Sigma}_m)$ is symmetric and positive definite (a contraction map) and  $(I - \gamma_j \Sigma_m)$ is symmetric and positive (semi)-definite.
	%That means, $U = \widetilde{U}$, and accordingly, we have
	Then we have
	\begin{equation}\label{sigmaawt}
		\begin{split}
			\left\| \Sigma_m^{1/2}\sum_{t=0}^{n-1}{\alpha}^{{\tt W}}_t \right\|_2 &=  \left\|\sum_{t=0}^{n-1}  \sum_{k=1}^t \gamma_k \prod_{j=k+1}^{t} (I - \gamma_j \Sigma_m) (\widetilde{\Sigma}_m - \Sigma_m) \prod_{s=1}^{k-1} (I- \gamma_s \widetilde{\Sigma}_m)\Sigma_m^{\frac{1}{2}} f^*  \right\|_2 \\
			& = \left\| \sum_{t=0}^{n-1} \sum_{k=1}^t \gamma_k \prod_{j=k+1}^{t} (I - \gamma_j \Lambda_m) (\widetilde{\Lambda}_m - \Lambda_m) \prod_{s=1}^{k-1} (I- \gamma_s \widetilde{\Lambda}_m)\Lambda_m^{\frac{1}{2}} f^*  \right\|_2 \\
			& \leqslant \max_{i \in \{1,2,\dots, m \}} \sum_{t=0}^{n-1} \sum_{k=1}^{t} \gamma_k ( \widetilde{\lambda}_i - \lambda_i)\lambda_i^{\frac{1}{2}} \prod_{j=k+1}^{t}(1-\gamma_j {\lambda}_i) \prod_{s=1}^{k-1}(1-\gamma_s \widetilde{\lambda}_i) \| f^* \|\,,
		\end{split}
	\end{equation}
	where the second equality holds by $\| \bm A \bm B \|_2 = \| \bm B \bm A\|_2$ for any two PSD matrices.
	
	By Lemma~\ref{lemintb2}, we have
	\begin{equation}
		\begin{split}
			{\tt B2} & = \frac{1}{n^2} \mathbb{E}_{\bm W} \left\| \Sigma_m^{\frac{1}{2}}  \sum_{t=0}^{n-1} \alpha_t^{{\tt W}} \right\|^2 = \frac{1}{n^2} \mathbb{E}_{\bm W} \left\| \max_{i \in \{ 1,2,\dots,m\}}  \varUpsilon_{i} \right\|^2 \lesssim \frac{1}{n^2} \mathbb{E}_{\bm W} \left[ \lambda_i^{\frac{1}{2}} \left( \frac{n^{\zeta}}{\gamma_0 \lambda_i} \wedge n \right)\right]^2 \| f^*\|^2 \\
			& := \| f^*\|^2 \mathbb{E}_{\bm W} \left[\frac{n^{2(1-\zeta)}}{\gamma_0^2 \lambda_{i^*}} \wedge \lambda_{i^*} \right] = \| f^*\|^2 \left\{
			\begin{array}{rcl}
				\begin{split}
					& \mathbb{E}_{\bm W} \left[\frac{n^{2(1-\zeta)}}{\gamma_0^2 \lambda_{i^*}} \right] ,~\quad \mbox{if $\lambda_{i^*} \geqslant \frac{n^{\zeta-1}}{\gamma_0}$} \\
					&   \mathbb{E}_{\bm W} [\lambda_{i^*}] ,~\quad \mbox{if $\lambda_{i^*} \leqslant \frac{n^{\zeta-1}}{\gamma_0}$}\,.
				\end{split}
			\end{array} \right. \\
		&	\lesssim \frac{ \| f^*\|^2}{\gamma_0} n^{\zeta-1}\,.
		\end{split}
	\end{equation}
	%where we use $\mathbb{E}_{\bm W} [\lambda_m] \sim \mathcal{O}(1/m)$ in Lemma~\ref{lemsubexp}.
	%where we use $\lambda_1 \sim \mathcal{O}(1)$ and $\lambda_m \sim \mathcal{O}(1/m)$ in Lemma~\ref{lemsubexp}.
\end{proof}

\subsection{Bound for ${\tt B1}$}
\label{app:b1}

Here we aim to bound $ {\tt B1} := \mathbb{E}_{\bm X, \bm W} \big[ \langle \bar{\eta}^{{\tt bias}}_n \!-\! \bar{\eta}^{{\tt bX}}_n, \Sigma_m ( \bar{\eta}^{{\tt bias}}_n \!-\! \bar{\eta}^{{\tt bX}}_n ) \rangle \big] $.
Define $\alpha_t^{\tt X} := {\eta}^{{\tt bias}}_t - {\eta}^{{\tt bX}}_{t}$, we have
\begin{equation}\label{axtada}
	\alpha_t^{\tt X} =[I - \gamma_t \varphi(\bm x_t) \otimes \varphi(\bm x_t)] \alpha_{t-1}^{\tt X} + \gamma_t [{\Sigma}_m - \varphi(\bm x_t) \otimes \varphi(\bm x_t) ] {\eta}^{{\tt bX}}_{t-1}\,,
\end{equation}
with $\alpha_0^{\tt X} = 0$ and ${\eta}^{{\tt bX}}_{t-1} = \prod_{j=1}^{t-1}(I - \gamma_j \Sigma_m) f^*$.
Accordingly, we have
\begin{equation*}
	{\tt B1} := \mathbb{E}_{\bm X, \bm W} \big[ \langle \bar{\eta}^{{\tt bias}}_n \!-\! \bar{\eta}^{{\tt bX}}_n, \Sigma_m ( \bar{\eta}^{{\tt bias}}_n \!-\! \bar{\eta}^{{\tt bX}}_n ) \rangle \big] = \mathbb{E}_{\bm W} \left( \mathbb{E}_{\bm X} [\langle \bar{\alpha}_n^{\tt X}, \Sigma_m \bar{\alpha}_n^{\tt X} \rangle ] \right)\,.
\end{equation*}

\begin{proposition}\label{propb1}
	Under Assumption~\ref{assumdata},~\ref{assexist},~\ref{assumact},~\ref{assump:bound_fourthmoment} with $r' \geqslant 1$, if the step-size $\gamma_t := \gamma_0 t^{-\zeta}$ with $\zeta \in [0,1)$ satisfies
	\begin{equation*}
		\gamma_0 < \min \left\{ \frac{1}{r' \mathrm{Tr}(\Sigma_m)}, \frac{1}{c' \mathrm{Tr}(\Sigma_m)} \right\} \,,
	\end{equation*}
	where the constant $c'$ is defined in Eq.~\eqref{eqconstant}.
	Then ${\tt B1}$  can be bounded by
	\begin{equation*}
		\begin{split}
			{\tt B1} & \lesssim \frac{\gamma_0 r' n^{\zeta-1}}{\sqrt{\mathbb{E}[1- \gamma_0 r' \mathrm{Tr}(\Sigma_m)]^4}} \| f^* \|^2  \sim \mathcal{O}\left( n^{\zeta -1} \right)\,.
		\end{split}
	\end{equation*}
\end{proposition}

To prove Proposition~\ref{propb1}, we need a lemma on stochastic recursions based on $\mathbb{E}[\alpha_t^{\tt X}| \alpha_{t-1}^{\tt X}] = (I - \gamma_t \Sigma_m) \alpha_{t-1}^{\tt X}$, that shares the similar proof fashion with \cite[Lemma 1]{bach2013non} and \cite[Lemma 11]{dieuleveut2016nonparametric}. 

\begin{lemma}\label{lemstorecada}
	Under Assumption~\ref{assumdata},~\ref{assexist},~\ref{assumact},~\ref{assump:bound_fourthmoment} with $r' \geqslant 1$, denoting $H_{t-1} := [{\Sigma}_m - \varphi(\bm x_t) \otimes \varphi(\bm x_t) ] {\eta}^{{\tt bX}}_{t-1}$, if the step-size $\gamma_t := \gamma_0 t^{-\zeta}$ with $\zeta \in [0,1)$ satisfies
	\begin{equation*}
		\gamma_0 <  \frac{1}{r' \mathrm{Tr}(\Sigma_m)} \,,
	\end{equation*}
	we have
	\begin{equation*}
		\mathbb{E}_{\bm X} [\langle \bar{\alpha}_n^{\tt X}, \Sigma_m \bar{\alpha}_n^{\tt X} \rangle ] \leqslant \frac{1}{2n[1-\gamma_0 r' \mathrm{Tr}(\Sigma_m)]} \left( \sum_{k=1}^{n-1} \mathbb{E} \| {\alpha}_k^{\tt X} \|^2 (\frac{1}{\gamma_{k+1}} - \frac{1}{\gamma_k}) + 2\sum_{t=0}^{n-1} \gamma_{t+1} \mathbb{E}_{\bm X} \| H_t \|^2 \right) \,.
	\end{equation*}
\end{lemma}
{\bf Remark:} We require $\| \Sigma_m \|_2 \neq \frac{1}{r' \gamma_0}$ to avoid the denominator to be zero, which naturally holds as the probability measure of the continuous random variable $\| \Sigma_m \|_2$ at a point is zero.

\begin{proof}
	According to the definition of ${\alpha}^{{\tt X}}_t$ in Eq.~\eqref{axtada}, define $H_{t-1} := [{\Sigma}_m - \varphi(\bm x_t) \otimes \varphi(\bm x_t) ] {\eta}^{{\tt bX}}_{t-1}$, we have
	\begin{equation*}
		\begin{split}
			\| {\alpha}^{{\tt X}}_t \|^2 & = \| {\alpha}^{{\tt X}}_{t-1} - \gamma_t ( [\varphi(\bm x_t) \otimes \varphi(\bm x_t)] {\alpha}^{{\tt W}}_{t-1} - H_{t-1}) \|^2 \\
			& =	\| {\alpha}^{{\tt X}}_{t-1} \|^2 + \gamma_t^2 \| H_{t-1} - [\varphi(\bm x_t) \otimes \varphi(\bm x_t)] {\alpha}^{{\tt X}}_{t-1} \|^2 + 2 \gamma_t \langle {\alpha}^{{\tt W}}_{t-1}, H_{t-1} - [\varphi(\bm x_t) \otimes \varphi(\bm x_t)] {\alpha}^{{\tt X}}_{t-1} \rangle \\
			& \leqslant \| {\alpha}^{{\tt X}}_{t-1} \|^2 + 2\gamma_t^2 \left( \| H_{t-1} \|^2 + \| [\varphi(\bm x_t) \otimes \varphi(\bm x_t)] {\alpha}^{{\tt X}}_{t-1} \|^2 \right) + 2 \gamma_t \langle {\alpha}^{{\tt X}}_{t-1}, H_{t-1} - [\varphi(\bm x_t) \otimes \varphi(\bm x_t)] {\alpha}^{{\tt X}}_{t-1} \rangle\,,
		\end{split}  
	\end{equation*}
	which implies (by taking the conditional expectation)
	\begin{equation}\label{condexaxtada}
		\begin{split}
			\mathbb{E}_{\bm X} [ \| {\alpha}^{{\tt W}}_t \|^2 | {\alpha}^{{\tt W}}_{t-1}] & \leqslant \| {\alpha}^{{\tt X}}_{t-1} \|^2 + 2\gamma_t^2 \| H_{t-1} \|^2 + 2\gamma_t^2 \langle {\alpha}^{{\tt X}}_{t-1}, \mathbb{E}_{\bm X} [\varphi(\bm x_t) \otimes \varphi(\bm x_t) \otimes \varphi(\bm x_t) \otimes \varphi(\bm x_t)] {\alpha}^{{\tt X}}_{t-1} \rangle\\
			& \quad - 2 \gamma_t \langle {\alpha}^{{\tt X}}_{t-1}, {\Sigma}_m {\alpha}^{{\tt X}}_{t-1} \rangle \\
			& \leqslant \| {\alpha}^{{\tt X}}_{t-1} \|^2 + 2\gamma_t^2 \| H_{t-1} \|^2 + 2\gamma_t^2 r' \mathrm{Tr}({\Sigma}_m) \langle {\alpha}^{{\tt X}}_{t-1}, {\Sigma}_m {\alpha}^{{\tt X}}_{t-1} \rangle - 2 \gamma_t \langle {\alpha}^{{\tt X}}_{t-1}, {\Sigma}_m {\alpha}^{{\tt X}}_{t-1} \rangle \\
			& = \| {\alpha}^{{\tt X}}_{t-1} \|^2 + 2\gamma_t^2 \| H_{t-1} \|^2 - 2 \gamma_t [1- \gamma_t r' \mathrm{Tr}({\Sigma}_m) ]\langle {\alpha}^{{\tt X}}_{t-1}, {\Sigma}_m {\alpha}^{{\tt X}}_{t-1} \rangle \,.
		\end{split}
	\end{equation}
	
	where the first inequality holds by $\mathbb{E}_{\bm X} [H_{t-1}] = 0$, and the second inequality satisfies by Assumption~\ref{assump:bound_fourthmoment}.
	
	By taking the expectation of Eq.~\eqref{condexaxtada}, we have
	\begin{equation*}
		\mathbb{E}_{\bm X} [\| {\alpha}^{{\tt X}}_t \|^2] \leqslant 	\mathbb{E}_{\bm X} [\| {\alpha}^{{\tt X}}_{t-1} \|^2] + 2\gamma_t^2 \mathbb{E}_{\bm X}  [\| H_{t-1} \|^2] - 2 \gamma_t [1- \gamma_t r' \mathrm{Tr}({\Sigma}_m) ] \mathbb{E}_{\bm X}\langle {\alpha}^{{\tt X}}_{t-1}, {\Sigma}_m {\alpha}^{{\tt X}}_{t-1} \rangle \,,
	\end{equation*}
	which indicates that
	\begin{equation*}
		\begin{split}
			\mathbb{E}_{\bm X} \big[ \langle \bar{\alpha}^{{\tt X}}_{n}, {\Sigma}_m \bar{\alpha}^{{\tt X}}_{n}  \rangle \big] \rangle  & \leqslant \frac{1}{n} \sum_{t=0}^{n-1} \mathbb{E}_{\bm X} \langle {\alpha}^{{\tt W}}_{t}, {\Sigma}_m {\alpha}^{{\tt W}}_{t} \rangle  \leqslant \frac{1}{2n[1-\gamma_0 r' \mathrm{Tr}(\Sigma_m)]} \Bigg( \sum_{k=1}^{n-1} \mathbb{E}_{\bm X} \| {\alpha}_k^{\tt X} \|^2 (\frac{1}{\gamma_{k+1}} - \frac{1}{\gamma_k}) \\ 
			& \quad + \frac{1}{2\gamma_1} \mathbb{E}_{\bm X} \| {\alpha}_0^{\tt X} \|^2 - \frac{1}{2\gamma_t} \mathbb{E}_{\bm X} \| {\alpha}_t^{\tt X} \|^2 + \sum_{t=0}^{n-1} \gamma_{t+1} \mathbb{E}_{\bm X} \| H_t \|^2 \Bigg)  \\
			& \leqslant \frac{1}{2n[1-\gamma_0 r' \mathrm{Tr}(\Sigma_m)]} \left( \sum_{k=1}^{n-1} \mathbb{E}_{\bm X} \| {\alpha}_k^{\tt X} \|^2 (\frac{1}{\gamma_{k+1}} - \frac{1}{\gamma_k}) + 2\sum_{t=0}^{n-1} \gamma_{t+1} \mathbb{E}_{\bm X} \| H_t \|^2 \right) \,,
		\end{split}
	\end{equation*}
	due to ${\alpha}^{{\tt W}}_0 = 0$.
\end{proof}

In the next, we present the error bounds for two respective terms in Lemma~\ref{lemstorecada}.

\begin{lemma}\label{lemaxtb2ada}
	Based on the definition of $\alpha^{{\tt X}}_t$ in Eq.~\eqref{axtadarec}, 
	under Assumption~\ref{assumdata},~\ref{assexist},~\ref{assumact},~\ref{assump:bound_fourthmoment} with $r' \geqslant 1$, if the step-size $\gamma_t := \gamma_0 t^{-\zeta}$ with $\zeta \in [0,1)$ satisfies
	\begin{equation*}
		\gamma_0 < \min \left\{ \frac{1}{r'\mathrm{Tr}(\Sigma_m)}, \frac{1}{c'\mathrm{Tr}(\Sigma_m)} \right\} \,,
	\end{equation*}
	where the constant $c'$ is defined in Eq.~\eqref{eqconstant}.
	Then, we have
	\begin{equation*}
		\sum_{k=1}^{n-1} \mathbb{E} \| {\alpha}_k^{\tt X} \|^2 (\frac{1}{\gamma_{k+1}} - \frac{1}{\gamma_k}) \lesssim \frac{\gamma_0 r' \mathrm{Tr}(\Sigma_m)}{1-\gamma_0 r' \mathrm{Tr}(\Sigma_m)} (n^{\zeta} - 1) \| f^* \|^2 \,.
	\end{equation*}
\end{lemma}
\begin{proof}
	Based on the definition of $\alpha^{{\tt X}}_t$ in Eq.~\eqref{axtada}, it can be reformulated as 
	\begin{equation}\label{axtadarec}
		\begin{split}
			\alpha^{{\tt X}}_t & = [I - \gamma_t \varphi(\bm x_t) \otimes \varphi(\bm x_t)]  	\alpha^{{\tt X}}_{t-1} + \gamma_t [\Sigma_m - \varphi(\bm x_t) \otimes \varphi(\bm x_t)] \prod_{j=1}^{k-1} (I- \gamma_j {\Sigma}_m) f^* \\
			&= \sum_{s=1}^t \gamma_s \prod_{i=s+1}^{t} [I - \gamma_i \varphi(\bm x_i) \otimes \varphi(\bm x_i)] [{\Sigma}_m - \varphi(\bm x_s) \otimes \varphi(\bm x_s)] \prod_{j=1}^{s-1} (I- \gamma_j {\Sigma}_m) f^* \,.
		\end{split}
	\end{equation}
	and accordingly
	\begin{equation}\label{axtboundada}
		\begin{split}
			C_t^{\tt b-X} &:= \mathbb{E}_{\bm X} [ {\alpha}_t^{\tt X} \otimes  {\alpha}_t^{\tt X} ] = (I - \gamma_t T^{\tt W}) \circ C_{t-1}^{\tt b-X} + \gamma_t^2 (S^{\tt W} - \widetilde{S}^{\tt W}) \circ [\eta_{t-1}^{\tt bX} \otimes \eta_{t-1}^{\tt bX}] \\
			& \preccurlyeq (I - \gamma_t T^{\tt W}) \circ C_{t-1}^{\tt b-X} + \gamma_t^2 S^{\tt W} \circ [\eta_{t-1}^{\tt bX} \otimes \eta_{t-1}^{\tt bX}]  \\
			& \preccurlyeq (I - \gamma_t T^{\tt W}) \circ C_{t-1}^{\tt b-X} + \gamma_t^2 r' \mathrm{Tr}\left[\prod_{s=1}^{t-1} (I - \gamma_s \Sigma_m)^2 \Sigma_m\right] \Sigma_m (f^* \otimes f^*) \quad \mbox{[using Assumption~\ref{assump:bound_fourthmoment}]} \\
			& \preccurlyeq (I - \gamma_t T^{\tt W}) \circ C_{t-1}^{\tt b-X} + \gamma_t^2 r' \mathrm{Tr}(\Sigma_m) \Sigma_m (f^* \otimes f^*) \quad \mbox{[using $\exp(-2\lambda_i \gamma_0 \frac{t^{1-\zeta}-1}{1-\zeta}) \leqslant 1$]} \\
			& = r' \mathrm{Tr}(\Sigma_m) \sum_{s=1}^t \prod_{i=s+1}^t \left(I - \gamma_i T^{\tt W} \right) \circ \gamma_s^2 \Sigma_m (f^* \otimes f^*) \\
			& \preccurlyeq  \frac{\gamma_0 r' \mathrm{Tr}(\Sigma_m)}{1-\gamma_0 r' \mathrm{Tr}(\Sigma_m)} (f^* \otimes f^*)\,. \quad \mbox{[using Lemma~\ref{dinfvx}]}
		\end{split}
	\end{equation}
	Accordingly, we have %$\mathbb{E}_{\bm X} \| {\alpha}_t^{\tt X} \|^2 = \mathrm{Tr} [C_t^{\tt b-X}] = \| C_t^{\tt b-X} \|_2$
	\begin{equation*}
		\begin{split}
			\sum_{t=1}^{n-1} \mathbb{E}_{\bm X} \| {\alpha}_t^{\tt X} \|^2 (\frac{1}{\gamma_{t+1}} - \frac{1}{\gamma_t}) & = \sum_{t=1}^{n-1} \| C_t^{\tt b-X} \|_2 \left(\frac{1}{\gamma_{t+1}} - \frac{1}{\gamma_t} \right) \quad \mbox{[using Eq.~\eqref{axtboundada}]} \\
			& \leqslant \sum_{t=1}^{n-1} \frac{\gamma_0 r' \mathrm{Tr}(\Sigma_m)}{1-\gamma_0 r' \mathrm{Tr}(\Sigma_m)} [(t+1)^{\zeta} - t^{\zeta}] \| f^*\|^2 \\
			& \lesssim \frac{\gamma_0 r' \mathrm{Tr}(\Sigma_m)}{1-\gamma_0 r' \mathrm{Tr}(\Sigma_m)} (n^{\zeta} - 1) \| f^* \|^2\,,
		\end{split}
	\end{equation*}
	which concludes the proof.
\end{proof}

\begin{lemma}\label{lemahtb2ada}
	Denote $H_{t-1} := [{\Sigma}_m - \varphi(\bm x_t) \otimes \varphi(\bm x_t) ] {\eta}^{{\tt bX}}_{t-1}$, Assumption~\ref{assumdata},~\ref{assexist},~\ref{assumact},~\ref{assump:bound_fourthmoment} with $r' \geqslant 1$, 
	if the step-size $\gamma_t := \gamma_0 t^{-\zeta}$ with $\zeta \in [0,1)$ satisfies
	\begin{equation*}
		\gamma_0 \leqslant \frac{1}{\mathrm{Tr}(\Sigma_m)} \,,
	\end{equation*} we have
	\begin{equation*}
		\sum_{t=0}^{n-1} \gamma_{t+1} \mathbb{E}_{\bm X} \| H_t \|^2 \leqslant \frac{1}{2} \| f^* \|^2 r' \mathrm{Tr}(\Sigma_m)\,.
	\end{equation*}
\end{lemma}
\begin{proof}
	\begin{equation*}
		\begin{split}
			\sum_{t=0}^{n-1} \gamma_{t+1} \mathbb{E}_{\bm X} \| H_t \|^2 & = \sum_{t=0}^{n-1} \gamma_{t+1} \left\langle f^*, \prod_{j=1}^{t-1}(I - \gamma_j \Sigma_m) \mathbb{E}_{\bm X}[{\Sigma}_m - \varphi(\bm x_t) \otimes \varphi(\bm x_t) ]^2 \prod_{j=1}^{t-1}(I - \gamma_j \Sigma_m) f^*  \right\rangle \\
			& \leqslant \sum_{t=0}^{n-1} \gamma_{t+1} \left\langle f^*, r' \mathrm{Tr}(\Sigma_m) \Big[\prod_{j=1}^{t-1}(I - \gamma_j \Sigma_m) \Big]^2 \Sigma_m  f^* \right\rangle \quad \mbox{[using Assumption~\ref{assump:bound_fourthmoment}]} \\
			& \leqslant \| f^* \|^2 r' \mathrm{Tr}(\Sigma_m) \left\| \sum_{t=0}^{n-1} \gamma_{t+1} \Big[\prod_{j=1}^{t-1}(I - \gamma_j \Sigma_m) \Big]^2 \Sigma_m \right\|_2 \\
			& = \| f^* \|^2 r' \mathrm{Tr}(\Sigma_m) \max_{i \in \{1,2,\dots,m \}} \sum_{t=0}^{n-1} \gamma_{t+1} \prod_{j=1}^{t-1} (1-\gamma_j \lambda_i)^2 \lambda_i \\
			& \leqslant \| f^* \|^2 r' \mathrm{Tr}(\Sigma_m) \max_{i \in \{1,2,\dots,m \}} \gamma_0 \lambda_i \int_{0}^n u^{-\zeta} \exp \left(-2 \gamma_0 {\lambda}_i \frac{u^{1-\zeta} - 1}{1-\zeta} \right) \mathrm{d} u \\
			& \leqslant \frac{1}{2} \| f^* \|^2 r' \mathrm{Tr}(\Sigma_m)\,, \quad \mbox{[using Eq.~\eqref{Iifirst}]}
		\end{split}
	\end{equation*}
	which concludes the proof.
\end{proof}

Based on the above results, we are ready to prove Proposition~\ref{propb1}.
\begin{proof}
	According to Lemma~\ref{lemaxtb2ada}, we have
	\begin{equation*}
		\begin{split}
			\mathbb{E}_{\bm W} \frac{\sum_{k=1}^{n-1} \mathbb{E} \| {\alpha}_k^{\tt X} \|^2 (\frac{1}{\gamma_{k+1}} - \frac{1}{\gamma_k})}{2n[1-\gamma_0 r' \mathrm{Tr}(\Sigma_m)]} & \lesssim \mathbb{E}_{\bm W} \frac{ \gamma_0 r' \mathrm{Tr}(\Sigma_m)}{2n[1-\gamma_0 r' \mathrm{Tr}(\Sigma_m)]^2}  (n^{\zeta} - 1) \| f^* \|^2 \\
			& \leqslant \gamma_0 r' n^{\zeta-1} \sqrt{\mathbb{E}[\mathrm{Tr}(\Sigma_m)]^2} \frac{1}{\sqrt{\mathbb{E}[1- \gamma_0 r' \mathrm{Tr}(\Sigma_m)]^4}} \| f^* \|^2 \\
			& \lesssim \frac{\gamma_0 r' n^{\zeta-1}}{\sqrt{\mathbb{E}[1- \gamma_0 r' \mathrm{Tr}(\Sigma_m)]^4}} \| f^* \|^2 \\
			& \sim \mathcal{O}(n^{\zeta -1}) \,,
		\end{split}
	\end{equation*}
	where the second inequality holds by Cauchy-Schwarz inequality and the last inequality holds by Lemma~\ref{lemsubexp} with
	\begin{equation}\label{stepbound}
	\frac{1}{\mathbb{E}[1 - \gamma_0r' \mathrm{tr}(\Sigma_m)]^4} \leqslant \frac{1}{[1 - \gamma_0r' \mathbb{E}\mathrm{tr}(\Sigma_m)]^4} = \frac{1}{[1 - \gamma_0r' \mathrm{tr}(\widetilde{\Sigma}_m)]^4}  \sim \mathcal{O}(1)\,.
	\end{equation}
	
	According to Lemma~\ref{lemahtb2ada}, we have
	\begin{equation*}
		\begin{split}
			\mathbb{E}_{\bm W} \frac{2\sum_{t=0}^{n-1} \gamma_{t+1} \mathbb{E}_{\bm X} \| H_t \|^2}{2n[1-\gamma_0 r' \mathrm{Tr}(\Sigma_m)]} & \leqslant \mathbb{E}_{\bm W} \frac{ r'\mathrm{Tr}(\Sigma_m)}{2n[1-\gamma_0 r' \mathrm{Tr}(\Sigma_m)]} \| f^*\|^2 \\
			& \lesssim \frac{r'}{n} \sqrt{\mathbb{E}[\mathrm{Tr}(\Sigma_m)]^2} \frac{1}{\sqrt{\mathbb{E}[1- \gamma_0 r' \mathrm{Tr}(\Sigma_m)]^2}} \| f^* \|^2 \\
			& \lesssim \frac{r'}{n\sqrt{\mathbb{E}[1- \gamma_0 r' \mathrm{Tr}(\Sigma_m)]^2}} \| f^* \|^2 \quad \mbox{[using Lemma~\ref{lemsubexp}]} \\
			& \sim \mathcal{O} \left( \frac{1}{n} \right) \,.
		\end{split}
	\end{equation*}
	Accordingly, combining the above two equations, we have
	\begin{equation*}
		\begin{split}
			{\tt B1}:= \mathbb{E}_{\bm W} \mathbb{E}_{\bm X} [\langle \bar{\alpha}_n^{\tt X}, \Sigma_m \bar{\alpha}_n^{\tt X} \rangle ] & \leqslant \frac{1}{2n[1-\gamma_0 r' \mathrm{Tr}(\Sigma_m)]} \mathbb{E}_{\bm W}  \left( \sum_{k=1}^{n-1} \mathbb{E} \| {\alpha}_k^{\tt X} \|^2 (\frac{1}{\gamma_{k+1}} - \frac{1}{\gamma_k}) + 2\sum_{t=0}^{n-1} \gamma_{t+1} \mathbb{E}_{\bm X} \| H_t \|^2 \right) \\
			& \lesssim \frac{\gamma_0 r' n^{\zeta-1}}{\sqrt{\mathbb{E}[1- \gamma_0 r' \mathrm{Tr}(\Sigma_m)]^4}} \| f^* \|^2 \,,
		\end{split}
	\end{equation*}
	which concludes the proof.
\end{proof}

\subsection{Proof of Theorem~\ref{promainba}}
\label{sec:proofbiasfinal}
\begin{proof}
	Combining the above results for three terms ${\tt B1}$, ${\tt B2}$, ${\tt B3}$,
	if 	\begin{equation}\label{constepbias}
		\gamma_0 < \min \left\{ \frac{1}{\mathrm{Tr}(\widetilde{\Sigma}_m)},  \frac{1}{r' \mathrm{Tr}(\Sigma_m)}, \frac{1}{c' \mathrm{Tr}(\Sigma_m)} \right\} \sim \mathcal{O}(1) \,,
	\end{equation}
	where the constant $c$ is defined in Eq.~\eqref{eqconstant}.
	Then the ${\tt Bias}$ can be upper bounded by
	\begin{equation*}
		\begin{split}
			{\tt Bias}	& \leqslant \left( \sqrt{\tt B1} + \sqrt{\tt B2} + \sqrt{\tt B3} \right)^2 \leqslant 3 ({\tt B1} + {\tt B2} + {\tt B3})  \\
			& \lesssim \frac{\gamma_0 r' n^{\zeta-1}}{\sqrt{\mathbb{E}[1- \gamma_0 r' \mathrm{Tr}(\Sigma_m)]^4}} \| f^* \|^2 \\
			& \lesssim \gamma_0 r' n^{\zeta-1} \| f^* \|^2\,,
		\end{split}
	\end{equation*}
	where the last inequality holds by \cref{stepbound}.
	
	%Note that, the condition in Eq.~\eqref{constepbias} ensures $I - \gamma_t \Sigma_m$ ($t=1,2,\dots,m$) to be a contraction map.
	
	%Though $\mathrm{Tr}(\Sigma_m)$ is a random variable, it is with a sub-exponential $\mathcal{O}(1)$ norm, that means, the condition $\gamma_0 < 1/\mathrm{Tr}(\Sigma_m)$ can be equivalently substituted by $\gamma_0 < 1/[c\mathrm{Tr}(\widetilde{\Sigma}_m)]$ for some large $c$ (independent of $n$, $m$, $d$) with exponentially high probability.
	%For example, the probability with $\exp(-10) < 10^{-4}$ and $\exp(-100) < 10^{-43}$ by taking $c=10$, or $100$ is enough small in practice.
	
	%As aforementioned, we can reformulate the condition in Eq.~\eqref{constepbias} as $\gamma_0 \lesssim \frac{1}{r'\mathrm{Tr}(\widetilde{\Sigma}_m)}$ with exponentially high probability.
\end{proof}

%This is because, $\mathrm{Tr}(\Sigma_m)$ is a sub-exponential random variable with $\mathcal{O}(1)$ norm in Lemma~\ref{lemsubexp}, which makes the constant $c$ unnecessary to be quite large. For example, the probability with $\exp(-10) < 10^{-4}$ and $\exp(-100) < 10^{-43}$ by taking $c=10$, or $100$ is enough small in practice.

\section{Proof for Variance}
\label{app:variance}
In this section, we present the error bound for ${\tt Variance}$.
Recall the definition of $\eta_t^{{\tt vX}}$ in Eq.~\eqref{eq:var_xada} and $\eta_t^{{\tt vXW}}$ in Eq.~\eqref{eq:var_xwada}, and 
\begin{equation*}
	\bar{\eta}_n^{{\tt vX}} := \frac{1}{n} \sum_{t=0}^{n-1} \bar{ \eta}_t^{{\tt vX}},\qquad \bar{\eta}_n^{{\tt vXW}} := \frac{1}{n} \sum_{t=0}^{n-1} \bar{\eta}_t^{{\tt vXW}}\,,
\end{equation*}
by virtue of Minkowski inequality, ${\tt Variance}$ can be further decomposed as
\begin{equation}\label{vardecom}
	\begin{split}
		& \Big( \mathbb{E}_{\bm X, \bm W, \bm \varepsilon} \big[ \langle \bar{\eta}^{{\tt var}}_n, \Sigma_m \bar{\eta}^{{\tt var}}_n \rangle \big] \Big)^{\frac{1}{2}}   \!\!\leqslant\!\! \Big( \underbrace{ \mathbb{E}_{\bm X, \bm W,\bm \varepsilon} \big[ \langle \bar{\eta}^{{\tt var}}_n - \bar{\eta}^{{\tt vX}}_n, \Sigma_m ( \bar{\eta}^{{\tt var}}_n - \bar{\eta}^{{\tt vX}}_n ) \rangle \big] }_{ \triangleq \tt V1} \Big)^{\frac{1}{2}} + \Big( \mathbb{E}_{\bm X, \bm W, \bm \varepsilon} \big[ \langle \bar{\eta}^{{\tt vX}}_n, \Sigma_m \bar{\eta}^{{\tt vX}}_n \rangle \big] \Big)^{\frac{1}{2}} \\
		& \qquad \!\!\leqslant\!\! ({\tt V1})^{\frac{1}{2}} \!+\! \Big( \underbrace{ \mathbb{E}_{\bm X, \bm W, \bm \varepsilon} \big[ \langle \bar{\eta}^{{\tt vX}}_n \!-\! \bar{\eta}^{{\tt vXW}}_n, \Sigma_m ( \bar{\eta}^{{\tt vX}}_n \!-\! \bar{\eta}^{{\tt vXW}}_n ) \rangle \big] }_{ \triangleq \tt V2} \Big)^{\frac{1}{2}} \!\!+\! [ \underbrace{ \mathbb{E}_{\bm X, \bm W, \bm \varepsilon} \langle   \bar{\eta}^{{\tt vXW}}_n, {\Sigma}_m \bar{\eta}^{{\tt vXW}}_n \rangle}_{\triangleq {\tt V3}} ]^{\frac{1}{2}} \,.
	\end{split}
\end{equation}
Accordingly, the ${\tt Variance}$ can be decomposed as ${\tt Variance} \lesssim {\tt V1} + {\tt V2} + {\tt V3}$, and in the next we give the error bounds for them, respectively.

\subsection{Bound for  ${\tt V3}$}
\label{app:v3}

In this section, we aim to bound $ {\tt V3} := \mathbb{E}_{\bm X, \bm W, \bm \varepsilon} \langle   \bar{\eta}^{{\tt vXW}}_n, {\Sigma}_m \bar{\eta}^{{\tt vXW}}_n \rangle$.
Note that $\mathbb{E}_{\bm X, \bm \varepsilon} [{\eta}^{{\tt vXW}}_t | {\eta}^{{\tt vXW}}_{t-1}] = (I - \gamma_t \widetilde{\Sigma}_m) {\eta}^{{\tt vXW}}_{t-1}$, similar to Appendix~\ref{app:b2} for ${\tt B2}$, we have the following expression for ${\tt V3}$
\begin{equation}\label{v3expression}
	\begin{split}
		{\tt V3} &:=	\mathbb{E}_{\bm X, \bm W, \bm \varepsilon} \langle   \bar{\eta}^{{\tt vXW}}_n, {\Sigma}_m \bar{\eta}^{{\tt vXW}}_n \rangle = \mathbb{E}_{\bm W} [\mathbb{E}_{\bm X, \bm \varepsilon} \langle  \Sigma_m, \bar{\eta}^{{\tt vXW}}_n \otimes \bar{\eta}^{{\tt vXW}}_n \rangle] \\
		& = \frac{1}{n^2} \mathbb{E}_{\bm W} \left( \left\langle \Sigma_m, \sum_{0 \leqslant k \leqslant t \leqslant n-1}  \mathbb{E}_{\bm X, \bm \varepsilon} [{\eta}^{{\tt vXW}}_t \otimes {\eta}^{{\tt vXW}}_k] + \sum_{0 \leqslant k < t \leqslant n-1}  \mathbb{E}_{\bm X, \bm \varepsilon} [{\eta}^{{\tt vXW}}_t \otimes {\eta}^{{\tt vXW}}_k] \right\rangle \right) \\
		& \leqslant \frac{1}{n^2} \mathbb{E}_{\bm W} \left( \left\langle \Sigma_m, \sum_{0 \leqslant k \leqslant t \leqslant n-1}  \mathbb{E}_{\bm X, \bm \varepsilon} [{\eta}^{{\tt vXW}}_t \otimes {\eta}^{{\tt vXW}}_k] + \sum_{0 \leqslant k \leqslant t \leqslant n-1}  \mathbb{E}_{\bm X, \bm \varepsilon} [{\eta}^{{\tt vXW}}_t \otimes {\eta}^{{\tt vXW}}_k] \right\rangle \right) \\
		& =  \frac{2}{n^2} \sum_{t=0}^{n-1} \sum_{k=t}^{n-1} \mathbb{E}_{\bm W} \left\langle  \prod_{j=t}^{k-1}(I-\gamma_j \widetilde{\Sigma}_m)  {\Sigma}_m, \underbrace{\mathbb{E}_{\bm X, \bm \varepsilon} [{\eta}^{{\tt vXW}}_t \otimes {\eta}^{{\tt vXW}}_t]}_{:= C^{{\tt vXW}}_t} \right\rangle \,,
	\end{split}
\end{equation}
and thus we have the following error bound for ${\tt V3}$.

\begin{proposition}\label{propv3}
	Under Assumption~\ref{assumdata},~\ref{assumact},~\ref{assump:noise} with $\tau > 0$, if the step-size $\gamma_t := \gamma_0 t^{-\zeta}$ with $\zeta \in [0,1)$ satisfies
	$\gamma_0 \leqslant \frac{1}{\mathrm{Tr}(\widetilde{\Sigma}_m)}$, 
	then ${\tt V3}$ can be bounded by
	\begin{equation*}
		{\tt V3} \lesssim  \left\{
		\begin{array}{rcl}
			\begin{split}
				& \gamma_0\tau^2 \frac{m}{n^{1-\zeta}} ,~\quad \mbox{if $m \leqslant n$} \\
				& \gamma_0\tau^2 \left( n^{\zeta-1} + \frac{n}{m} \right) ,~\quad \mbox{if $m > n$}\,.
			\end{split}
		\end{array} \right.
	\end{equation*}
\end{proposition}

To prove Proposition~\ref{propv3}, we need the following lemma.
\begin{lemma}\label{lemcvtwtada}
	Denote $C^{{\tt vXW}}_t := \mathbb{E}_{\bm X, \bm \varepsilon} [{\eta}^{{\tt vXW}}_t \otimes {\eta}^{{\tt vXW}}_t]$, under Assumptions~\ref{assumdata},~\ref{assumact},~\ref{assump:noise} with $\tau > 0$, if $\gamma_0 \leqslant 1/ \mathrm{Tr}(\widetilde{\Sigma}_m)$, we have
	\begin{equation*}\label{Cvxwtada}
		C^{{\tt vXW}}_t \preccurlyeq \tau^2 \sum_{k=1}^t \gamma_k^2 \prod_{j=k+1}^t(I - \gamma_j \widetilde{\Sigma}_m)^2 \Sigma_m \,.
	\end{equation*}
\end{lemma}
\begin{proof}
	Recall the definition of ${\eta}^{{\tt vXW}}_t$ in Eq.~\eqref{eq:var_xwada}, it can be further represented as
	\begin{equation*}
		\eta_t^{{\tt vXW}} = (I - \gamma_t \widetilde{\Sigma}_m ) \eta_{t-1}^{{\tt vXW}} + \gamma_t \varepsilon_k \varphi(\bm x_k) = \sum_{k=1}^t \prod_{j=k+1}^t (I - \gamma_j \widetilde{\Sigma}_m) \gamma_k \varepsilon_k \varphi(\bm x_k)  \quad \mbox{with $\eta_0^{{\tt vXW}} = 0$}\,.
	\end{equation*}
	%Denote $C^{{\tt vXW}}_t := \mathbb{E}_{\bm X, \bm \varepsilon} [{\eta}^{{\tt vXW}}_t \otimes {\eta}^{{\tt vXW}}_t]$, which admits
	%\begin{equation*}
		%	C^{{\tt vXW}}_t = \sum_{k=1}^t \prod_{j=k+1}^t (I - \gamma_j \widetilde{\Sigma}_m)^2 \gamma^2_k \Xi \quad \mbox{with}~ C_0 = 0\,,
		%	\end{equation*}
	%which holds by  $\mathbb{E}[\varepsilon_i \varepsilon_j] =0 $ for $i\neq j$.
	%\begin{equation*}
		%	C^{{\tt vXW}}_t = (I - \gamma_t \widetilde{T}) \circ C^{{\tt vXW}}_{t-1} + \gamma_t^2 %\Xi,\quad \mbox{with}~ C_0 = 0\,.
		%\end{equation*}
	Accordingly, $C^{{\tt vXW}}_t$ admits (with $C^{{\tt vXW}}_0 = 0$)
	\begin{equation*}
		\begin{split}
			C^{{\tt vXW}}_t & = \sum_{k=1}^t \prod_{j=k+1}^t (I - \gamma_j \widetilde{\Sigma}_m)^2 \gamma^2_k \Xi  \preccurlyeq \tau^2 \sum_{k=1}^{t} \gamma_k^2 \prod_{j=k+1}^{t} (I - \gamma_j \widetilde{\Sigma}_m)^2 \Sigma_m \quad \mbox{[using Assumption~\ref{assump:noise}]} \\
		\end{split}
	\end{equation*}
	where we use $\mathbb{E}[\varepsilon_i \varepsilon_j] =0 $ for $i\neq j$.
\end{proof}

In the next, we are ready to bound ${\tt V3}$ in Proposition~\ref{propv3}.
\begin{proof}[Proof of Proposition~\ref{propv3}]
	Note that $\widetilde{\lambda}_1 \sim \mathcal{O}(1)$ and $\widetilde{\lambda}_2 \sim \mathcal{O}(1/m)$ in Lemma~\ref{lemsubexp}, we take the upper bound of the integral in Eq.~\eqref{intut} to $\frac{n^\zeta}{\widetilde{\lambda}_1 \gamma_0}$ for $\widetilde{\lambda}_1$.
	However, according to the order of $\widetilde{\lambda}_2$, if $\widetilde{\lambda}_2 \lesssim 1/n$, the exact upper bound is tight.
	Based on this, we first consider that $m \leqslant n$ case such that $\widetilde{\lambda}_2 \gtrsim 1/n$, and then focus on the $m \geqslant n$ case.
	Taking $\frac{n^\zeta}{\widetilde{\lambda}_i \gamma_0}$ in Eq.~\eqref{intut} and $\frac{\gamma_0}{\widetilde{\lambda}_i}$ in Eq.~\eqref{intu2zetatu}, we have
	\begin{equation}\label{V3bound1ada}
		\begin{split}
			{\tt V3} &:=	\mathbb{E}_{\bm X, \bm W, \bm \varepsilon} \langle   \bar{\eta}^{{\tt vXW}}_n, {\Sigma}_m \bar{\eta}^{{\tt vXW}}_n \rangle = \mathbb{E}_{\bm X, \bm W, \bm \varepsilon} \langle  \Sigma_m, \bar{\eta}^{{\tt vXW}}_n \otimes \bar{\eta}^{{\tt vXW}}_n \rangle \\
			& \leqslant  \frac{2}{n^2} \sum_{t=0}^{n-1} \sum_{k=t}^{n-1} \mathbb{E}_{\bm W} \left\langle  \prod_{j=t}^{k-1}(I-\gamma_j \widetilde{\Sigma}_m)  {\Sigma}_m, \underbrace{\mathbb{E}_{\bm X, \bm \varepsilon} [\bar{\eta}^{{\tt vXW}}_t \otimes \bar{\eta}^{{\tt vXW}}_t]}_{:= C^{{\tt vXW}}_t} \right\rangle \quad \mbox{[using Eq.~\eqref{v3expression}]} \\
			& \leqslant \frac{2\tau^2}{n^2} \sum_{t=0}^{n-1} \sum_{k=t}^{n-1} \mathbb{E}_{\bm W} \left\langle  \prod_{j=t}^{k-1}(I-\gamma_j \widetilde{\Sigma}_m)   {\Sigma}_m, \sum_{s=1}^{t} \gamma_s^2 \prod_{j=s+1}^{t} (I - \gamma_j \widetilde{\Sigma}_m)^2 \Sigma_m \right\rangle \quad \mbox{[using Lemma~\ref{lemcvtwtada}]} \\
			& \leqslant \frac{2\tau^2}{n^2} \sum_{t=0}^{n-1} \sum_{k=t}^{n-1}  \left\| \prod_{j=t}^{k-1}(I-\gamma_j \widetilde{\Sigma}_m)   \widetilde{\Sigma}_m \sum_{s=1}^{t} \gamma_s^2 \prod_{j=s+1}^{t} (I - \gamma_j \widetilde{\Sigma}_m)^2  \right\|_2 \mathrm{Tr}\left( \mathbb{E}_{\bm W}[\Sigma_m^2 \widetilde{\Sigma}_m^{-1}] \right) \\
			& \lesssim  \frac{2\tau^2}{n^2} \sum_{t=0}^{n-1} \sum_{k=t}^{n-1} \max_{i \in \{ 1,2,\dots,m\}} \left\| \prod_{j=t}^{k-1}(1-\gamma_j \widetilde{\lambda}_i)   \widetilde{\lambda}_i \sum_{s=1}^{t} \gamma_s^2 \prod_{j=s+1}^{t} (1 - \gamma_j \widetilde{\lambda}_i)^2 \right\|_2 \quad \mbox{[using Lemma~\ref{trace1}]} \\
			& \leqslant \frac{2\tau^2}{n^2} \sum_{t=0}^{n-1} \sum_{k=t}^{n-1} \max_{i \in \{ 1,2,\dots,m\}} \left\| \widetilde{\lambda}_i \exp \left(- \widetilde{\lambda}_i \gamma_0 \frac{k^{1-\zeta}-t^{1-\zeta}}{1-\zeta} \right) \sum_{s=1}^{t} \gamma_s^2 \exp \left(-2 \widetilde{\lambda}_i \gamma_0 \frac{(t+1)^{1-\zeta}-(s+1)^{1-\zeta}}{1-\zeta} \right) \right\|_2 \\
			& \lesssim \frac{\tau^2}{n^2} \sum_{t=0}^{n-1} \max_{i \in \{ 1,2,\dots,m\}} \left[  \widetilde{\lambda}_i \frac{n^\zeta}{\widetilde{\lambda}_i \gamma_0} \left( \frac{\gamma_0}{\widetilde{\lambda}_i} + \gamma_t^2 \right) \right] \quad \mbox{[using Eqs.~\eqref{intut},~\eqref{intu2zetatu}]} \\
			& \leqslant \frac{\tau^2}{n^2} \left[ n^{1+\zeta}m + n^{\zeta} \mathrm{Tr}(\widetilde{\Sigma}_m) \gamma_0 \int_0^n t^{-2\zeta} \mathrm{d}t\right] \\
			& \lesssim \gamma_0\tau^2 \frac{m}{n^{1-\zeta}}\,, \quad \mbox{[using Lemma~\ref{lemsubexp}]}
		\end{split}
	\end{equation}
	where the last equality holds that $\int_0^n t^{-2\zeta} \mathrm{d}t \leqslant n$ for any $\zeta \in [0,1)$.
	%	\begin{equation*}\label{intlndis}
		%		\int_0^n t^{-2\zeta} \mathrm{d}t \leqslant \left\{
		%		\begin{array}{rcl}
			%			\begin{split}
				%				& n^{1-2\zeta},~\mbox{if $0\leqslant \zeta < \frac{1}{2}$} \\
				%				& \ln n,~\mbox{if $\zeta = \frac{1}{2}$} \\
				%				& C,~\mbox{$\frac{1}{2} < \zeta < 1$ for some constant $C$} 
				%			\end{split}
			%		\end{array} \right. \\
		%	\end{equation*}
	%The above holds for $n \geqslant m$, i.e., $\widetilde{\lambda}_2 \gtrsim 1/n$.
	
	If $\widetilde{\lambda}_2 \lesssim 1/n$, that means, $m > n$ in the over-parameterized regime, we have 
	\begin{equation*}
		\begin{split}
			{\tt V3} & \lesssim \frac{2\tau^2}{n^2} \sum_{t=0}^{n-1}  \left[  \widetilde{\lambda}_1 \frac{n^\zeta}{\widetilde{\lambda}_1 \gamma_0} \left( \frac{\gamma_0}{\widetilde{\lambda}_1} + \gamma_t^2 \right) +  \widetilde{\lambda}_2 (n-t) t \right] \\
			& \lesssim \frac{\gamma_0\tau^2}{n^2} \Big( n^{1+\zeta} +  \widetilde{\lambda}_2 \frac{n(n-1)(n+1)}{6}\Big) \quad \mbox{[since $\lambda_1 \sim \mathcal{O}(1)$]} \\
			& \lesssim \gamma_0\tau^2 \left( n^{\zeta-1} + \frac{n}{m} \right) \,,
		\end{split}
	\end{equation*}
	which concludes the proof.
	%where the derived order also holds for the case $1/2 \leqslant \zeta \leqslant 1$, similar to the discussion in Eq.~\eqref{intlndis}.
	
\end{proof}

\subsection{Bound for  ${\tt V2}$}
\label{app:v2}

Here we aim to bound ${\tt V2}$
\begin{equation*}
	{\tt V2} := \mathbb{E}_{\bm X, \bm W, \bm \varepsilon} \big[ \langle \bar{\eta}^{{\tt vX}}_n \!-\! \bar{\eta}^{{\tt vXW}}_n, \Sigma_m ( \bar{\eta}^{{\tt vX}}_n \!-\! \bar{\eta}^{{\tt vXW}}_n ) \rangle \big]\,.
\end{equation*}
Recall the definition of ${\eta}^{{\tt vX}}_t$ and ${\eta}^{{\tt vXW}}_t$ in Eqs.~\eqref{eq:var_xada} and \eqref{eq:var_xwada}, we have
\begin{equation*}
	\eta_t^{{\tt vXW}} = (I - \gamma_t \widetilde{\Sigma}_m ) \eta_{t-1}^{{\tt vXW}} + \gamma_t \varepsilon_k \varphi(\bm x_k) = \sum_{k=1}^t \prod_{j=k+1}^t (I - \gamma_j \widetilde{\Sigma}_m) \gamma_k \varepsilon_k \varphi(\bm x_k)  \quad \mbox{with $\eta_0^{{\tt vXW}} = 0$}\,,
\end{equation*}
and accordingly, we define
\begin{equation*}
	\begin{split}
		{\alpha}^{{\tt vX-W}}_t & := {\eta}^{{\tt vX}}_t - {\eta}^{{\tt vXW}}_t = (I - \gamma_t \Sigma_m) {\alpha}^{{\tt vX-W}}_{t-1} + \gamma_t (\widetilde{\Sigma}_m - \Sigma_m) {\eta}^{{\tt vXW}}_{t-1}\,, \quad \mbox{with $\alpha^{{\tt vX-W}}_0 = 0$} \\
		%& = (I - \gamma_t \Sigma_m) {\alpha}^{{\tt vX-W}}_{t-1} + \gamma_t (\widetilde{\Sigma}_m - \Sigma_m) \sum_{k=1}^{t-1} \prod_{j=k+1}^{t-1} (I - \gamma_j \widetilde{\Sigma}_m ) \gamma_k \varepsilon_k \varphi(\bm x_k)\\
		& = \sum_{s=1}^t \prod_{i=s+1}^{t} (I - \gamma_i \Sigma_m) \gamma_s (\widetilde{\Sigma}_m - \Sigma_m) \sum_{k=1}^{s-1} \prod_{j=k+1}^{s-1} (I - \gamma_j \widetilde{\Sigma}_m ) \gamma_k \varepsilon_k \varphi(\bm x_k)\,.
	\end{split}
\end{equation*}

\begin{proposition}\label{propv2}
	Under Assumptions~\ref{assumdata},~\ref{assumact},~\ref{assump:noise} with $\tau > 0$, if the step-size $\gamma_t := \gamma_0 t^{-\zeta}$ with $\zeta \in [0,1)$ satisfies
	\begin{equation}\label{gammacexp}
		\gamma_0 \leqslant \frac{1}{\mathrm{Tr}(\Sigma_m)} \,,
	\end{equation}
	then ${\tt V2}$ can be bounded by
	\begin{equation*}
		{\tt V2} \lesssim  \left\{
		\begin{array}{rcl}
			\begin{split}
				& \gamma_0\tau^2 \frac{m}{n^{1-\zeta}} ,~\quad \mbox{if $m \leqslant n$} \\
				& \gamma_0\tau^2 ,~\quad \mbox{if $m > n$}\,.
			\end{split}
		\end{array} \right.
	\end{equation*}
\end{proposition}

To prove Proposition~\ref{propv2}, we need the following lemma.

\begin{lemma}\label{lemcvx-wada}
	Denote $C^{{\tt vX-W}}_t := \mathbb{E}_{\bm X, \bm \varepsilon} [{\alpha}^{{\tt vX-W}}_t \otimes {\alpha}^{{\tt vX-W}}_t]$, under Assumptions~\ref{assumdata},~\ref{assumact},~\ref{assump:noise} with $\tau > 0$, if the step-size $\gamma_t := \gamma_0 t^{-\zeta}$ with $\zeta \in [0,1)$ satisfies
	\begin{equation*}
		\gamma_0 \leqslant \min \left\{ \frac{1}{\mathrm{Tr}(\Sigma_m)}, \frac{1}{\mathrm{Tr}(\widetilde{\Sigma}_m)} \right\}\,,
	\end{equation*} we have
	%	\begin{equation*}\label{Cvx-wtada}
		%		\mathrm{Tr}(C^{{\tt vX-W}}_t) \lesssim \left\{
		%		\begin{array}{rcl}
			%			\begin{split}
				%				& \tau^2 \mathrm{Tr}(\widetilde{\Sigma}_m)  \gamma_0^3[\mathrm{Tr}(\Sigma_m)]^2 \left\| I + \widetilde{\Sigma}_m^{-2} \Sigma^2_m \right\|_2 t  ,~\quad \mbox{if $m \geqslant n$} \\
				%				& \tau^2 \mathrm{Tr}(\widetilde{\Sigma}_m) \gamma_0^2\mathrm{Tr}(\Sigma_m) \left\| I + \widetilde{\Sigma}_m^{-2} \Sigma^2_m \right\|_2 m,~\quad \mbox{otherwise}\,.
				%			\end{split}
			%		\end{array} \right.
		%	\end{equation*}
	\begin{equation*}\label{Cvx-wtada}
		\| C^{{\tt vX-W}}_t \|_2 \lesssim \tau^2 \gamma_0^2 \left( \gamma_0 \| \Sigma_m \|_2 + 1 \right) \left( \gamma_0 \| \widetilde{\Sigma}_m \|_2 + 1 \right) \,.
	\end{equation*}
\end{lemma}

\begin{proof}
	According to the definition of $C^{{\tt vX-W}}_t$, it admits the following expression
	\begin{equation*}
		\begin{split}
			C^{{\tt vX-W}}_t 
			& = \sum_{s=1}^t \prod_{i=s+1}^t (I - \gamma_i \Sigma_m) \gamma_s^2 (\widetilde{\Sigma}_m - \Sigma_m) \sum_{k=1}^{s-1} \prod_{j=k+1}^{s-1} (I - \gamma_j \widetilde{\Sigma}_m)^2 \gamma_k^2 \Xi (\widetilde{\Sigma}_m - \Sigma_m) (I - \gamma_i \Sigma_m) \\
			& \preccurlyeq \sum_{s=1}^t \prod_{i=s+1}^t (I - \gamma_i \Sigma_m) \gamma_s^2 (\widetilde{\Sigma}_m - \Sigma_m) \sum_{k=1}^{s-1} \prod_{j=k+1}^{s-1} (I - \gamma_j \widetilde{\Sigma}_m)^2 \gamma_k^2 \Xi (\widetilde{\Sigma}_m - \Sigma_m) (I - \gamma_i \Sigma_m) \\
			& \preccurlyeq \tau^2 \sum_{s=1}^t \prod_{i=s+1}^t (I - \gamma_i \Sigma_m) \gamma_s^2 (\widetilde{\Sigma}_m - \Sigma_m) \sum_{k=1}^{s-1} \prod_{j=k+1}^{s-1} (I - \gamma_j \widetilde{\Sigma}_m)^2 \gamma_k^2 \Sigma_m (\widetilde{\Sigma}_m - \Sigma_m) (I - \gamma_i \Sigma_m) \,,
		\end{split}
	\end{equation*}
	where the first equality holds by $\mathbb{E}[\varepsilon_i \varepsilon_j] =0 $ for $i\neq j$ and the second inequality holds by Assumption~\ref{assump:noise}.
	
	Accordingly, $\| C^{{\tt vX-W}}_t \|_2$ can be upper bounded by
	\begin{equation}\label{trcvx-w1}
		\begin{split}
			\| C^{{\tt vX-W}}_t \|_2 & \leqslant \tau^2 \sum_{s=1}^t \gamma_s^2  \left\| \prod_{i=s+1}^t (I - \gamma_i \Sigma_m)^2 \Sigma_m (\widetilde{\Sigma}_m - \Sigma_m)^2 \sum_{k=1}^{s-1} \gamma_k^2 \prod_{j=k+1}^{s-1} (I - \gamma_j \widetilde{\Sigma}_m)^2 \right\|_2 \notag \\
			& \leqslant \tau^2 \sum_{s=1}^t \gamma_s^2 \left\| \prod_{i=s+1}^t (I - \gamma_i \Sigma_m)^2 \Sigma_m \right\|_2 \left\| \sum_{k=1}^{s-1} \gamma_k^2 \prod_{j=k+1}^{s-1} (I - \gamma_j \widetilde{\Sigma}_m)^2 \widetilde{\Sigma}_m \right\|_2 \left\| \widetilde{\Sigma}_m -2 \Sigma_m + \widetilde{\Sigma}_m^{-1} \Sigma^2_m \right\|_2 \notag \\
			%& \lesssim \tau^2 \sum_{s=1}^t \gamma_s^2 \mathrm{Tr} \left[ \prod_{i=s+1}^t (I - \gamma_i \Sigma_m)^2 \Sigma_m \right] \mathrm{Tr} \left[ \sum_{k=1}^{s-1} \gamma_k^2 \prod_{j=k+1}^{s-1} (I - \gamma_j \widetilde{\Sigma}_m)^2 \widetilde{\Sigma}_m^2 \right] \quad \mbox{[using Lemma~\ref{trace1}]} \\
			%& \leqslant \tau^2 \sum_{s=1}^t \gamma_s^2 \sum_{q=1}^m \prod_{i=s+1}^t (1 - \gamma_i \lambda_q)^2 \lambda_q \sum_{k=1}^{s-1} \gamma_k^2 \sum_{p=1}^m \prod_{j=k+1}^{s-1} (1 - \gamma_j \widetilde{\lambda}_p)^2 \widetilde{\lambda}_p^2 \\
			& \lesssim \tau^2 \sum_{s=1}^t \max_{q \in \{ 1,2,\dots,m \}} \gamma_s^2 \exp \left(-2\lambda_q \sum_{i=s+1}^t \gamma_i \right)\lambda_q \sum_{k=1}^{s-1} \gamma_k^2 \max_{p \in \{1,2 \}} \exp \left(- 2 \widetilde{\lambda}_p \sum_{j=k+1}^{s-1} \gamma_j \right) \widetilde{\lambda}_p \\
			& \qquad \left\| \widetilde{\Sigma}_m -2 \Sigma_m + \widetilde{\Sigma}_m^{-1} \Sigma^2_m \right\|_2  \,.
		\end{split}
	\end{equation}
	%where the last inequality holds by Lemma~\ref{trace1}.
	
	Similar to Eq.~\eqref{Iisecond}, we have the following estimation
	\begin{equation*}
		\begin{split}
			\sum_{k=1}^{s-1} \gamma_k^2  \prod_{j=k+1}^{s-1} (1 - \gamma_j \widetilde{\lambda}_p)^2 
			& \leqslant  \sum_{k=1}^{s-1} \gamma_k^2 \exp\left(- 2 \widetilde{\lambda}_p \sum_{j=k+1}^{s-1} \gamma_j \right)  \\
			& \leqslant \gamma_{s-1}^2 + \gamma_0^2 \int_1^{s-1} u^{-2\zeta} \exp \bigg( -2\widetilde{\lambda}_p \gamma_0  \frac{s^{1-\zeta} - (u+1)^{1-\zeta}}{1-\zeta} \bigg) \mathrm{d}u \\
			& \leqslant  \gamma_0^2 + \left( \frac{\gamma_0}{\widetilde{\lambda}_p} \wedge \gamma_0^2 s \right) \,,
		\end{split}
	\end{equation*}
	which implies 
	\begin{equation}\label{maxtildelambda}
		\max_{ p= 1,2 }~~  \widetilde{\lambda}_p \sum_{k=1}^{s-1} \gamma_k^2  \prod_{j=k+1}^{s-1} (1 - \gamma_j \widetilde{\lambda}_p)^2 \leqslant \gamma_0^2 \widetilde{\lambda}_1 + \gamma_0  \leqslant \gamma_0^2 \widetilde{\Sigma}_m + \gamma_0\,.
	\end{equation}
	
	Similar to Eq.~\eqref{Iisecond}, we have the following estimation
	\begin{equation*}
		\begin{split}
			\sum_{s=1}^t \gamma_s^2 \exp \left(-2\lambda_q \sum_{i=s+1}^t \gamma_i \right) & \leqslant  \sum_{s=1}^{t} \gamma_s^2 \exp \bigg( -2{\lambda}_q \gamma_0  \frac{(t+1)^{1-\zeta} - (s+1)^{1-\zeta}}{1-\zeta} \bigg) \\
			& \leqslant \gamma_t^2 + \gamma_0^2 \int_1^t  u^{-2\zeta} \exp \bigg( -2{\lambda}_q \gamma_0  \frac{(t+1)^{1-\zeta} - (u+1)^{1-\zeta}}{1-\zeta} \bigg) \mathrm{d}u \\
			& \leqslant \gamma_0^2 + \left( \frac{\gamma_0}{{\lambda}_q} \wedge \gamma_0^2 t \right)\,,
		\end{split}
	\end{equation*}
	which implies 
	\begin{equation}\label{maxlambda}
		\max_{q \in \{1,2,\dots,m \}} \sum_{s=1}^t \gamma_s^2 \lambda_q \exp \left(-2\lambda_q \sum_{i=s+1}^t \gamma_i \right) = \gamma_0^2 \| \Sigma_m \|_2 + \gamma_0 \,.
	\end{equation}
	Combining the above two equations~\eqref{maxtildelambda} and~\eqref{maxlambda}, we have
	\begin{equation*}
		\begin{split}
			\| C^{{\tt vX-W}}_t\|_2 & \lesssim \tau^2 \gamma_0^2 \left( \gamma_0 \| \Sigma_m \|_2 + 1 \right) \left( \gamma_0 \| \widetilde{\Sigma}_m \|_2 + 1 \right) \,.
			%		& \lesssim \left\{
			%		\begin{array}{rcl}
				%			\begin{split}
					%				& \tau^2 \mathrm{Tr}(\widetilde{\Sigma}_m) \gamma_0^3[\mathrm{Tr}(\Sigma_m)]^2 \left\| I + \widetilde{\Sigma}_m^{-2} \Sigma^2_m \right\|_2 t  ,~\quad \mbox{if $m \geqslant n$} \\
					%				& \tau^2 \mathrm{Tr}(\widetilde{\Sigma}_m) \gamma_0^2\mathrm{Tr}(\Sigma_m) \left\| I + \widetilde{\Sigma}_m^{-2} \Sigma^2_m \right\|_2 m,~\quad \mbox{otherwise}\,,
					%			\end{split}
				%		\end{array} \right.
		\end{split}
	\end{equation*}
\end{proof}

\begin{proof}
	[Proof of Proposition~\ref{propv2}]
	By virtue of $\mathbb{E}_{\bm X, \bm \varepsilon} [{\alpha}^{{\tt vX-W}}_t| {\alpha}^{{\tt vX-W}}_{t-1}] = (I - \gamma_t \Sigma_m) {\alpha}^{{\tt vX-W}}_{t-1}$ and Lemma~\ref{lemcvx-wada}, ${\tt V2}$ can be bounded by
	\begin{equation*}
		\begin{split}
			{\tt V2} &= \mathbb{E}_{\bm X, \bm W, \bm \varepsilon} \big[ \langle \bar{\eta}^{{\tt vX}}_n \!-\! \bar{\eta}^{{\tt vXW}}_n, \Sigma_m ( \bar{\eta}^{{\tt vX}}_n \!-\! \bar{\eta}^{{\tt vXW}}_n ) \rangle \big] = \mathbb{E}_{\bm W} \langle  \Sigma_m, \mathbb{E}_{\bm X, \bm \varepsilon} [\bar{\alpha}^{{\tt vX-W}}_n \otimes \bar{\alpha}^{{\tt vX-W}}_n ] \rangle \\
			& \leqslant  \frac{2}{n^2} \sum_{t=0}^{n-1} \sum_{k=t}^{n-1} \mathbb{E}_{\bm W} \left\langle  \prod_{j=t}^{k-1}(I-\gamma_j {\Sigma}_m)  {\Sigma}_m, \underbrace{\mathbb{E}_{\bm X, \bm \varepsilon} [{\eta}^{{\tt vX-W}}_t \otimes {\eta}^{{\tt vX-W}}_t]}_{:= C^{{\tt vX-W}}_t} \right\rangle \\
			&  \lesssim \frac{\tau^2\gamma_0^2}{n^2} \| \widetilde{\Sigma}_m \|_2  \mathbb{E}_{\bm W} \left( \left\| \widetilde{\Sigma}_m -2 \Sigma_m + \widetilde{\Sigma}_m^{-1} \Sigma^2_m \right\|_2  \left[ \| \Sigma_m\|_2 \gamma_0 + 1 \right] \mathrm{Tr}\left[\sum_{t=0}^{n-1} \sum_{k=t}^{n-1} \prod_{j=t}^{k-1}(I-\gamma_j {\Sigma}_m)  {\Sigma}_m \right]  \right) \\
			& \lesssim \frac{\tau^2\gamma_0^2}{n^2} \| \widetilde{\Sigma}_m \|_2  \mathbb{E}_{\bm W} \left[ \| \Sigma_m \|_2 \left\| \widetilde{\Sigma}_m -2 \Sigma_m + \widetilde{\Sigma}_m^{-1} \Sigma^2_m \right\|_2 \sum_{i=1}^m \sum_{t=0}^{n-1} \lambda_i \left( \frac{n^{\zeta}}{\lambda_i \gamma_0} \wedge (n-t) \right) \right]\,. \quad \mbox{[using Eq.~\eqref{intut}]}
		\end{split}
	\end{equation*}
	
	In the $m \leqslant n$ case, we choose $n^{\zeta}/(\lambda_i \gamma_0)$, and thus
	\begin{equation*}
		\begin{split}
			{\tt V2} & \lesssim \frac{\tau^2m\gamma_0^2}{n^2} \| \widetilde{\Sigma}_m \|_2 \mathbb{E}_{\bm W} \left[ \| \Sigma_m \|_2 \left\| \widetilde{\Sigma}_m -2 \Sigma_m + \widetilde{\Sigma}_m^{-1} \Sigma^2_m \right\|_2 \right] \frac{n^{1+\zeta}}{\gamma_0} \\
			& \leqslant \tau^2\gamma_0\frac{m\| \widetilde{\Sigma}_m \|_2 }{n^{1-\zeta}} \sqrt{\mathbb{E}_{\bm W} \| \Sigma_m \|_2^2} \sqrt{\mathbb{E}_{\bm W} \left\| \widetilde{\Sigma}_m -2 \Sigma_m + \widetilde{\Sigma}_m^{-1} \Sigma^2_m \right\|_2^2}  \quad \mbox{[using Cauchy–Schwarz inequality]}\\
			& \lesssim \tau^2\gamma_0\frac{m}{n^{1-\zeta}}\,. \quad \mbox{[using Lemma~\ref{lemsubexp} and ~\ref{trace1}]}\\
		\end{split}
	\end{equation*}
	
	If $m > n$, we have
	\begin{equation*}
		\begin{split}
			{\tt V2} & \lesssim \frac{2\tau^2\gamma_0^2}{n^2} \| \widetilde{\Sigma}_m \|_2 \mathbb{E}_{\bm W} \left( [\mathrm{Tr}(\Sigma_m)]^2 \left\| \widetilde{\Sigma}_m -2 \Sigma_m + \widetilde{\Sigma}_m^{-1} \Sigma^2_m \right\|_2 \right) \sum_{t=0}^{n-1}t \\
			& \leqslant \tau^2\gamma_0 \| \widetilde{\Sigma}_m \|_2 \sqrt{\mathbb{E}_{\bm W} [ \mathrm{Tr}(\Sigma_m)]^2} \sqrt{\mathbb{E}_{\bm W} \left\| \widetilde{\Sigma}_m -2 \Sigma_m + \widetilde{\Sigma}_m^{-1} \Sigma^2_m \right\|_2^2}  \\
			& \lesssim \tau^2\gamma_0\,, \quad \mbox{[using Lemmas~\ref{lemsubexp} and ~\ref{trace1}]}\\
		\end{split}
	\end{equation*}
	which concludes the proof.
\end{proof}

\subsection{Bound for ${\tt V1}$}
\label{app:v1}

Here we aim to bound ${\tt V1}$
\begin{equation*}
	{\tt V1} := \mathbb{E}_{\bm X, \bm W, \bm \varepsilon} \big[ \langle \bar{\eta}^{{\tt var}}_n \!-\! \bar{\eta}^{{\tt vX}}_n, \Sigma_m ( \bar{\eta}^{{\tt var}}_n \!-\! \bar{\eta}^{{\tt vX}}_n ) \rangle \big] \,.
\end{equation*}
Recall the definition of ${\eta}^{{\tt var}}_t$ in Eq.~\eqref{eq:variance_iterates} and ${\eta}^{{\tt vX}}_t$ in Eq.~\eqref{eq:var_xada},
%\begin{equation*}
	%	\eta_t^{{\tt vX}} = \sum_{k=1}^t \gamma_k \prod_{j=k+1}^t (I - \gamma_j {\Sigma}_m ) \varepsilon_k \varphi(\bm x_k)\,,
	%\end{equation*}
we define
\begin{equation*}
	\begin{split}
		{\alpha}^{{\tt v-X}}_t & := {\eta}^{{\tt var}}_t - {\eta}^{{\tt vX}}_t = [I - \gamma_t \varphi(\bm x_t)\otimes \varphi(\bm x_t)] {\alpha}^{{\tt v-X}}_{t-1} + \gamma_t [\Sigma_m - \varphi(\bm x_t)\otimes \varphi(\bm x_t)] {\eta}^{{\tt vX}}_{t-1}\,, \quad \mbox{with $\alpha^{{\tt v-X}}_0 = 0$}\,. \\
		& = [I - \gamma_t \varphi(\bm x_t)\otimes \varphi(\bm x_t) ] {\alpha}^{{\tt v-X}}_{t-1} + \gamma_t [\Sigma_m - \varphi(\bm x_t)\otimes \varphi(\bm x_t)] \sum_{k=1}^{t-1} \prod_{j=k+1}^{t-1} (I - \gamma_j {\Sigma}_m ) \gamma_k \varepsilon_k \varphi(\bm x_k)\\
		& = \sum_{s=1}^t \prod_{i=s+1}^t \gamma_s [I - \gamma_i \varphi(\bm x_i)\otimes \varphi(\bm x_i) ] [\Sigma_m - \varphi(\bm x_t)\otimes \varphi(\bm x_t)] \sum_{k=1}^{s-1} \prod_{j=k+1}^{s-1} (I - \gamma_j {\Sigma}_m ) \gamma_k \varepsilon_k \varphi(\bm x_k)\,,
	\end{split}
\end{equation*}
and thus the error bound for ${\tt V1}$ is given by the following proposition.

\begin{proposition}\label{propv1}
	Under Assumption~\ref{assumdata},~\ref{assexist},~\ref{assumact},~\ref{assump:bound_fourthmoment} with $r' \geqslant 1$, and Assumption~\ref{assump:noise} with $\tau > 0$, if the step-size $\gamma_t := \gamma_0 t^{-\zeta}$ with $\zeta \in [0,1)$ satisfies
	\begin{equation*}
		\gamma_0 < \min \left\{ \frac{1}{r'\mathrm{Tr}(\Sigma_m)}, \frac{1}{2\mathrm{Tr}(\Sigma_m)} \right\} \,,
	\end{equation*}
	then ${\tt V1}$ can be bounded by
	\begin{equation*}
		{\tt V1} \lesssim  \frac{\tau^2r'\gamma_0^2}{\sqrt{\mathbb{E}[1- \gamma_0 r' \mathrm{Tr}(\Sigma_m)]^2}} \left\{
		\begin{array}{rcl}
			\begin{split}
				& \frac{m}{n^{1-\zeta}} ,~\quad \mbox{if $m \leqslant n$} \\
				& 1 ,~\quad \mbox{if $m > n$}\,.
			\end{split}
		\end{array} \right.
	\end{equation*}
\end{proposition}

To prove Proposition~\ref{propv1}, we need the following lemma.
Define $C^{{\tt v-X}}_t := \mathbb{E}_{\bm X, \bm \varepsilon} [{\alpha}^{{\tt v-X}}_t \otimes {\alpha}^{{\tt v-X}}_t]$, we have the following lemma that is useful to bound $C^{{\tt v-X}}_t$.

\begin{lemma}\label{lemcv-xada}
	Denote $C^{{\tt v-X}}_t := \mathbb{E}_{\bm X, \bm \varepsilon} [{\alpha}^{{\tt v-X}}_t \otimes {\alpha}^{{\tt v-X}}_t]$, under Assumptions~\ref{assumdata},~\ref{assexist},~\ref{assumact},~\ref{assump:bound_fourthmoment} with $r' \geqslant 1$, and Assumption~\ref{assump:noise} with $\tau > 0$, if the step-size $\gamma_t := \gamma_0 t^{-\zeta}$ with $\zeta \in [0,1)$ satisfies
	\begin{equation*}
		\gamma_0 < \min \left\{ \frac{1}{r'\mathrm{Tr}(\Sigma_m)}, \frac{1}{c'\mathrm{Tr}(\Sigma_m)} \right\} \,,
	\end{equation*}
	where $c'$ is defined in Eq.~\eqref{eqconstant}.
	Then, we have
	\begin{equation*}\label{Cv-xtada}
		C^{{\tt v-X}}_t \preccurlyeq \frac{\gamma_0^2 r' \tau^2 [\mathrm{Tr}(\Sigma_m)  + \gamma_0 \mathrm{Tr}(\Sigma_m^2)] }{1-\gamma_0 r' \mathrm{Tr}(\Sigma_m)} I \,.
	\end{equation*}
\end{lemma}
\begin{proof}
	According to the definition of $C^{{\tt v-X}}_t$, it admits the following expression
	\begin{equation}\label{cv-xt}
		\begin{split}
			C^{{\tt v-X}}_t & = \sum_{s=1}^t \prod_{i=s+1}^t \gamma^2_s \mathbb{E}_{\bm x}[I - \gamma_i \varphi(\bm x_i)\otimes \varphi(\bm x_i) ]^2 \mathbb{E}_{\bm x}[\Sigma_m - \varphi(\bm x_t)\otimes \varphi(\bm x_t)]^2 \sum_{k=1}^{s-1} \prod_{j=k+1}^{s-1} (I - \gamma_j {\Sigma}_m )^2 \gamma^2_k \Xi \\
			& = (I - \gamma_t T^{\tt W}) \circ C^{{\tt v-X}}_{t-1} + \gamma_t^2 (S^{\tt W} - \widetilde{S}^{\tt W}) \circ  \sum_{k=1}^{t-1} \prod_{j=k+1}^{t-1} (I - \gamma_j {\Sigma}_m)^2 \gamma^2_k \Xi \quad \mbox{[using PSD operators]}\\
			& \preccurlyeq (I - \gamma_t T^{\tt W}) \circ C^{{\tt v-X}}_{t-1} + \gamma_t^2 S^{\tt W} \circ  \sum_{k=1}^{t-1} \prod_{j=k+1}^{t-1} (I - \gamma_j {\Sigma}_m)^2 \gamma^2_k \Xi \quad \mbox{[using $S^{\tt W} \succcurlyeq \widetilde{S}^{\tt W}$]} \\
			& \preccurlyeq (I - \gamma_t T^{\tt W}) \circ C^{{\tt v-X}}_{t-1} + \tau^2 \gamma_t^2 S^{\tt W} \circ  \sum_{k=1}^{t-1} \prod_{j=k+1}^{t-1} (I - \gamma_j {\Sigma}_m)^2 \gamma^2_k \Sigma_m \quad \mbox{[using Assumption~\ref{assump:noise}]} \\
			& \preccurlyeq (I - \gamma_t T^{\tt W}) \circ C^{{\tt v-X}}_{t-1} + \tau^2 \gamma_t^2 r' \mathrm{Tr} \left[ \sum_{k=1}^{t-1} \prod_{j=k+1}^{t-1} (I - \gamma_j {\Sigma}_m)^2 \gamma^2_k \Sigma_m^2 \right] \Sigma_m \,. \quad \mbox{[using Assumption~\ref{assump:bound_fourthmoment}]}
		\end{split}
	\end{equation}
	
	Similar to Eq.~\eqref{Iisecond}, we have the following estimation
	\begin{equation*}
		\begin{split}
			\mathrm{Tr}\left[\sum_{k=1}^{t-1} \prod_{j=k+1}^{t-1} (I - \gamma_j {\Sigma}_m)^2 \Sigma_m^2 \gamma_k^2 \right] & =  \sum_{i=1}^m {\lambda}_i^2 \sum_{k=1}^{t-1} \gamma_k^2  \prod_{j=k+1}^{t-1} (1 - \gamma_j {\lambda}_i)^2 
			\leqslant \sum_{i=1}^m {\lambda}_i^2 \sum_{k=1}^{t-1} \gamma_k^2 \exp\left(- 2 {\lambda}_i \sum_{j=k+1}^{s-1} \gamma_j \right)  \\
			& \leqslant \gamma_0^2 \sum_{i=1}^m \lambda_i^2 \left[ 1 + \int_1^{t-1} u^{-2\zeta} \exp \bigg( -2{\lambda}_i \gamma_0  \frac{t^{1-\zeta} - (u+1)^{1-\zeta}}{1-\zeta} \bigg) \mathrm{d}u \right] \\
			& \leqslant \gamma_0^2 \mathrm{Tr}(\Sigma_m^2) + \sum_{i=1}^m \lambda_i^2 \left( \frac{\gamma_0}{\lambda_i} \wedge \gamma_0^2 t \right) \quad \mbox{[using Eq.~\eqref{intu2zetatu}]} \\
			&  \leqslant \gamma_0^2 \mathrm{Tr}(\Sigma_m^2) + \gamma_0 \mathrm{Tr}(\Sigma_m) \,,
		\end{split}
	\end{equation*}
	where we use the error bound $\frac{\gamma_0}{\lambda_i}$ instead of the exact one $\gamma_0^2 t$ for tight estimation.
	
	Taking the above equation back to Eq.~\eqref{cv-xt}, we have
	\begin{equation*}
		\begin{split}
			C^{{\tt v-X}}_t & \preccurlyeq (I - \gamma_t T^{\tt W}) \circ C^{{\tt v-X}}_{t-1} + \gamma_t^2  \tau^2 r'  \gamma_0  [\mathrm{Tr}(\Sigma_m) + \gamma_0 \mathrm{Tr}(\Sigma_m^2)] \Sigma_m \\
			& \preccurlyeq \tau^2 r'  \gamma_0  [\mathrm{Tr}(\Sigma_m) + \gamma_0 \mathrm{Tr}(\Sigma_m^2)] \sum_{s=1}^t \prod_{i=s+1}^t (I - \gamma_i T^{\tt W}) \circ \gamma_s^2 \Sigma_m \\
			& \preccurlyeq \frac{\gamma_0^2 r' \tau^2 [\mathrm{Tr}(\Sigma_m)  + \gamma_0 \mathrm{Tr}(\Sigma_m^2)] }{1-\gamma_0 r' \mathrm{Tr}(\Sigma_m)} I \,, \quad \mbox{[using Lemma~\ref{dinfvx}]} 
		\end{split}
	\end{equation*} 
	which concludes the proof.
\end{proof}

\begin{proof}
	[Proof of Proposition~\ref{propv1}]
	
	Accordingly, by virtue of $\mathbb{E}_{\bm X, \bm \varepsilon} [{\alpha}^{{\tt v-X}}_t| {\alpha}^{{\tt v-X}}_{t-1}] = (I - \gamma_t \Sigma_m) {\alpha}^{{\tt v-X}}_{t-1}$ and Lemma~\ref{lemcv-xada}, ${\tt V1}$ can be bounded by
	\begin{equation*}
		\begin{split}
			{\tt V1} &= \mathbb{E}_{\bm X, \bm W, \bm \varepsilon} \big[ \langle \bar{\eta}^{{\tt var}}_n \!-\! \bar{\eta}^{{\tt v-X}}_n, \Sigma_m ( \bar{\eta}^{{\tt var}}_n \!-\! \bar{\eta}^{{\tt v-X}}_n ) \rangle \big] = \mathbb{E}_{\bm W} \langle  \Sigma_m, \mathbb{E}_{\bm X, \bm \varepsilon} [\bar{\alpha}^{{\tt v-X}}_n \otimes \bar{\alpha}^{{\tt v-X}}_n ] \rangle \\
			& \leqslant  \frac{2}{n^2} \sum_{t=0}^{n-1} \sum_{k=t}^{n-1} \mathbb{E}_{\bm W} \left\langle  \prod_{j=t}^{k-1}(I-\gamma_j {\Sigma}_m)  {\Sigma}_m, \underbrace{\mathbb{E}_{\bm X, \bm \varepsilon} [{\eta}^{{\tt v-X}}_t \otimes {\eta}^{{\tt v-X}}_t]}_{:= C^{{\tt v-X}}_t} \right\rangle \\
			& \lesssim \frac{\tau^2\gamma_0^2r'}{n^2}  \mathbb{E}_{\bm W} \left[ \frac{ [\mathrm{Tr}(\Sigma_m)  + \gamma_0 \mathrm{Tr}(\Sigma_m^2)] }{1-\gamma_0 r' \mathrm{Tr}(\Sigma_m)}  \sum_{i=1}^m \sum_{t=0}^{n-1} \lambda_i \left( \frac{n^{\zeta}}{\lambda_i \gamma_0} \wedge (n-t) \right) \right]\,, \quad \mbox{[using Lemma~\ref{lemcv-xada}]} 
		\end{split}
	\end{equation*}
	where the last inequality follows the integral estimation in Eq.~\eqref{intut}.
	
	For $m \leqslant n$, we use $\frac{n^{\zeta}}{\lambda_i \gamma_0}$, and thus
	\begin{equation*}
		\begin{split}
			{\tt V1} & \lesssim \frac{\tau^2\gamma_0r'm}{n^{1-\zeta}}  \mathbb{E}_{\bm W} \left[ \frac{ [\mathrm{Tr}(\Sigma_m)  + \gamma_0 \mathrm{Tr}(\Sigma_m^2)] }{1-\gamma_0 r' \mathrm{Tr}(\Sigma_m)} \right]  \lesssim \frac{\tau^2r'\gamma_0}{\sqrt{\mathbb{E}[1- \gamma_0 r' \mathrm{Tr}(\Sigma_m)]^2}} \frac{m}{n^{1-\zeta}}\,,
		\end{split}
	\end{equation*}
	where we use the Cauchy–Schwarz inequality and $\mathrm{Tr}(\Sigma_m)$ as a nonnegative sub-exponential random variable with the sub-exponential norm $\mathcal{O}(1)$ in Lemma~\ref{lemsubexp}.
	%	\begin{equation*}
		%		\begin{split}
			%			\mathbb{E}_{\bm W}\left( 1+ \frac{r' \mathrm{Tr}(\Sigma_m)}{1 - \gamma_0 \mathrm{Tr}(\Sigma_m)} \right)^2 & = 1 + 2r' \mathbb{E} \left[\frac{X}{1-\gamma_0 X} \right] + r'^2 \mathbb{E} \left[\frac{X^2}{(1-\gamma_0 X)^2} \right] \\
			%			&  \leqslant 1+ 2r' \sqrt{\mathbb{E}X^2} \sqrt{\frac{1}{\mathbb{E}(1-\gamma_0 X)^2}} + r'^2 \sqrt{\mathbb{E}(X^4)} \sqrt{\frac{1}{\mathbb{E}(1-\gamma_0 X)^4}} \\
			%			& \lesssim \mathcal{O}(1) \quad \mbox{[using $X \sim \mathcal{O}(1)$ and $\gamma_0 \leqslant 1/\mathrm{Tr}(\Sigma_m)$]}\,.
			%		\end{split}
		%	\end{equation*}
	
	For $m > n$, take $n-t$ and Eq.~\eqref{stepbound}, we have
	\begin{equation*}
		\begin{split}
			{\tt V1} & \lesssim \tau^2\gamma_0^2r' \mathbb{E}_{\bm W} \left[ \frac{ [\mathrm{Tr}(\Sigma_m)  + \gamma_0 \mathrm{Tr}(\Sigma_m^2)] }{1-\gamma_0 r' \mathrm{Tr}(\Sigma_m)} \right]  \lesssim \frac{\tau^2r'\gamma_0^2}{\sqrt{\mathbb{E}[1- \gamma_0 r' \mathrm{Tr}(\Sigma_m)]^2}} \sim \mathcal{O}(1) \,.
		\end{split}
	\end{equation*}
\end{proof}

\subsection{Proof of Theorem~\ref{promainvar}}
\begin{proof}
	Combining the above results for three terms ${\tt V1}$, ${\tt V2}$, ${\tt V3}$, we can directly obtain the result for ${\tt Variance}$.
	\begin{equation*}
		\begin{split}
			{\tt Variance}	& \leqslant \left( \sqrt{\tt V1} + \sqrt{\tt V2} + \sqrt{\tt V3} \right)^2 \leqslant 3 ({\tt V1} + {\tt V2} + {\tt V3})  \\
			& \lesssim \frac{\gamma_0 r' \tau^2}{\sqrt{\mathbb{E}[1- \gamma_0 r' \mathrm{Tr}(\Sigma_m)]^2}} \left\{ \begin{array}{rcl}
				\begin{split}
					&   mn^{\zeta-1} ,~\mbox{if $m \leqslant n$} \\
					&  1+ n^{\zeta-1} + \frac{n}{m}  ,~\mbox{if $m > n$}  \\
				\end{split}
			\end{array} \right. \\
			& \lesssim \gamma_0 r' \tau^2 \left\{ \begin{array}{rcl}
				\begin{split}
					&   mn^{\zeta-1} ,~\mbox{if $m \leqslant n$} \\
					&  1+ n^{\zeta-1} + \frac{n}{m}  ,~\mbox{if $m > n$}  \\
				\end{split}
			\end{array} \right. \\
			& \sim \left\{ \begin{array}{rcl}
				\begin{split}
					&  \mathcal{O}\left( m n^{\zeta -1} \right) ,~\mbox{if $m \leqslant n$} \\
					&  \mathcal{O}\left(  1+ n^{\zeta-1} + \frac{n}{m} \right) ,~\mbox{if $m > n$}   \\
				\end{split}
			\end{array} \right. 
		\end{split}
	\end{equation*}
	where we use Eq.~\eqref{stepbound} for the third inequality.
	%Regarding to the condition on the step-size, we can also use $\gamma_0 \lesssim \frac{1}{\mathrm{Tr}(\widetilde{\Sigma}_m)}$ instead, as discussed in Appendix~\ref{sec:proofbiasfinal}.
\end{proof}

\section{More experiments}
\label{app:experiment}

\if 0
\begin{figure*}[t]
	\centering
	\subfigure[a synthetic regression dataset]{\label{figapp:opt}
		\includegraphics[width=0.41\linewidth]{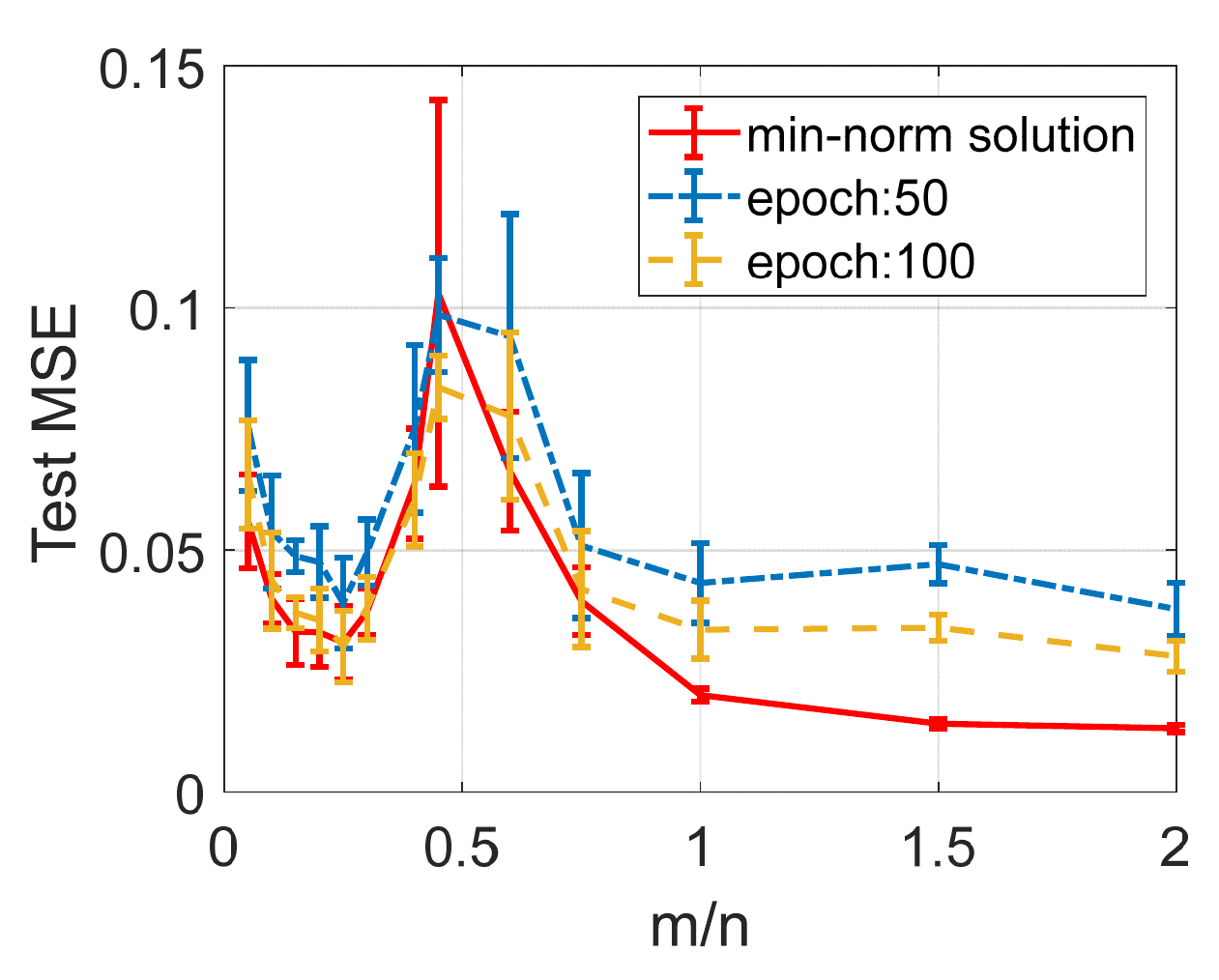}}
		\subfigure[ImageNet16 (class 1 \emph{vs.} 2)]{\label{figapp:imagenet16}
		\includegraphics[width=0.40\linewidth]{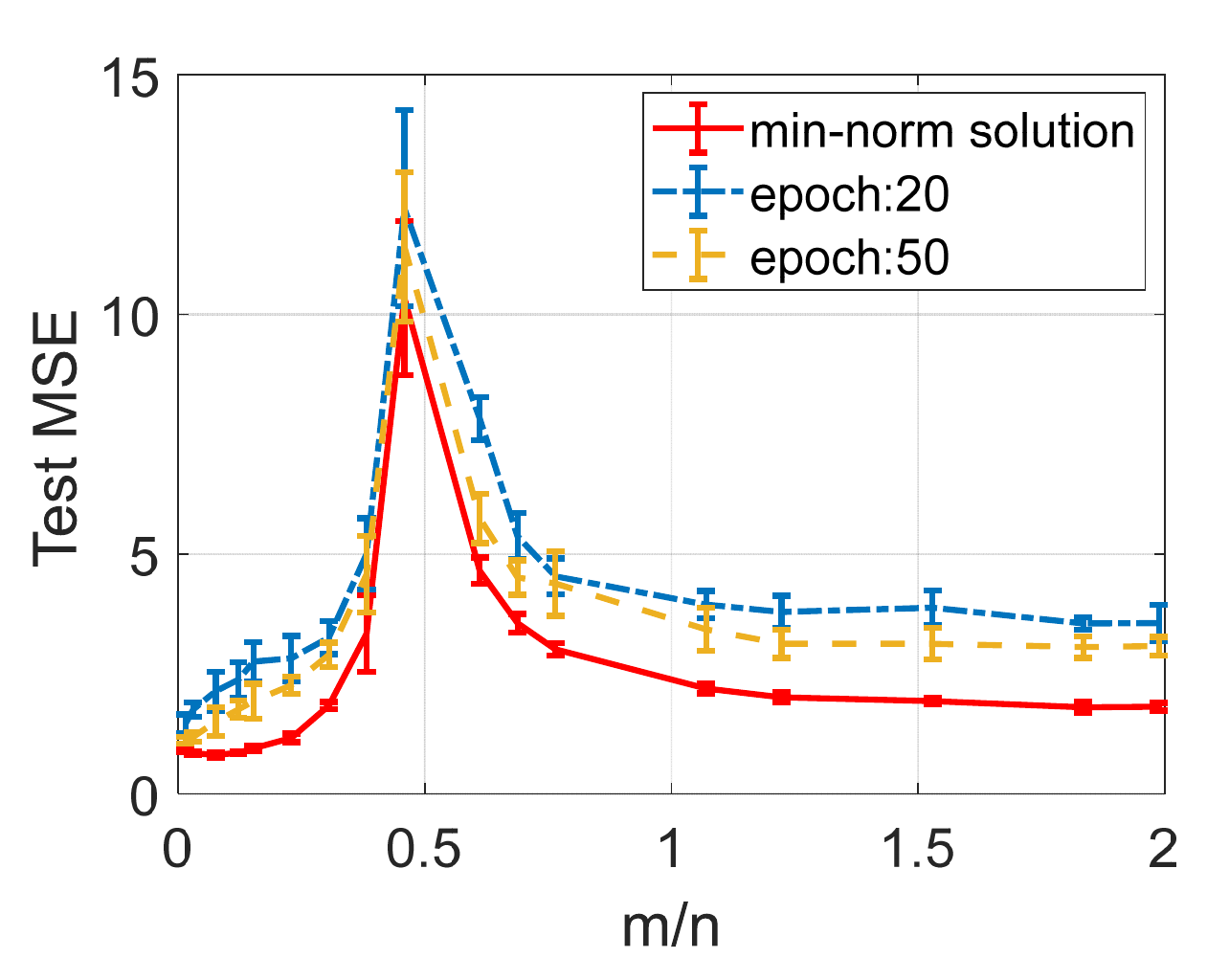}}
	\caption{Normalized MSE (mean$\pm$std.) of RF regression with different epochs on a synthetic regression dataset across the Gaussian kernel in (a) and {\color{red}Test MSE of RF regression on ImageNet16 (class 1 \emph{vs.} 2) under different epochs in (b).} }\label{fig-resepoch}
	\vspace{-0.2cm}
\end{figure*}

\begin{figure*}[t]
	\centering
	\subfigure[${\tt Bias}$]{\label{figapp:bias}
		\includegraphics[width=0.41\textwidth]{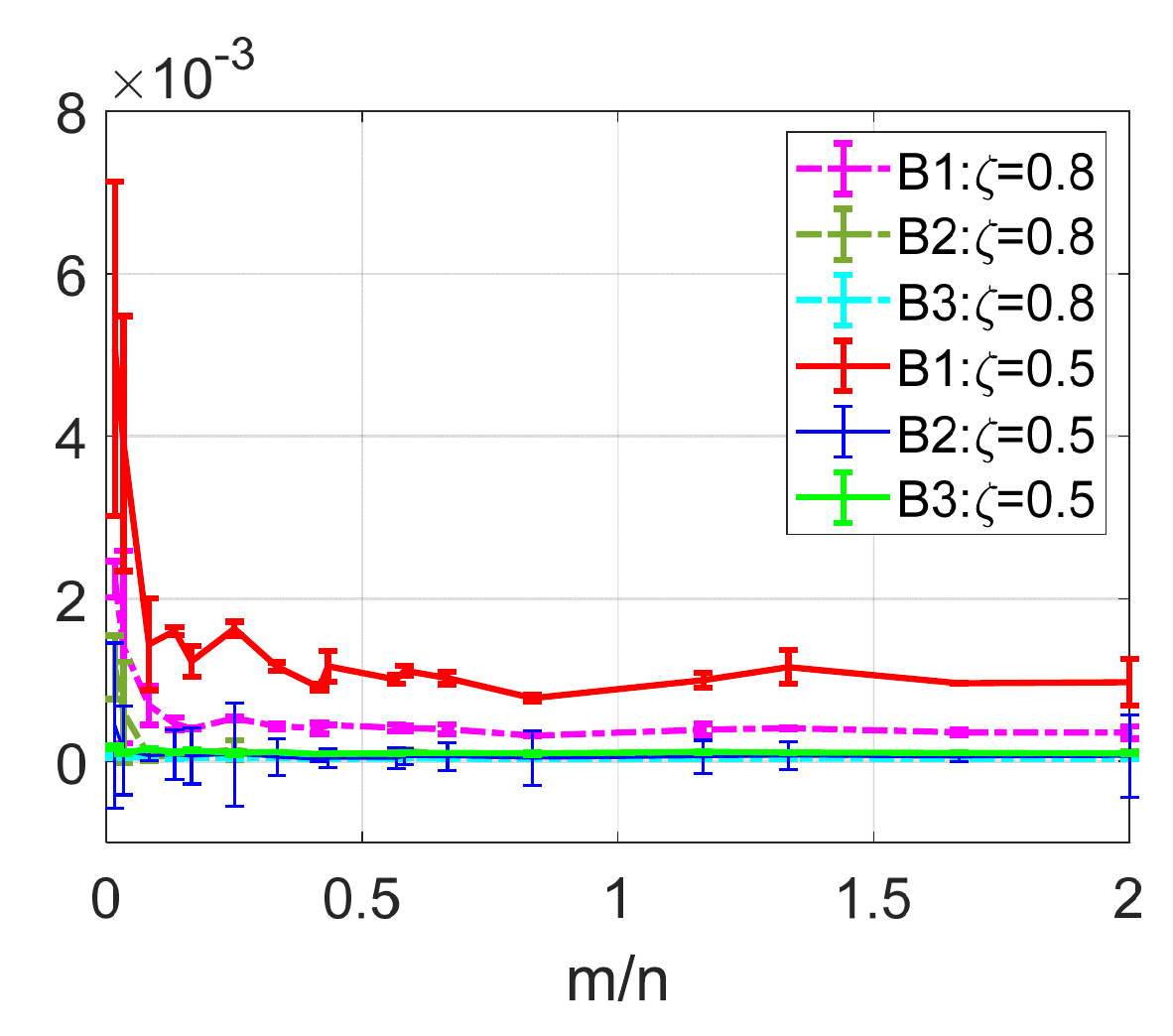}}
	\subfigure[${\tt Variance}$]{\label{figapp:variance}
		\includegraphics[width=0.41\textwidth]{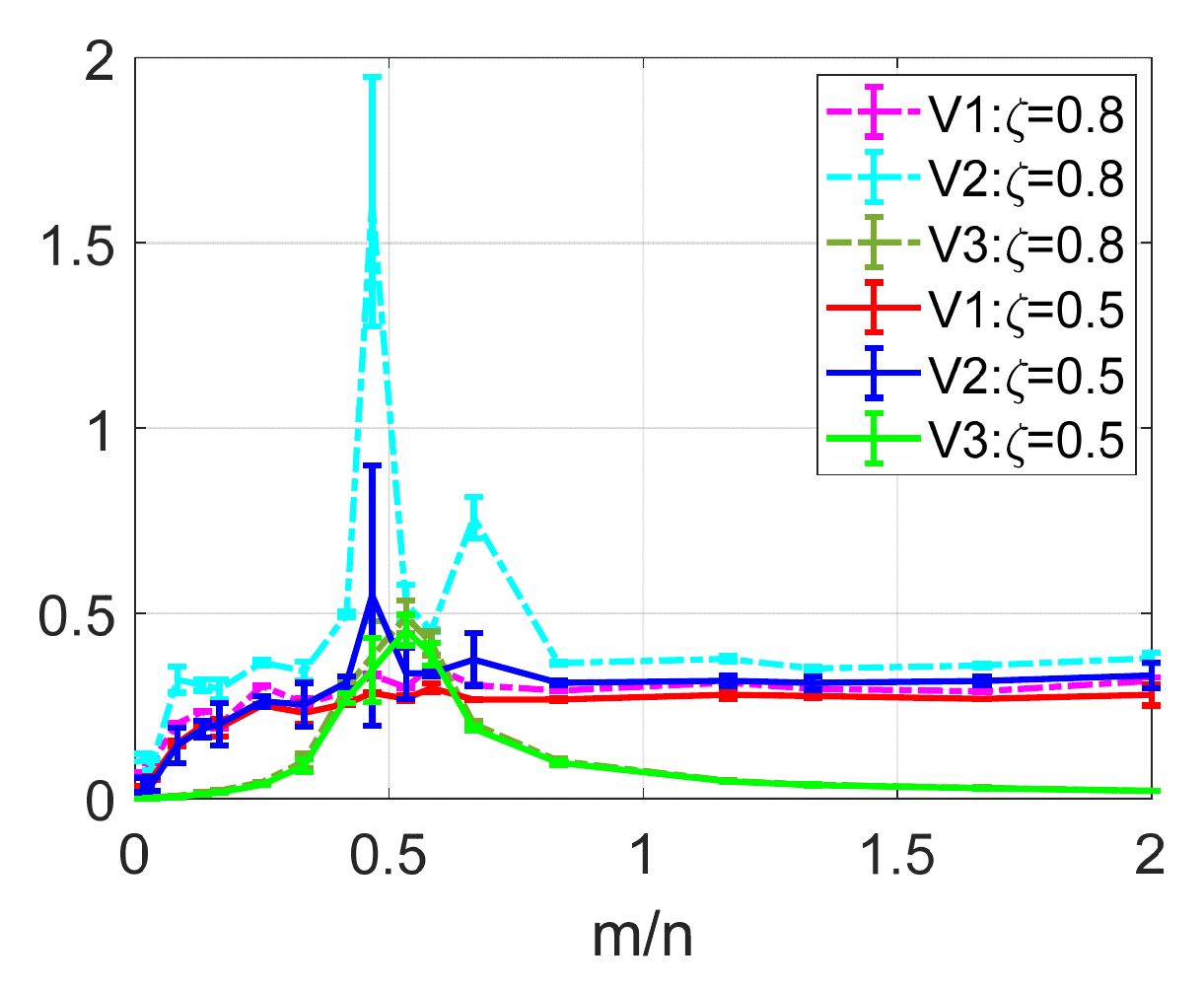}}
	\caption{Trends of ${\tt Bias}$ and ${\tt Variance}$ under different step-size are empirically given in (b) and (c), respectively.  }\label{fig-resapp}
	\vspace{-0.2cm}
\end{figure*}

\fi

\begin{figure*}[t]
	\centering
	\subfigure[a synthetic regression dataset]{\label{figapp:opt}
		\includegraphics[width=0.33\linewidth]{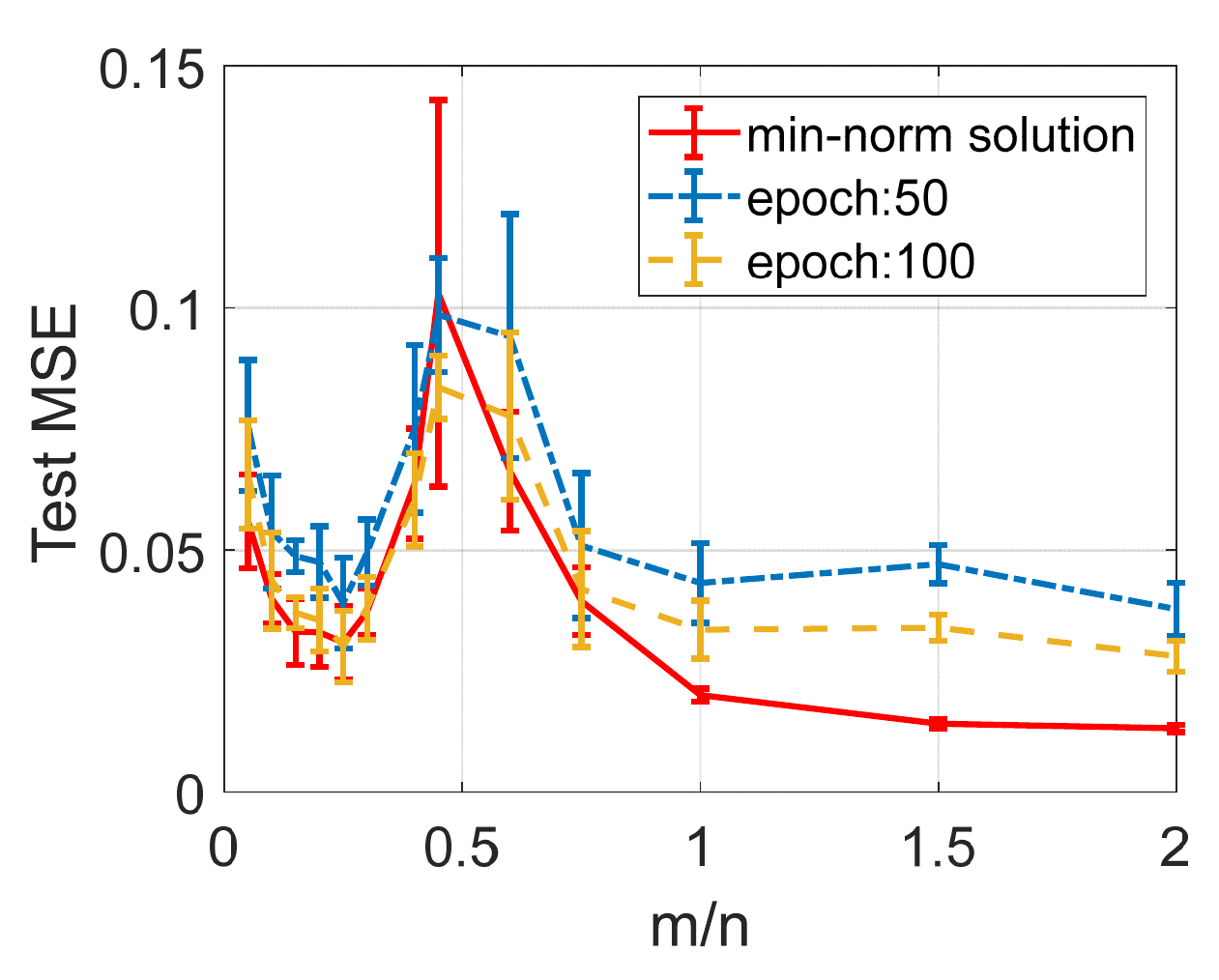}}
	\subfigure[${\tt Bias}$]{\label{figapp:bias}
		\includegraphics[width=0.31\textwidth]{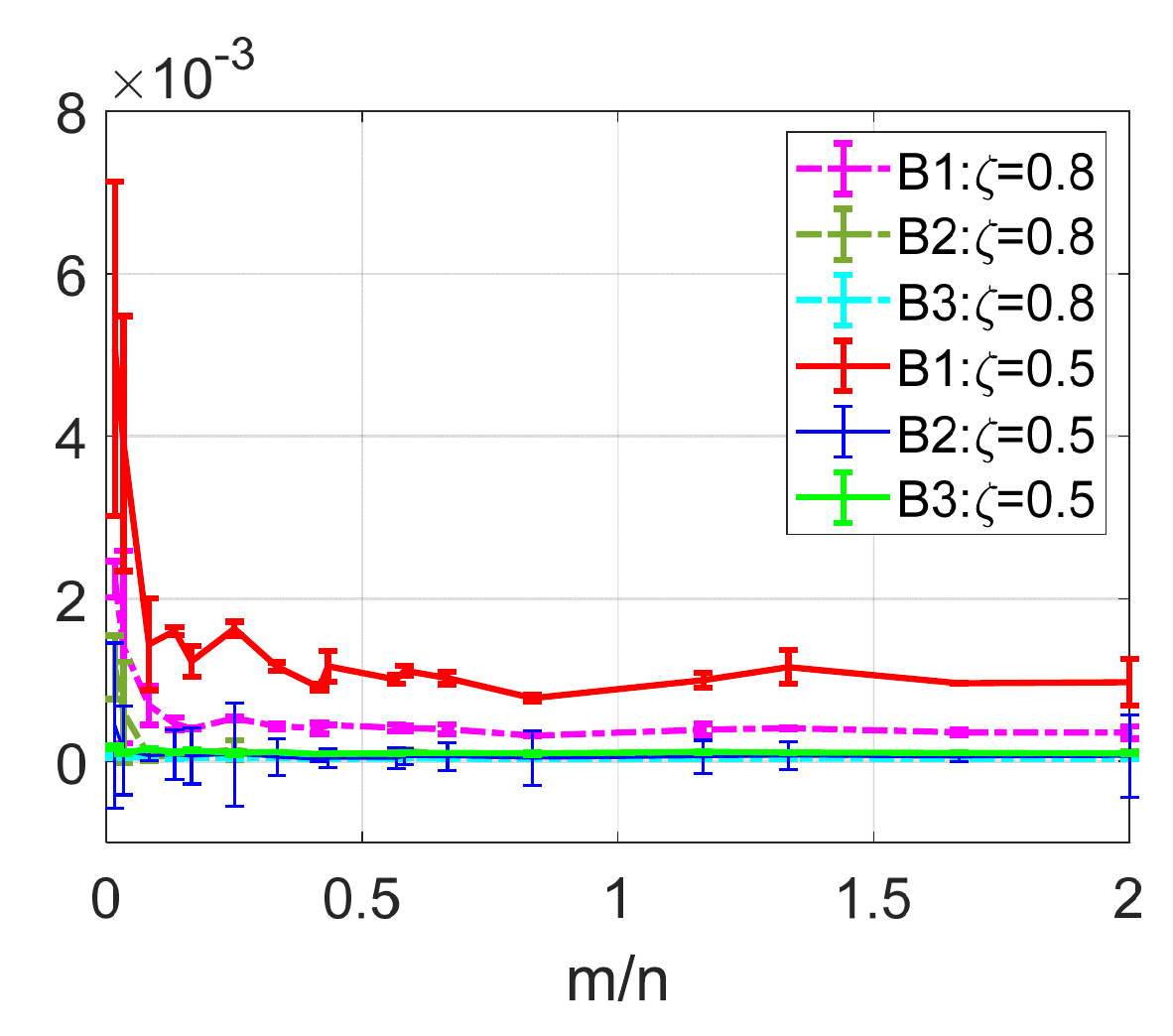}}
	\subfigure[${\tt Variance}$]{\label{figapp:variance}
		\includegraphics[width=0.313\linewidth]{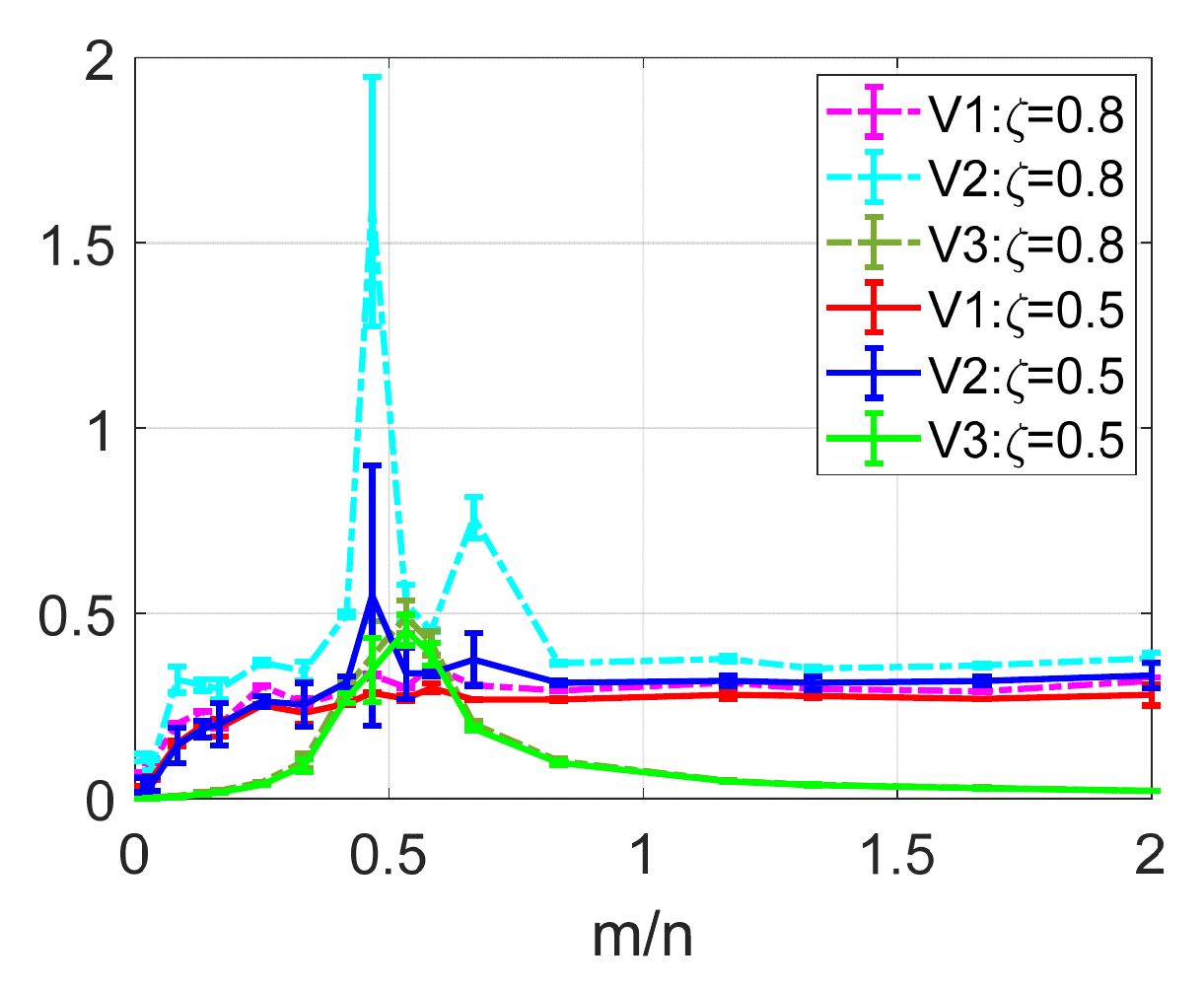}}
	\caption{Normalized MSE (mean$\pm$std.) of RF regression with different epochs on a synthetic regression dataset across the Gaussian kernel in (a); trends of ${\tt Bias}$ and ${\tt Variance}$ under different step-size are empirically given in (b) and (c), respectively.  }\label{fig-resapp}
	\vspace{-0.2cm}
\end{figure*}

In this section, we provide additional experimental results to support our theory.

\subsection{Results on a regression dataset}
We conduct the RF regression via averaged SGD and minimum solution under different initialization schemes and different epochs on a synthetic regression dataset across the Gaussian kernel.

{\bf data generation:}
Apart from the commonly used MNIST in the double descent topic \cite{liao2020random,derezinski2020exact}, we also add a synthetic regression dataset via normalized MSE in Figure~\ref{figapp:opt} for fully supporting our work.
The data are generated from a normal Gaussian distribution with the training data ranging from $n=10$ to $n=400$, the test data being $200$, and the feature dimension $d=50$.
The label is generated by $y = f_{\rho}(\bm x) + \epsilon$, 
where the $\epsilon$ is a Gaussian noise with the variance $0.01$.
The target function $f^*$ is generated by a Laplace kernel $k(\bm x, \bm x') = \exp\left(-\frac{\| \bm x - \bm x' \|_2}{d}\right)$, to ensure $f^* \in \mathcal{H}$. 
To be specific, for any a data point $\bm x \in \mathbb{R}^d$, its target function is $f^*(\bm x) = [k(\bm x, \bm x_1), k(\bm x, \bm x_2), \cdots, k(\bm x, \bm x_n)] \bm w$, where $\bm w \in \mathbb{R}^n$ is a standard random Gaussian vector as a sign.
We remark that the reason why we do not choose the Gaussian kernel as the target function is to avoid the data and model induced by a same (type) kernel.

{\bf experimental settings:}
We follow Figure~\ref{fig:opt} with the same experiment settings, i.e., conducting RF regression via averaged SGD and minimum-norm solution under the Gaussian kernel.
In our experiment, the initial step-size is set to $\gamma_0=1$ with $\zeta=0.5$.
Nevertheless, we take 
\emph{constant initialization} (i.e., set the initialization point as a constant vector) and \emph{different epochs} (i.e., 50 and 100) for broad comparison.

Fig.~\ref{figapp:opt} shows that, first, under this regression dataset with \emph{constant initialization}, we still observe a phase transition between the two sides of the interpolation threshold at $2m = n$ when min-norm solution and averaged SGD are employed, which leads to the double descent phenomenon.
Second, averaged SGD with more epochs result in a better generalization performance, but is still slightly inferior to that with min-norm solution.
We need remark that, when employing gradient descent, under mild conditions, the solution converges to the minimum norm solution, as suggested by \cite{mei2019generalization}.
Nevertheless, whether this result holds for SGD is unclear, depending on the choice of the ground truth, step-size, etc \cite{smith2021origin,vaswani2019fast}.
Studying the property of converged solution is indeed beyond the scope of this paper.

\subsection{Different step-size on Bias and Variance}
Following Section~\ref{sec:expbiasvar}, we also evaluate our error bounds for ${\tt Bias}$ and ${\tt Variance}$ under different step-sizes on the MNIST dataset.
Figure~\ref{figapp:bias} on bias and \ref{figapp:variance} on variance coincides with the results of Section~\ref{sec:expbiasvar}: monotonically decreasing bias and unimodal variance (phase transition of ${\tt V3}$ and non-decreasing ${\tt V1}$ and ${\tt V2}$) under different step-size. 
We remark that, the estimated error bounds are normalized for better illustration, and accordingly we cannot directly compare the value of these components under different step-size. 

%% file: DD-RFF.bbl
\begin{thebibliography}{10}

\bibitem{hastie2019surprises}
Trevor Hastie, Andrea Montanari, Saharon Rosset, and Ryan~J Tibshirani.
\newblock Surprises in high-dimensional ridgeless least squares interpolation.
\newblock {\em Annals of Statistics}, 50(2):949--986, 2022.

\bibitem{bartlett2020benign}
Peter~L. Bartlett, Philip~M. Long, G{\'a}bor Lugosi, and Alexander Tsigler.
\newblock Benign overfitting in linear regression.
\newblock {\em the National Academy of Sciences}, 2020.

\bibitem{wu2020optimal}
Denny Wu and Ji~Xu.
\newblock On the optimal weighted $\ell_2$ regularization in overparameterized
  linear regression.
\newblock In {\em Advances in Neural Information Processing Systems}, pages
  10112--10123, 2020.

\bibitem{mei2019generalization}
Song Mei and Andrea Montanari.
\newblock The generalization error of random features regression: Precise
  asymptotics and the double descent curve.
\newblock {\em Communications on Pure and Applied Mathematics}, 75(4):667--766,
  2022.

\bibitem{nakkiran2019deep}
Preetum Nakkiran, Gal Kaplun, Yamini Bansal, Tristan Yang, Boaz Barak, and Ilya
  Sutskever.
\newblock Deep double descent: Where bigger models and more data hurt.
\newblock In {\em International Conference on Learning Representations}, 2019.

\bibitem{yang2020rethinking}
Zitong Yang, Yaodong Yu, Chong You, Jacob Steinhardt, and Yi~Ma.
\newblock Rethinking bias-variance trade-off for generalization of neural
  networks.
\newblock In {\em International Conference on Machine Learning}, 2020.

\bibitem{ju2021generalization}
Peizhong Ju, Xiaojun Lin, and Ness~B. Shroff.
\newblock On the generalization power of overfitted two-layer neural tangent
  kernel models.
\newblock In {\em International Conference on Machine Learning}, pages
  5137--5147. PMLR, 2020.

\bibitem{zhang2016understanding}
Chiyuan Zhang, Samy Bengio, Moritz Hardt, Benjamin Recht, and Oriol Vinyals.
\newblock Understanding deep learning (still) requires rethinking
  generalization.
\newblock {\em Communications of the ACM}, 64(3):107--115, 2021.

\bibitem{belkin2019reconciling}
Mikhail Belkin, Daniel Hsu, Siyuan Ma, and Soumik Mandal.
\newblock Reconciling modern machine-learning practice and the classical
  bias--variance trade-off.
\newblock {\em the National Academy of Sciences}, 116(32):15849--15854, 2019.

\bibitem{rahimi2007random}
Ali Rahimi and Benjamin Recht.
\newblock Random features for large-scale kernel machines.
\newblock In {\em Advances in Neural Information Processing Systems}, pages
  1177--1184, 2007.

\bibitem{d2020double}
St{\'e}phane d’Ascoli, Maria Refinetti, Giulio Biroli, and Florent Krzakala.
\newblock Double trouble in double descent: Bias and variance (s) in the lazy
  regime.
\newblock In {\em International Conference on Machine Learning}, pages
  2280--2290, 2020.

\bibitem{ba2020generalization}
Jimmy Ba, Murat~A. Erdogdu, Taiji Suzuki, Denny Wu, and Tianzong Zhang.
\newblock Generalization of two-layer neural networks: an asymptotic viewpoint.
\newblock In {\em International Conference on Learning Representations}, pages
  1--8, 2020.

\bibitem{liao2020random}
Zhenyu Liao, Romain Couillet, and Michael Mahoney.
\newblock A random matrix analysis of random fourier features: beyond the
  gaussian kernel, a precise phase transition, and the corresponding double
  descent.
\newblock In {\em Neural Information Processing Systems}, 2020.

\bibitem{gerace2020generalisation}
Federica Gerace, Bruno Loureiro, Florent Krzakala, Marc M{\'e}zard, and Lenka
  Zdeborov{\'a}.
\newblock Generalisation error in learning with random features and the hidden
  manifold model.
\newblock In {\em International Conference on Machine Learning}, pages
  3452--3462, 2020.

\bibitem{lin2020causes}
Licong Lin and Edgar Dobriban.
\newblock What causes the test error? going beyond bias-variance via anova.
\newblock {\em Journal of Machine Learning Research}, 22(155):1--82, 2021.

\bibitem{zou2021benign}
Difan Zou, Jingfeng Wu, Vladimir Braverman, Quanquan Gu, and Sham~M Kakade.
\newblock Benign overfitting of constant-stepsize sgd for linear regression.
\newblock In {\em Conference on Learning Theory}, 2021.

\bibitem{kawaguchi2019gradient}
Kenji Kawaguchi and Jiaoyang Huang.
\newblock Gradient descent finds global minima for generalizable deep neural
  networks of practical sizes.
\newblock In {\em IEEE Conference on Communication, Control, and Computing},
  pages 92--99. IEEE, 2019.

\bibitem{allen2019convergence}
Zeyuan Allen-Zhu, Yuanzhi Li, and Zhao Song.
\newblock A convergence theory for deep learning via over-parameterization.
\newblock In {\em International Conference on Machine Learning}, pages
  242--252. PMLR, 2019.

\bibitem{zou2019improved}
Difan Zou and Quanquan Gu.
\newblock An improved analysis of training over-parameterized deep neural
  networks.
\newblock {\em Advances in Neural Information Processing Systems},
  32:2055--2064, 2019.

\bibitem{jacot2018neural}
Arthur Jacot, Franck Gabriel, and Cl{\'e}ment Hongler.
\newblock Neural tangent kernel: Convergence and generalization in neural
  networks.
\newblock In {\em Advances in Neural Information Processing Systems}, pages
  8571--8580, 2018.

\bibitem{arora2019fine}
Sanjeev Arora, Simon Du, Wei Hu, Zhiyuan Li, and Ruosong Wang.
\newblock Fine-grained analysis of optimization and generalization for
  overparameterized two-layer neural networks.
\newblock In {\em International Conference on Machine Learning}, pages
  322--332, 2019.

\bibitem{chizat2019lazy}
Lenaic Chizat, Edouard Oyallon, and Francis Bach.
\newblock On lazy training in differentiable programming.
\newblock In {\em Advances in Neural Information Processing Systems}, pages
  2933--2943, 2019.

\bibitem{mei2019mean}
Song Mei, Theodor Misiakiewicz, and Andrea Montanari.
\newblock Mean-field theory of two-layers neural networks: dimension-free
  bounds and kernel limit.
\newblock In {\em Conference on Learning Theory}, pages 2388--2464. PMLR, 2019.

\bibitem{chizat2020implicit}
Lenaic Chizat and Francis Bach.
\newblock Implicit bias of gradient descent for wide two-layer neural networks
  trained with the logistic loss.
\newblock In {\em Conference on Learning Theory}, pages 1305--1338, 2020.

\bibitem{li2021towards}
Zhu Li, Zhi-Hua Zhou, and Arthur Gretton.
\newblock Towards an understanding of benign overfitting in neural networks.
\newblock {\em arXiv preprint arXiv:2106.03212}, 2021.

\bibitem{rocks2020memorizing}
Jason~W Rocks and Pankaj Mehta.
\newblock Memorizing without overfitting: Bias, variance, and interpolation in
  over-parameterized models.
\newblock {\em arXiv preprint arXiv:2010.13933}, 2020.

\bibitem{adlam2020understanding}
Ben Adlam and Jeffrey Pennington.
\newblock Understanding double descent requires a fine-grained bias-variance
  decomposition.
\newblock In {\em Advances in Neural Information Processing Systems}, 2020.

\bibitem{hu2020universality}
Hong Hu and Yue~M Lu.
\newblock Universality laws for high-dimensional learning with random features.
\newblock {\em arXiv preprint arXiv:2009.07669}, 2020.

\bibitem{bach2013non}
Francis Bach and Eric Moulines.
\newblock Non-strongly-convex smooth stochastic approximation with convergence
  rate $o (1/n)$.
\newblock {\em Advances in Neural Information Processing Systems}, 26:773--781,
  2013.

\bibitem{jain2018parallelizing}
Prateek Jain, Sham Kakade, Rahul Kidambi, Praneeth Netrapalli, and Aaron
  Sidford.
\newblock Parallelizing stochastic gradient descent for least squares
  regression: mini-batching, averaging, and model misspecification.
\newblock {\em Journal of Machine Learning Research}, 18, 2018.

\bibitem{dieuleveut2016nonparametric}
Aymeric Dieuleveut and Francis Bach.
\newblock Nonparametric stochastic approximation with large step-sizes.
\newblock {\em Annals of Statistics}, 44(4):1363--1399, 2016.

\bibitem{dieuleveut2017harder}
Aymeric Dieuleveut, Nicolas Flammarion, and Francis Bach.
\newblock Harder, better, faster, stronger convergence rates for least-squares
  regression.
\newblock {\em Journal of Machine Learning Research}, 18(1):3520--3570, 2017.

\bibitem{carratino2018learning}
Luigi Carratino, Alessandro Rudi, and Lorenzo Rosasco.
\newblock Learning with \protect{SGD} and random features.
\newblock In {\em Advances in Neural Information Processing Systems}, pages
  10212--10223, 2018.

\bibitem{kuzborskij2021role}
Ilja Kuzborskij, Csaba Szepesv{\'a}ri, Omar Rivasplata, Amal Rannen-Triki, and
  Razvan Pascanu.
\newblock On the role of optimization in double descent: A least squares study.
\newblock In {\em Advances in Neural Information Processing Systems}, 2021.

\bibitem{chen2020dimension}
Xi~Chen, Qiang Liu, and Xin~T Tong.
\newblock Dimension independent generalization error by stochastic gradient
  descent.
\newblock {\em arXiv preprint arXiv:2003.11196}, 2020.

\bibitem{berthier2020tight}
Rapha\"{e}l Berthier, Francis Bach, and Pierre Gaillard.
\newblock Tight nonparametric convergence rates for stochastic gradient descent
  under the noiseless linear model.
\newblock In {\em Advances in Neural Information Processing Systems},
  volume~33, pages 2576--2586, 2020.

\bibitem{varre2021last}
Aditya~Vardhan Varre, Loucas Pillaud-Vivien, and Nicolas Flammarion.
\newblock Last iterate convergence of sgd for least-squares in the
  interpolation regime.
\newblock In {\em Advances in Neural Information Processing Systems},
  volume~34, pages 21581--21591, 2021.

\bibitem{wu2021last}
Jingfeng Wu, Difan Zou, Vladimir Braverman, Quanquan Gu, and Sham Kakade.
\newblock Last iterate risk bounds of sgd with decaying stepsize for
  overparameterized linear regression.
\newblock In {\em International Conference on Machine Learning}, pages
  24280--24314. PMLR, 2022.

\bibitem{chen2020multiple}
Lin Chen, Yifei Min, Mikhail Belkin, and Amin Karbasi.
\newblock Multiple descent: Design your own generalization curve.
\newblock In {\em Advances in Neural Information Processing Systems},
  volume~34, pages 8898--8912, 2021.

\bibitem{liang2020multiple}
Tengyuan Liang, Alexander Rakhlin, and Xiyu Zhai.
\newblock On the multiple descent of minimum-norm interpolants and restricted
  lower isometry of kernels.
\newblock In {\em Conference on Learning Theory}, pages 2683--2711, 2020.

\bibitem{neyshabur2017geometry}
Behnam Neyshabur, Ryota Tomioka, Ruslan Salakhutdinov, and Nathan Srebro.
\newblock Geometry of optimization and implicit regularization in deep
  learning.
\newblock {\em arXiv preprint arXiv:1705.03071}, 2017.

\bibitem{smith2020origin}
Samuel~L. Smith, Benoit Dherin, David Barrett, and Soham De.
\newblock On the origin of implicit regularization in stochastic gradient
  descent.
\newblock In {\em International Conference on Learning Representations}, 2020.

\bibitem{cucker2007learning}
Felipe Cucker and Dingxuan Zhou.
\newblock {\em Learning theory: an approximation theory viewpoint}, volume~24.
\newblock Cambridge University Press, 2007.

\bibitem{nitanda2021optimal}
Atsushi Nitanda and Taiji Suzuki.
\newblock Optimal rates for averaged stochastic gradient descent under neural
  tangent kernel regime.
\newblock In {\em International Conference on Learning Representations}, 2020.

\bibitem{Rudi2017Generalization}
Alessandro Rudi and Lorenzo Rosasco.
\newblock Generalization properties of learning with random features.
\newblock In {\em Advances in Neural Information Processing Systems}, pages
  3215--3225, 2017.

\bibitem{el2010spectrum}
Noureddine El~Karoui.
\newblock The spectrum of kernel random matrices.
\newblock {\em Annals of Statistics}, 38(1):1--50, 2010.

\bibitem{liang2020just}
Tengyuan Liang and Alexander Rakhlin.
\newblock Just interpolate: Kernel “ridgeless” regression can generalize.
\newblock {\em Annals of Statistics}, 48(3):1329--1347, 2020.

\bibitem{mei2021generalization}
Song Mei, Theodor Misiakiewicz, and Andrea Montanari.
\newblock Generalization error of random feature and kernel methods:
  hypercontractivity and kernel matrix concentration.
\newblock {\em Applied and Computational Harmonic Analysis}, 2021.

\bibitem{bach2017breaking}
Francis Bach.
\newblock Breaking the curse of dimensionality with convex neural networks.
\newblock {\em Journal of Machine Learning Research}, 18(1):629--681, 2017.

\bibitem{ghorbani2019linearized}
Behrooz Ghorbani, Song Mei, Theodor Misiakiewicz, and Andrea Montanari.
\newblock Linearized two-layers neural networks in high dimension.
\newblock {\em Annals of Statistics}, 49(2):1029--1054, 2021.

\bibitem{cho2009kernel}
Youngmin Cho and Lawrence~K Saul.
\newblock Kernel methods for deep learning.
\newblock In {\em Advances in Neural Information Processing Systems}, pages
  342--350, 2009.

\bibitem{L1998Gradient}
Yann Lecun, Leon Bottou, Yoshua Bengio, and Patrick Haffner.
\newblock Gradient-based learning applied to document recognition.
\newblock {\em Proceedings of the IEEE}, 86(11):2278--2324, 1998.

\bibitem{derezinski2020exact}
Micha{\l} Derezi{\'n}ski, Feynman Liang, and Michael~W Mahoney.
\newblock Exact expressions for double descent and implicit regularization via
  surrogate random design.
\newblock In {\em Advances in Neural Information Processing Systems},
  volume~33, pages 5152--5164, 2020.

\bibitem{dhifallah2020precise}
Oussama Dhifallah and Yue~M Lu.
\newblock A precise performance analysis of learning with random features.
\newblock {\em arXiv preprint arXiv:2008.11904}, 2020.

\bibitem{caponnetto2007optimal}
Andrea Caponnetto and Ernesto De~Vito.
\newblock Optimal rates for the regularized least-squares algorithm.
\newblock {\em Foundations of Computational Mathematics}, 7(3):331--368, 2007.

\bibitem{liu2020kernelreg}
Fanghui Liu, Zhenyu Liao, and Johan~A.K. Suykens.
\newblock Kernel regression in high dimensions: Refined analysis beyond double
  descent.
\newblock In {\em International Conference on Artificial Intelligence and
  Statistics}, pages 649--657, 2021.

\bibitem{donhauser2021rotational}
Konstantin Donhauser, Mingqi Wu, and Fanny Yang.
\newblock How rotational invariance of common kernels prevents generalization
  in high dimensions.
\newblock In {\em International Conference on Machine Learning}, pages
  2804--2814. PMLR, 2021.

\bibitem{vershynin2018high}
Roman Vershynin.
\newblock {\em High-dimensional probability: An introduction with applications
  in data science}, volume~47.
\newblock Cambridge university press, 2018.

\bibitem{williams1998computation}
Christopher~KI Williams.
\newblock Computation with infinite neural networks.
\newblock {\em Neural Computation}, 10(5):1203--1216, 1998.

\bibitem{louart2018random}
Cosme Louart, Zhenyu Liao, and Romain Couillet.
\newblock A random matrix approach to neural networks.
\newblock {\em The Annals of Applied Probability}, 28(2):1190--1248, 2018.

\bibitem{wainwright2019high}
Martin~J Wainwright.
\newblock {\em High-dimensional statistics: A non-asymptotic viewpoint},
  volume~48.
\newblock Cambridge University Press, 2019.

\bibitem{bochner2005harmonic}
Salomon Bochner.
\newblock {\em Harmonic Analysis and the Theory of Probability}.
\newblock Courier Corporation, 2005.

\bibitem{jain2017markov}
Prateek Jain, Sham~M Kakade, Rahul Kidambi, Praneeth Netrapalli,
  Venkata~Krishna Pillutla, and Aaron Sidford.
\newblock A markov chain theory approach to characterizing the minimax
  optimality of stochastic gradient descent (for least squares).
\newblock {\em arXiv preprint arXiv:1710.09430}, 2017.

\bibitem{smith2021origin}
Samuel~L Smith, Benoit Dherin, David Barrett, and Soham De.
\newblock On the origin of implicit regularization in stochastic gradient
  descent.
\newblock In {\em International Conference on Learning Representations}, 2021.

\bibitem{vaswani2019fast}
Sharan Vaswani, Francis Bach, and Mark Schmidt.
\newblock Fast and faster convergence of sgd for over-parameterized models and
  an accelerated perceptron.
\newblock In {\em International Conference on Artificial Intelligence and
  Statistics}, pages 1195--1204. PMLR, 2019.

\end{thebibliography}
